\documentclass[fleqn,msom,nonblindrev]{informs4}

\OneAndAHalfSpacedXI % current default line spacing
\usepackage{endnotes}
\let\footnote=\endnote
\let\enotesize=\scriptsize
\def\notesname{Endnotes}%
\def\makeenmark{$^{\theenmark}$}
\def\enoteformat{\rightskip0pt\leftskip0pt\parindent=1.75em
	\leavevmode\llap{\theenmark.\enskip}}

\usepackage{graphicx}
\usepackage{multirow}
\usepackage{hhline}

\usepackage[small, margin=1cm]{caption}

\usepackage{appendix}
\usepackage{color}
\definecolor{strcolor}{rgb}{0.6, 0.2, 0.6}
\definecolor{commentcolor}{rgb}{0.3125, 0.5, 0.3125}
\definecolor{keycol}{rgb}{0, 0, 1}

\usepackage{bbm}

\usepackage{bm}
\usepackage{tikz,pgfplots,forest}
\usepackage{algorithm}
\usepackage{algorithmicx}
\usepackage{algpseudocode}
\usepackage{subfigure}

\usepackage{listings}
\usepackage{url}

\newcommand {\bea}{\begin{eqnarray}}
	\newcommand {\eea}{\end{eqnarray}}
\newcommand{\E}{\ensuremath{\mathsf{E}}} % expectation
\DeclareMathOperator{\VaR}{VaR}
\newcommand{\field}[1]{\ensuremath{\mathbb{#1}}}
\newcommand{\Z}{\ensuremath{\field{Z}}} % integers
\newcommand{\I}[1]{\ensuremath{\mathbb{I}_{\left\{#1\right\}}}} % indicator function
\newcommand{\PR}{\ensuremath{\mathsf{P}}} % probability
\newcommand{\Tscr}{\ensuremath{\mathcal T}}

\def\blot{\quad \mbox{$\vcenter{ \vbox{ \hrule height.4pt
				\hbox{\vrule width.4pt height.9ex \kern.9ex \vrule width.4pt}
				\hrule height.4pt}}$}}

\usepackage{natbib}
\bibpunct[, ]{(}{)}{,}{a}{}{,}%
\def\bibfont{\fontsize{8}{9.5}\selectfont}%
\TheoremsNumberedThrough     % Preferred (Theorem 1, Lemma 1, Theorem 2)
\ECRepeatTheorems

\EquationsNumberedThrough    % Default: (1), (2), ...
\gdef\AQ#1{}
\gdef\CQ#1{}
\newcommand{\revise}{\color{black}}

\usepackage[most]{tcolorbox}
\newtcolorbox{myquote}[1][]{%
	colback=black!5,
	colframe=black!5,
	notitle,
	sharp corners,
	borderline west={2pt}{0pt}{red!80!black},
	enhanced,
	breakable,
}
\usepackage{clipboard}
\usepackage{makecell}

\usepackage{fancyhdr}
\fancypagestyle{responseletter}{
	\fancyhead[L]{Response Letter of MSOM-2024-0992} % left top
	\fancyhead[R]{\thepage}       % right top (page number)
	\fancyfoot{}
}

\begin{document}
	%%%%%%%%%%%%%%%%
	
%	\AIA
% \setcounter{page}{1} %
% \VOL{00}%
% \NO{0}%
% \MONTH{Xxxxx}%
% \YEAR{2017}%
% \FIRSTPAGE{1}%
% \LASTPAGE{16}%
% \FIRSTPAGEAIA{1}%
% \LASTPAGEAIA{16}%
\def\COPYRIGHTHOLDER{INFORMS}%
\def\COPYRIGHTYEAR{2017}%
\def\DOI{\fontsize{7.5}{9.5}\selectfont\sf\bfseries\noindent https://doi.org/10.1287/opre.2017.1714\CQ{Word count = 9740}}
%\def\RECEIVED{November 1, 2016}
%\def\REVISED{June 22, 2017; October 6, 2017}
%\def\ACCEPTED{November 15, 2017}
% \PUBONLINEAIA{}

	\RUNAUTHOR{Chen et~al.} %

	\RUNTITLE{The Use of BCF to Model and Estimate Discrete Choices}

\TITLE{The Use of Binary Choice Forests to Model and Estimate Discrete Choices}


\ARTICLEAUTHORS{
	\AUTHOR
	{Ningyuan Chen}
	\AFF{Rotman School of Management, University of Toronto, Toronto, Ontario M5S 1A1, Canada,  \EMAIL{ningyuan.chen@utoronto.ca}}
	
	\AUTHOR
	{Guillermo Gallego}
	\AFF{School of Data Science, The Chinese University of Hong Kong, Shenzhen, China, 518172, \EMAIL{gallegoguillermo@cuhk.edu.cn, }}
	
	\AUTHOR
	{Zhuodong Tang}
	\AFF{Antai College of Economics and Management, Shanghai Jiao Tong University,  Shanghai, China, 200030, \EMAIL{zdtang@sjtu.edu.cn}}
}
	 % end of the block
	
%\ARTICLEINFO{\textbf{Received:} November 1, 2016\\ \textbf{Revised:} June 22, 2017; October 6, 2017\\ \textbf{Accepted:} November 15, 2017\\ \textbf{Published Online in Articles in Advance:}}

	\ABSTRACT{
		\textbf{Problem definition.} In retailing, discrete choice models (DCMs) are commonly used to capture the choice behavior of customers when offered an assortment of products. When estimating DCMs using transaction data, flexible models (such as machine learning models or nonparametric models) are typically not interpretable and hard to estimate, while tractable models (such as the multinomial logit model) tend to misspecify the complex behavior represeted in the data. 
		\textbf{Methodology/results.} In this study, we use a forest of binary decision trees to represent DCMs. This approach is based on random forests, a popular machine learning algorithm. The resulting model is interpretable: the decision trees can explain the decision-making process of customers during the purchase. We show that our approach can predict the choice probability of any DCM consistently and thus never suffers from misspecification. Moreover, our algorithm predicts assortments unseen in the training data. The mechanism and errors can be theoretically analyzed. We also prove that the random forest can recover preference rankings of customers thanks to the splitting criterion such as the Gini index and information gain ratio. 
		\textbf{Managerial implications.} The framework has unique practical advantages. It can capture customers' behavioral patterns such as irrationality or sequential searches when purchasing a product. It handles nonstandard formats of training data that result from aggregation. It can measure product importance based on how frequently a random customer would make decisions depending on the presence of the product. It can also incorporate price information and customer features. Our numerical experiments using synthetic and real data show that using random forests to estimate customer choices can outperform existing methods. }

%\FUNDING{The research of the first author is supported by NWO Grant 613.001.208. The third author acknowledges the funding support from the Singapore Ministry of Education Social Science Research Thematic Grant MOE2016-SSRTG-059.}

%\SUBJECTCLASS{\AQ{Please confirm subject classifications.} {\color{blue} discrete choice model; machine learning.}}
%
%\AREAOFREVIEW{\color{blue} Analytics in Operations.}

%{\CQ{Kindly provide the keywords.}}

	%%%%%%%%%%%%%%%%%%%%%%%%%%%%%%%%%%%%%%%%%%%%%%%%%%%%%%%%%%%%%%%%%%%%%%
	
	% Samples of sectioning (and labeling) in OPRE
	% NOTE: (1) \section and \subsection do NOT end with a period
	%       (2) \subsubsection and lower need end punctuation
	%       (3) capitalization is as shown (title style).
	%
	%\section{Introduction.}\label{intro} %%1.
	%\subsection{Duality and the Classical EOQ Problem.}\label{class-EOQ} %% 1.1.
	%\subsection{Outline.}\label{outline1} %% 1.2.
	%\subsubsection{Cyclic Schedules for the General Deterministic SMDP.}
	%  \label{cyclic-schedules} %% 1.2.1
	%\section{Problem Description.}\label{problemdescription} %% 2.
	% Text of your paper here

\maketitle

\section{Introduction}\label{sec:intro}
In retailing, firms collect data on the choice behavior of past customers when they are offered an assortment of products.
The data can be used to predict the choice behavior of future customers, which in turn can help firms to develop effective assortment strategies to improve profits or market shares. 
Discrete choice models (DCMs) play a central role in describing and estimating the choice behavior from the data.

To understand and predict consumers' choice behavior, academics and practitioners have proposed several frameworks, some of which are widely adopted in the industry. 
One ubiquitous framework is to first propose a DCM that potentially captures the underlying choice behavior of customers.
The DCM has a number of parameters that can be estimated with historical data.
Once the model has been appropriately estimated, it can be used as a workhorse to predict the choice behavior of future consumers.
The well-known multinomial logit (MNL) model \citep{McFa73} and its maximum likelihood estimation is a classic example of this framework.

In this framework, there is a trade-off between flexibility and accuracy.  
A flexible DCM has a large number of parameters and incorporates a wide range of consumers' behavior, but it may be difficult to estimate and may overfit the training data. 
On the other hand, a parsimonious model may fail to capture complex choice patterns in the data.
Such misspecification may lead to poor performance in prediction.
The key to a successful model is to reach a delicate balance between flexibility and predictability.
Not surprisingly, it is not straightforward to find the ``sweet spot'' when selecting the DCM. 
For this reason, firms usually estimate and test various models using cross-validation to find the best DCM, which is often repeated as more data is collected.

Another framework favored by data scientists is to apply advanced machine learning algorithms
to the historical data and predict future choice behavior.
This framework skips ``modeling'' entirely and does not attempt to understand the rationality (or irrationality) hidden behind the patterns observed in the training data.
With engineering tweaks, the algorithms can be implemented efficiently and capture various choice behavior.
This approach may sound appealing: if an algorithm achieves impressive accuracy when predicting the choice behavior of consumers, why do we care about the actual rationale in consumers' minds when they make choices?
There are two counterarguments. First, the firm may be interested in not only making accurate predictions but also in other goals such as finding an optimal assortment that maximizes the expected revenue, which may not have appeared in the data.
Without a proper model, it is unclear if revenue maximization can be formulated as an optimization problem. 
Second, when the market environment is subject to secular changes,  having a reasonable model often provides a certain degree of generalizability, while black-box algorithms may fail to capture an obvious pattern just because the pattern has not appeared frequently in the past.

In this paper, we introduce a data-driven framework combining machine learning with interpretable behavioral models, retaining the strengths of both frameworks mentioned previously.
We propose to use random forests, a popular machine learning algorithm \citep{breiman2001random} that is fully nonparametric, to fit the data and predict future customer behavior.
Hence, it differs from the first framework because it does not postulate a parametric DCM upfront.
It also differs from the second framework by providing an interpretable model to describe customers' decision-making process.
Random forests are easy to implement using R or Python \citep{scikit-learn,rrandomforest} and have been shown to have extraordinary predictive power in other applications.
The resulting predictive model, which we refer to as \emph{binary choice forests}, is a mixture of decision trees, each of which reflects the decision-making process of a potential customer.
We provide theoretical analyses for the framework.
First, as the sample size increases, random forests can successfully recover \emph{any} DCM underlying the data.
Although the proof technique is standard, this is the first estimator for DCMs to be shown to have such a property.
Second, we show that random forests serve as adaptive nearest neighbors and effectively use historical data to extrapolate the choice of future customers.
Viewing random forests as adaptive nearest neighbors has been proposed in the statistics literature, and we have adapted the analysis of sampling error (Section~\ref{sec:sampling}) to our setting.
However, due to the special structure of DCMs (the domain of predictors is not continuous), 
the analysis of the distance to the nearest neighbors is combinatorial in nature (Section~\ref{sec:distance}) and has not appeared in the random forest literature.
Third, we show the splitting criterion used by the random forest is intrinsically connected to the preference ranking of customers.
That is, when each customer is endowed with a preference ranking of the products and always chooses the most preferred product in the offered assortment, the random forest can recover the ranking from the data by representing it as a decision tree.
To the best of our knowledge, the third set of results is new to both streams of literature on DCMs and random forests.
%sed for classification and regression, it is inherently a decision tree with the essence of sequential decision-making.

Besides the theoretical properties, 
we explore the practical advantages of the framework thanks to random forests:
(1) It can capture patterns of behavior that elude other models, such as irregularity and sequential searches \citep{weitzman1979optimal}, 
whose details can be found in Section~\ref{sec:behavoral}.
(2) It can return an importance index for all products based on how frequently a random customer would make decisions depending on the presence of the product. The details can be found in Section~\ref{sec:importance}.
(3) It can incorporate prices and reflect the information in consumers' decision-making. The details can be found in Section~\ref{sec:price}.
(4) It can naturally incorporate customer features and is compatible with personalized online retailing. The details can be found in Section~\ref{sec:customer_features}.
(5) It can deal with nonstandard historical data formats, which is a major challenge in practice. The details can be found in Appendix~\ref{sec:aggregate-choice}.
Therefore, we propose random forests as an effective approach to learning consumer choice behavior when the data is abundant and parametric models cannot capture the complex and potentially irrational patterns in the data.

\subsection{Literature Review}\label{sec:literature}
We first review DCMs proposed in the literature, in the increasing order of flexibility and the difficulty of estimation.
The independent demand model and the MNL model \citep{McFa73} have very few parameters (one per product), which are easy to estimate \citep{train2009discrete}.
Although the MNL model is still widely used, its inherent property of independence of irrelevant alternatives (IIA) has been criticized for being unrealistic (see \citealt{anderson1992discrete}).
The nested logit model, the Markov chain DCM, the mixed logit model and the rank-based DCM (see, e.g., \citealt{williams1977formation,train2009discrete,farias2013nonparametric,blanchet2016markov}) are able to capture more complex choice behavior than the MNL model. The mixed logit model can approximate any random utility model (RUM), encompassing an important class of DCMs. 
Although there has been exciting progress in recent years \citep{farias2013nonparametric,van2014market,van2017expectation,csimcsek2018expectation,jagabathula2018conditional},
the computational feasibility and susceptibility to overfitting remain a challenge in practice.
In addition, the class of RUM belongs to the class of \emph{regular} DCMs: the probability of choosing an alternative cannot increase if the offered set is enlarged.  
Experimental studies show strong evidence that regularity may be violated in practice \citep{simonson1992choice}.
Several models are proposed to capture more general behavior than RUM \citep{natarajan2009persistency,flores2017assortment,berbeglia2018generalized,feng2017relation, liu2020assortment, yousefi2020choice}, but it is not yet clear if such models can be estimated efficiently.

The binary choice forest in this paper can be seen as a mixture of customer segments, where each segment has the choice behavior represented by a decision tree. In this sense, it is related to recent studies on consumer segmentation such as \cite{bernstein2018dynamic,jagabathula2018model,aouad2023market, feng2021robust}.
This paper focuses on the estimation of DCMs and the segments emerge as a byproduct to improve the predictive accuracy: the trees have equal weights and we do not control for the number of trees (segments).
In contrast, \citet{bernstein2018dynamic,jagabathula2018model,aouad2023market, feng2021robust} design algorithms to cluster customers so the objectives differ. 
It is worth noticing that \citet{aouad2023market} also use the tree structure for market segmentation. However, the tree splits on the customer features instead of the products.

%\subsection{Other Related Literature}
The specifications of random forests used in this paper are introduced by \citet{breiman2001random}, although many of the ideas were discovered earlier.
The readers may refer to \citet{hastie2009elements} for a general introduction.
Although random forests have been very successful in practice, little is known about their theoretical properties relatively.
Most studies focus on stylized assumptions or simplified versions of the random forest algorithm used in practice.\label{page:literature}
\citet{biau2016random} provide an excellent survey of the recent theoretical and methodological developments in the field. 
\label{page:lite}
{\revise \Copy{rev:lite}{
There are recent papers on theoretical properties (e.g., consistency and asymptotic normality) under less restrictive assumptions.
For example, \citet{lin2006random} study the effect of terminal node sizes on the mean squared error (MSE) of random forest predictions. 
\citet{scornet2015consistency} establish the $L^2$ consistency of random forests in regression problems. \cite{wager2014asymptotic, wager2018estimation} show that asymptotic normality can be established from the ``honest'' assumption:
a tree is honest if it uses separate samples in a node to determine respectively the split points and the prediction. 
This assumption is typically violated by random forest algorithms used in practice.
Although our paper borrows some ideas from the existing literature, its contributions differ substantially in several key aspects:  

\begin{itemize}
	\item \textbf{Problem setting.}  
	\citet{lin2006random, scornet2015consistency, wager2014asymptotic, wager2018estimation} study the \emph{regression} problem in a \emph{continuous} setting, where data points are sampled from $[0,1]^d$. By contrast, we focus on a multi-label \emph{classification} problem in a \emph{discrete} setting, where data points are restricted to the corner points $\{0,1\}^N$, with $N$ denoting the number of products.  
	
	\item \textbf{Assumptions.}  
	These papers typically assume the regression function is continuous (or Lipschitz continuous) and that samples have positive density in the predictor domain. \citet{wager2014asymptotic, wager2018estimation} further require honesty and certain regularity conditions. Such assumptions are not applicable in the context of DCMs. Instead, we rely on $c$-continuity, defined specifically for DCMs. The discrete structure of DCMs enables us to establish consistency and analyze error bounds under milder assumptions, using standard tools such as concentration inequalities.  
	
	\item \textbf{Methodology and main results.}  
	While these papers analyze regression errors (e.g., MSE), we focus on classification error, measured by the $L^1$ distributional distance. In Section~\ref{sec:nearest-neighbor}, we build on the nearest-neighbor interpretation of random forests to explain why the algorithm performs well on assortments that have not appeared in the training data. Although the connection between random forests and nearest-neighbor methods was noted by \citet{lin2006random}, our work extends this perspective to the discrete domain under $c$-continuity for DCMs and leverages the combinatorial nature of the problem.  
	Furthermore, in Section~\ref{sec:gini_rank_list}, we analyze the recovery of customer preference rankings under various splitting criteria for classification problems in random forests. We show that the information gain ratio allows the random forest to recover the preference ranking of customers under a much milder assumption than the Gini index.
\end{itemize}

%The study of popular splitting criteria, including the Gini index, information gain and information gain ratio, can be dated back to \citet{leo1984classification, quinlan1986induction, quinlan1993program}.
%We refer readers to the textbook \citet{zhou2021machine} for a detailed discussion.
%The empirical and theoretical performance of the criteria is shown to be only mildly different
%\citep{mingers1988comparison, mingers1989empirical,raileanu2004theoretical}. 
%In this study, however, we show that the information gain ratio allows the random forest to recover the preference ranking of customers under a much milder assumption than the Gini index.

A recent paper by \citet{chen2019decision} proposes a similar tree-based DCM.
Our study differs substantially in how to estimate the tree from the data and key results.
We use random forests algorithm, while \citet{chen2019decision} rely on mixed-integer programming to estimate the tree structure.
On the theoretical side, \citet{chen2019decision} mainly characterize the
depth of the forest needed to fit a training dataset of assortments. In contrast, we focus on explain why the algorithm works well for \emph{unseen assortments}, which refer to the assortments having not appeared in the training data, and investigate the impact of the splitting criteria.
For the numerical studies, \citet{chen2019decision} benchmark their model against MNL and rank-based DCMs, and focus mainly on datasets with around 10 products.
By constrast, we conduct a more comprehensive and extensive numerical studies.
We find that random forests are quite robust and have good performance even compared with the Markov chain model estimated using the expectation-maximization (EM) algorithm, which has been shown to have outstanding empirical performance compared to MNL, the nested logit, the mixed logit and rank-based DCM \citep{berbeglia2022comparative}, especially when the training data is large.
We also compare the performance of random forests to \cite{chen2019decision} in real and synthetic datasets, and show random forests perform better and more robust. We also evaluate settings with a large number of products (up to 400), and explore aspects such as parameter tuning, assortment variation, and extensions to pricing and customer features. }}
It is worth noticing that some recent studies on optimization frameworks (e.g., assortment optimization) under tree ensemble models greatly extend the applicability of the tree-based modeling frameworks \citep{mivsic2020optimization, biggs2023constrained, perakis2021umotem}. 
%that provides an optimization framework to solve the optimal assortment planning problem for tree-based DCMs.
 %\citet{biggs2023constrained} propose a data-driven optimization framework when an uncertain objective function is estimated by random forests. 
%\citet{perakis2021umotem} propose UMOTEM to solve a constrained optimization problem where the objective function is determined by a tree ensemble model.

{\revise
To help readers clearly distinguish our work from the existing literature, we summarize the key differences from representative literature in Table~\ref{tab:literature}.}

\begin{table}[t]
	\centering
	\caption{\revise Comparison with related literature.}
	\renewcommand{\arraystretch}{1.2}
	{\revise
	\begin{tabular}{c c c c c}
		\hline
		Paper & \makecell{Feature $\bm x$} & \makecell{Task} & \makecell{Method} & Main unique result \\
		\hline
		\cite{lin2006random} & Continuous & Regression & RF & MSE on leaf node size \\
		\cite{scornet2015consistency} & Continuous & Regression & RF & $L^2$ Consistency \\
		\cite{wager2018estimation} & Continuous & Regression & RF & Asymptotic normality  \\
		\cite{chen2019decision} & Discrete & Classification & MIP & Tree depth needed to fit training data \\
		This paper & Discrete & Classification & RF & \makecell{Error bound of unseen assortments,\\preference ranking recovery} \\
		\hline
	\end{tabular}}
	\label{tab:literature}
\end{table}

\section{Data and Estimation}\label{sec:data_estimation}
Consider a set $[N]\triangleq\left\{1,\ldots,N\right\}$ of $N$ products and define $[N]_+ \triangleq [N] \cup \{0\}$ where $0$ represents the no-purchase option.
We use $\bm x \in \left\{0,1\right\}^N$, a binary vector, to represent an assortment of products, where $\bm x(i) = 1$ indicates the inclusion of product $i$ in the assortment and $\bm x(i) = 0$ otherwise.
A {\bf discrete choice model} (DCM)  is a non-negative mapping $p(i,\bm x): [N]_+\times \left\{0,1\right\}^N \to [0,1]$ such that
\begin{equation}\label{eq:dcm}
	\sum_{i \in [N]_+}p(i,\bm x)=1,\quad
	p(i,\bm x)=0\quad \text{if }\;\bm x(i) = 0.
\end{equation}
Here $p(i,\bm x) \in [0,1]$ represents the probability that a customer selects product $i$ from assortment $\bm x$. We refer to a subset $S\subseteq [N]$ as an assortment associated with $\bm x \in \{0,1\}^N$, i.e., $i\in S$ if and only if $\bm x(i) = 1$.
In the remaining paper, we use $p(i,S)$ and $p(i,\bm x)$ interchangeably.

We assume that arriving consumers make choices independently based on an unknown DCM $p(i, \bm x)$.
The firm collects data of the form $(i_t, \bm x_t)$ (or equivalently $(i_t,S_t)$) where $\bm x_t$ is the assortment offered to the $t$th consumer and $i_t \in S_t \cup \{0\}$ is the choice made by consumer $t = 1,\ldots, T$.  
Our goal is to use the data to estimate the underlying DCM $p(i,\bm x)$.
We view the problem as a classification problem. Given the input $\bm x$, we would like to provide a classifier that maps the input to a probability distribution over the \emph{class labels} $i\in [N]_+$, which is referred to as the \emph{class probability}.

To this end, we use a random forest as a classifier.
The output of a random forest consists of $B$ individual classification 
and regression trees (CARTs), $\left\{t_b(\bm x)\right\}_{b=1}^B$, where $B$ is a tunable hyper-parameter.
Here $t_b\colon \{0,1\}^N\to [N]_+$ is the output of CART $b$. 
The choice probability of item $i$ in the assortment $\bm x$ is estimated as
\begin{equation}\label{eq:class-prob}
	\sum_{b=1}^B \frac{1}{B}\I{t_b(\bm x)=i},
\end{equation}
which measures the fraction of trees that assign label $i \in [N]_+$ to the input assortment $\bm x$.
Note that although a single tree only outputs a deterministic class label for each assortment $\bm x$, the aggregation of the trees, i.e., the forest, is naturally equipped with the class probabilities.

We briefly review the basic mechanism of CART and describe how it is used to fit the data $\{(i_t, \bm x_t)\}_{t=1}^T$.
CART recursively splits the input space $[0,1]^N$ (a hypercube), which is a continuous extension to $\left\{0,1\right\}^N$, along its $N$ dimensions.
In the first iteration, it selects a product $i\in [N]$ and a split point $s_i\in[0,1]$ to split the input space.
More precisely, the split $(i,s_i)$ divides the samples to $ \left\{(i_t,\bm x_t): \bm x_t(i)\le s_i\right\}$ and $\left\{(i_t,\bm x_t): \bm x_t(i)> s_i\right\}$.
In our problem, because the predictor $\bm x_t \in \left\{0,1\right\}^N$ is at the corner of the hypercube, all split points between 0 and 1 create the same partition of the observations, and thus we simply set $s_i\equiv 0.5$.
To select a product dimension for splitting, an empirical criterion is optimized to favor splits that create ``purer'' regions (or child nodes in the language of decision trees).
Each of the resulting two regions, $R_1$ and $R_2$, should contain observations that mostly belong to the same class.
We use a common criterion called \emph{Gini index} to demonstrate the idea. (Other splitting criteria and their theoretical properties are discussed in Section~\ref{sec:gini_rank_list}.)
With the Gini index, $(i,s_i)$ is chosen to minimize $\sum_{j=1,2} {t_j}/{T}\sum_{k=0}^N \hat p_{jk}(1-\hat p_{jk})$ where $t_j$ is the number of observations in region $R_j$ and $\hat p_{jk}$ is the empirical frequency of class $k$ in $R_j$.
The splitting procedure is then applied recursively to the two regions and their subregions.
Finally, the outcome of the tree $t_b(\bm x)$ is typically defined as the majority class label in the leaf node (the smallest region after the splits) $\bm x$ belongs to.

Next, we explain the details of random forests on top of CART.
To create $B$ CARTs, for each $b=1,\dots,B$, we randomly choose $z$ samples with replacement from the $T$ observations (bootstrap samples).
Only the sub-sample of $z$ observations is used to train the $b$th CART.
Splits are performed only on a random subset of $[N]$ of size $m$ to optimize a criterion such as the Gini index.
The random sub-sample of training data and random products to split are two key ingredients in creating less correlated CARTs in the random forest. The depth of the tree is controlled by the minimal number of observations, say $l$, in a leaf node for the tree to keep splitting.
These steps are formalized in Algorithm~\ref{alg:rf}.
{\SingleSpacedXI
	\begin{algorithm}
		\caption{Random forests for DCM estimation}\label{alg:rf}
		\begin{algorithmic}[1]
			\State Data: $\left\{(i_t,\bm x_t)\right\}_{t=1}^T$
			\State Tunable hyper-parameters: number of trees $B$, sub-sample size $z\in\left\{1,\dots,T\right\}$,
			number of products to split $m\in \left\{1,\dots,N\right\}$, terminal leaf size $l\in\{1,\dots,z\}$\label{step:parameter}
			\For{$b=1$ to $B$}
			\State Select $z$ observations from the training data with replacement, denoted by $Z$\label{step:select-subsample}
			\State Initialize the tree $t_b(\bm x)\equiv 0$ with a single root node
			\While{some leaf has greater than or equal to $l$ observations belonging to $Z$ and can be split}\label{step:while-split}
			\State Select $m$ products without replacement among $\left\{1,\dots,N\right\}$\label{step:choose-direction}
			\State Select the optimal one to split among the $m$ products that optimizes a splitting criterion such as the Gini index  \label{step:criterion}
			\State Split the leaf node into two
			\EndWhile
			\State Denote the partition associated with the leaves of the tree by $\left\{R_1,\dots,R_M\right\}$. 
			Let $c_i\in [N]_+$ be the class label of a randomly chosen observation in $R_i$ from the training data\label{step:random-label}
			\State Define $t_b(\bm x) = \sum_{i=1}^{M} c_i\I{\bm x\in R_i}$
			\EndFor
			\State The trees $\left\{t_b(\cdot)\right\}_{b=1}^B$ are used to estimate the class probabilities as \eqref{eq:class-prob}\label{step:forest}
		\end{algorithmic}
\end{algorithm}}

The use of random forests as a generic classifier has a few benefits:
(1) Many machine learning algorithms such as neural networks have numerous hyper-parameters to tune and the performance crucially depends on the suitable choice of hyper-parameters.
Random forests, on the other hand, have only a few hyper-parameters.
In the numerical studies in this paper, we simply choose a set of hyper-parameters that are commonly used for classification problems, without cross-validation or tuning, in order to demonstrate the robustness of the algorithm.
In particularly, we mostly use $B=1000$, $z=T$, $m=\sqrt{N}$ and $l=50$.
The effect of the hyper-parameters is studied in Online Appendix~\ref{sec:hyper-parameter}. (Due to page limit, Online Appendix is available in online version or upon request.)
(2) The implementation of the generic algorithm is included in packages of R and Python.
In Appendix~\ref{sec:sample-code}, we demonstrate the algorithm using scikit-learn, a popular machine learning package in Python that implements random forests, to estimate consumer choices.
It usually takes less than 20 lines to implement the procedure.

More specifically, because of the application structure, there are a few observations.
(1) Because the entries of $\bm x$ are binary $\left\{0,1\right\}$, the split position of decision trees is always $0.5$.
Therefore, along a branch of a decision tree, there can be at most one split on a particular product,
and the depth of a decision tree is at most $N$.
(2) The output random forest is not necessarily a DCM.
\label{page:normalization}{\revise \Copy{rev:normalization}{In particular, the probability of class $i$, or the choice probability of product $i$ given assortment $\bm x$, may be positive even when $\bm x(i)=0$, i.e., product $i$ is not included in the assortment.
To fix the issue, we adjust the probability of class $i$ by conditioning on the trees that output reasonable class labels:
$\sum_{b=1}^B\frac{1}{\sum_{j:\bm x(j)=1} \sum_{b=1}^B \I{t_b(\bm x)=j}}\I{t_b(\bm x)=i, \bm x(i)=1}.$}}
%\begin{equation*}
%	\label{eq:normalization}
%	\sum_{b=1}^B\frac{1}{\sum_{j:\bm x(j)=1} \sum_{b=1}^B \I{t_b(\bm x)=j}}\I{t_b(\bm x)=i, \bm x(i)=1}.
%\end{equation*}
(3) We slightly modify the generic algorithm.
In particular, when returning the class label of a leaf node in a decision tree, we use a randomly chosen observation instead of taking a majority vote (Step~\ref{step:random-label} in Algorithm~\ref{alg:rf}).
While not a typical choice, it is crucial in deriving our consistency result (Theorem~\ref{thm:consistency}).
In practice, such adjustments are minimal and hardly change the predictive accuracy.
In the experiments, we fix $l=50$ for the terminal leaf regardless of the sample size.
It implies that the trees in the random forest are deep for large datasets. 
This is a feature of random forests that average out the high variance and reap the benefit of low bias of individual deep trees.

\subsection{Interpretability}\label{sec:interpretability}
Next, we connect the random forest fitted from Algorithm~\ref{alg:rf} to the context of customer choices.
Consider a simple scenario of $N=2$ products. 
Suppose one CART resulting from Step~\ref{step:random-label} of Algorithm~\ref{alg:rf} is illustrated in the left panel of Figure~\ref{fig:partition}.
It can be equivalently represented by a decision tree in the right panel of Figure~\ref{fig:partition}. 
The decision tree can be interpreted as follows: a customer first checks the presence of product one ($\bm x(1)\ge 0.5$). If it is present, then product one is purchased ($c_2=1$). 
Otherwise, she proceeds to check product two, and purchases it if it is present ($c_3=2$).
We refer to the decision tree and the associated behavior as a {\bf binary choice tree}.
\begin{figure}[t]
	\begin{center}
		\caption{A binary choice tree representation of the partition.}
		\scalebox{0.8}{
		\begin{tikzpicture}
			\draw (0,0) --(0,4) --(4,4) --(4,0) --(0,0);
			\draw (0,2) --(2,2);
			\draw (2,0) --(2,4);
			\draw [->] (4,0) -- (5,0) node[below right] {Product 1};
			\draw [->] (0,4) -- (0,5) node[left] {Product 2};
			\node [below left] at (0,0) {0};
			\node [below] at (4,0) {1};
			\node [left] at (0,4) {1};
			\node at (1,1) {$R_1, c_1 = 0$};
			\node at (3,2) {$R_2, c_2 = 1$};
			\node at (1,3) {$R_3, c_3 = 2$};
		\end{tikzpicture}
		\hspace{1cm}
		\begin{forest}
			for tree={l sep+=.8cm,s sep+=.5cm,shape=rectangle, rounded corners,
				draw, align=center,
				top color=white, bottom color=gray!20}
			[$\bm x(1)\ge 0.5$
			[1,edge label={node[midway,left]{Y}}
			]
			[$\bm x(2)\ge 0.5$,edge label={node[midway,right]{N}}
			[2, edge label={node[midway,left]{Y}}]
			[0, edge label={node[midway,right]{N}}]
			]
			]
		\end{forest}}
		\label{fig:partition}
	\end{center}
\end{figure}
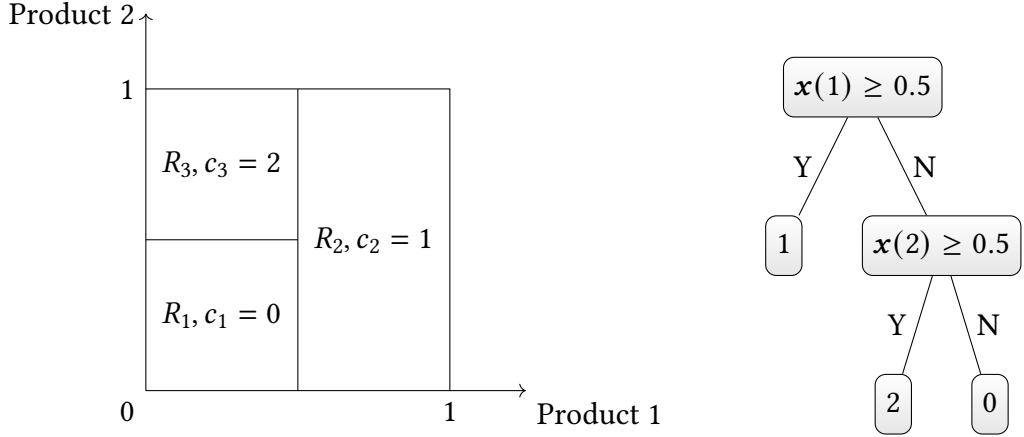

As Algorithm~\ref{alg:rf} aggregates the outputs of $B$ binary choice trees in \eqref{eq:class-prob},
the class probability can be interpreted as randomly following the choice of one of $B$ customers, each of which makes a decision based on a binary choice tree $t_b(\bm x)$.
We refer to such a mixture of binary choice trees as a {\bf binary choice forest} (BCF).
Therefore, Algorithm~\ref{alg:rf} not only allows us to predict the choice behavior from the data, but also provides a modeling tool to capture the behavior.
See Section~\ref{sec:behavoral} for more details.

\section{Why Do Random Forests Work Well?}\label{sec:theory}
Many machine learning algorithms, including random forests, have strong performances in practice.
However, with a few exceptions, there is little theoretical understanding of the impressive performance.  
For example, consistency and asymptotic normality, the two most fundamental properties a statistician would demand, are only recently established for random forests under restrictive assumptions \citep{wager2014asymptotic,scornet2015consistency,wager2018estimation}.
The lack of theoretical understanding may worry practitioners when stakes are high and the failure may have harmful consequences. 
In this section, we attempt to answer the ``why'' question.
The section consists of three major results:
(1) random forests are consistent for \emph{any} DCM, (2) random forests can be viewed as nearest neighbors, whose performance is explained by a few crucial factors, and (3) the choice of the splitting criterion can help random forests recover a class of widely used DCMs called rank-based models.
Note that all three theoretical explanations depend specifically on the structure of the discrete choice and do not hold for general classification or regression problems. Therefore, our findings reveal the benefits of applying random forests to DCMs specifically.

\subsection{Random Forests are Consistent for Any DCM} \label{sec:consistency}
We now show that given enough data, random forests can recover the choice probability of \emph{any} DCM. To obtain our theoretical results, we impose mild assumptions on how the data is generated.
\begin{assumption}\label{asp:independent-choice}
	There is an underlying DCM from which all $T$ consumers independently make choices from the offered assortments, generating data $(i_t,\bm x_t)$, $t = 1, \ldots T$.
\end{assumption}

Notice that the assumption only requires consumers to make choices independently. We do not assume that the offered assortments $\bm x_t$s are IID, and allow the sequence of assortments offered $\bm x_t$ to be arbitrary as the assortments are chosen by the firm and are typically not randomly generated.
Such a design reflects how firms select assortments to maximize expected revenues or to explore customer preferences.

%Since the consistency result requires the sample size $T\to\infty$, we use the subscript $T$ to emphasize the fact that the parameters may be chosen based on $T$.

For a given assortment $\bm x$, let $k_T(\bm x)  \triangleq \sum_{t=1}^T \I{\bm x_t=\bm x}$ be the number of consumers who have seen assortment $\bm x$. We are now ready to establish the consistency of random forests.
\begin{theorem}\label{thm:consistency}
	Suppose Assumption~\ref{asp:independent-choice} holds, then for any $\bm x$ and $i$, if  $\liminf_{T \to \infty} k_T(\bm x) / T>0$, $l_T$ is fixed, $z_T\to\infty$, $B_T\to\infty$, then the random forest is consistent for the assortment $\bm x$:
	\begin{equation*}
		\lim_{T\to\infty}\PR\left(\left|\sum_{b=1}^{B_T} \frac{1}{B_T}\I{t_b(\bm x)=i}-p(i, \bm x)\right|>\epsilon\right) =0
	\end{equation*}
	for all $\epsilon>0$.
\end{theorem}

\label{page:thm1_compare}
All technical proofs are provided in Appendix \ref{sec:proof}. {\revise \Copy{rev:thm1_compare}{  \cite{chen2019decision} show that decision forests can represent any general choice model, and our Theorem \ref{thm:consistency} demonstrates that the forest generated by the random forest algorithm can accurately predict any choice model as $T \rightarrow \infty$. The consistency of {\em individual} decision trees has been shown under certain conditions, typically including the diminishing diameter and the increasing number of data points in terminal nodes.
See Theorem 13.1 in \cite{gyorfi2006distribution} or Chapters 20 and 21 of \cite{devroye2013probabilistic}.
While these conditions hold in our setting, because of the clustering of data points at the corners,
we have to handle the additional resampling issue in Step~\ref{step:select-subsample} in Algorithm~\ref{alg:rf}.
The proof is based on standard concentration inequalities.}}

According to Theorem~\ref{thm:consistency}, the random forest can accurately predict the choice probability of any DCM, given that the firm has offered the assortment many times.
Practically, the result can guide us about the choice of parameters.
In fact, we just need to generate many trees in the forest ($B_T\to\infty$), resample many observations in a decision tree ($z_T\to\infty$), and keep the terminal leaf small ($l_T$ is fixed).
The requirement is easily met by choice of parameters in the remarks following Algorithm~\ref{alg:rf}, i.e., $z=T$, $m=\sqrt{N}$ and $l=50$.
Theorem~\ref{thm:consistency} guarantees good performance of the random forest when the seller has collected a large dataset. This is a typical case in online retailing, especially in the era of ``big data.''
Random forests thus provide a novel \emph{data-driven} approach to model customer choices.
In particular, the model is first trained from data, then used to interpret the inherent thought process of consumers when they make purchases.
By Theorem~\ref{thm:consistency}, when the historical data has a large sample size, the model can accurately predict how consumers make decisions in reality.
%This reflects the universality of the model.
%Indeed, in Section~\ref{sec:practical-considerations}, we provide concrete examples demonstrating several practical considerations that can hardly be captured by other DCMs and handled well by random forests.

\subsection{Random Forests and Nearest Neighbors} \label{sec:nearest-neighbor}
In Section~\ref{sec:consistency}, we state that when an assortment is offered frequently, the choice probabilities estimated by random forests of this assortment are consistent.
However, this doesn't explain the strong performance of random forests on assortments that have not been offered frequently in the training data (so-called unseen assortments).
In this section, following the intuition provided in \cite{lin2006random}, we attempt to provide a unique perspective based on nearest neighbors.
To motivate, consider the following examples.

%When $N$ is large, firms do not have the luxury of offering all of the $2^N$ possible assortments.
%This brings up the problem of making good prediction for unseen assortments. In this section, we show that this is possible provided that the dataset contains neighbors that are close to the unseen assortment for which we need to make predictions.

%Unfortunately for binary choice forests, the atoms are concentrated on the extreme points, so  the continuity condition is violated. Nevertheless, we can still borrow some ideas to estimate choice probabilities for unseen assortments. The following example will show that customers' behaviour would not significantly change when only a few products are different. To facilitate the discussion we will not consider the outside option in the rest of this section.

\begin{example}
	\label{exmp:PNN_frequency}
	Consider $N=4$ products. Suppose only four assortments are offered in the training data:
	$S_1 = \{1,2,3\}$, $S_2 = \{1,2,4\}$, $S_3 = \{3,4\}$, $S_4 = \{1,2\}$.
	As a result, the assortment $S=\left\{1,2,3,4\right\}$ (or $\bm x=(1,1,1,1)$) is never offered in the data.
	How would random forests predict the choice probabilities of the unseen $S$?
	By searching for the terminal leaves which $S$ falls in among individual trees, Step~\ref{step:random-label} in Algorithm~\ref{alg:rf} uses the samples of the assortment appearing in the training data which happens to be in the \emph{same} leaf node to extrapolate the choice probabilities of $S$.
	If we use Algorithm~\ref{alg:rf} with $B = 1000, z = T, m = \sqrt{N}, l = 1$, then Figure~\ref{fig:PNN_frequency} illustrates the frequencies of the three assortments appearing in the same leaf node as $S$ (a deeper color indicates a higher frequency).
	As we can see, $S_1$ and $S_2$ are more likely to fall in the same leaf node as $S$, while $S_3$ is less frequent and $S_4$ is never used to predict the choice probabilities of $S$.
	The frequency roughly aligns with the ``distance'' from $S_i$ to the unseen $S$ in the graph, defined as the number of edges to traverse between two vertices.
	However, the distance doesn't explain why $S_4$, which is of the same distance as $S_3$, is never used.
	We will address this puzzle later in the section.
	\begin{figure}[t]
		\centering
		\caption{The frequency (number of trees in the forest) at which the assortments in the training data are used to predict the unseen assortment for $N=4$.}
		\scalebox{0.8}{
		\begin{tikzpicture}[scale=0.6]
			\pgfmathsetmacro{\a}{9};  %outer cube side length
			\pgfmathsetmacro{\b}{\a/3};  %inner cube side length
			\pgfmathsetmacro{\xo}{\a/9};  %outer cube stereo perspective increment on x axis
			\pgfmathsetmacro{\yo}{\a*2/9};  %outer cube stereo perspective increment on y axis
			\pgfmathsetmacro{\xi}{\a*2/27};  %inner cube stereo perspective increment on x axis
			\pgfmathsetmacro{\yi}{\a*4/27};  %inner cube stereo perspective increment on y axis
			%\pgfmathsetmacro{\a}
			
			% outer cube nodes
			\node [circle, draw, color = black, fill = red, fill opacity = 1, minimum size=38pt, label = {[red] left: \large Unseen $S$}] (1111) at (0,\a) {1111};
			\node [circle, draw, color = black, fill = blue, fill opacity = 0.5, text opacity = 1, label = { left: \large $S_2$}] (1101) at (0,0) {1101};
			\node [circle, draw, color = black, fill = blue, fill opacity = 0, text opacity = 1, label = { right: \large $S_4$}] (1100) at (\a,0) {1100};
			\node [circle, draw, color = black, fill = blue, fill opacity = 0.6, text opacity = 1, label = { below right: \large $S_1$}] (1110) at (\a,\a) {1110};
			\node (1001) at (\xo, \yo) {1001};
			\node (1000) at (\a+\xo, \yo) {1000};
			\node (1011) at (\xo, \a+\yo) {1011};
			\node (1010) at (\a+\xo,\a+\yo) {1010};
			
			% inner cube nodes
			\node (0101) at (\b,\b) {0101};
			\node (0100) at (\b*2, \b) {0100};
			\node (0111) at (\b, \b*2) {0111};
			\node (0110) at (\b*2, \b*2) {0110};
			\node (0001) at (\b+\xi,\b+\yi) {0001};
			\node (0000) at (\b*2+\xi, \b+\yi) {0000};
			\node [circle, draw, color = black, fill = blue, fill opacity = 0.2, text opacity = 1, label = { above: \large $S_3$}] (0011) at (\b+\xi, \b*2+\yi) {0011};
			\node (0010) at (\b*2+\xi, \b*2+\yi) {0010};
			
			% draw lines of outer cube
			\draw[black] (1111) -- (1110) -- (1100) -- (1101) -- (1111);
			\draw[black] (1011) -- (1010) -- (1000) -- (1001) -- (1011);
			\draw[black] (1111) -- (1011);
			\draw[black] (1110) -- (1010);
			\draw[black] (1101) -- (1001);
			\draw[black] (1100) -- (1000);
			
			% draw lines of inner cube
			\draw[black] (0111) -- (0110) -- (0100) -- (0101) -- (0111);
			\draw[black] (0011) -- (0010) -- (0000) -- (0001) -- (0011);
			\draw[black] (0111) -- (0011);
			\draw[black] (0110) -- (0010);
			\draw[black] (0101) -- (0001);
			\draw[black] (0100) -- (0000);
			
			% draw lines between outer and inner cubes
			\draw[black] (1111) -- (0111);
			\draw[black] (1101) -- (0101);
			\draw[black] (1110) -- (0110);
			\draw[black] (1100) -- (0100);
			\draw[black] (1001) -- (0001);
			\draw[black] (1000) -- (0000);
			\draw[black] (1011) -- (0011);
			\draw[black] (1010) -- (0010);
			
		\end{tikzpicture}}
		\label{fig:PNN_frequency}
	\end{figure}
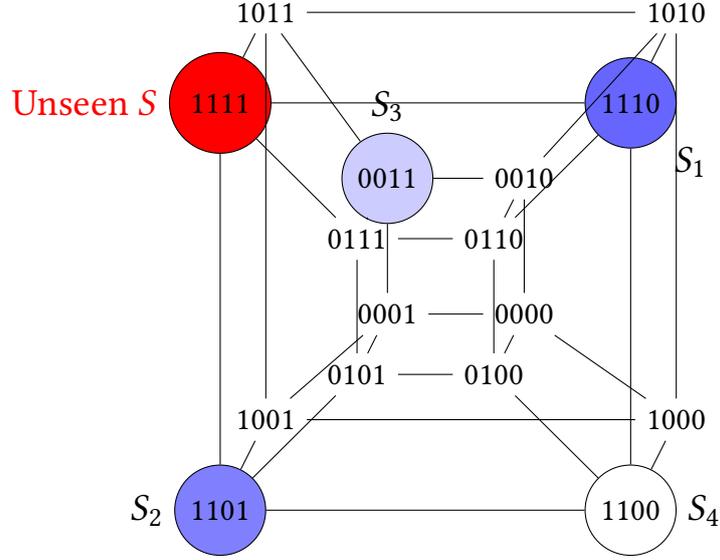
\end{example}

\begin{example}\label{exmp:PNN_frequency_table}
	To provide a more concrete example, we consider $N=10$ and only allow $100$ out of $2^{10}=1024$ assortments to be included in the training data.
	After generating $T=1000$ samples for the training data using the MNL model,
	we use Algorithm~\ref{alg:rf} with $B = 1000, z = T, m = \sqrt{N}, l = 1$ to predict the choice probability of the unseen assortment $S = \{1,2,3,4,5,6\}$.
	The 10 most frequent assortments in the training data that fall into the same leaf node as $S$ are listed in Table~\ref{tab:PNN_result} as well as their frequencies.
	We also count the number of different products between the assortments and $S$, in terms of the symmetric difference of two sets.
	Clearly, the frequency is strongly negatively associated with the number of different products, a distance measure of two assortments.
	\begin{table*}[t]
		\centering
		\caption{The frequency at which the assortments in the training data are used to predict the unseen assortment $S=\left\{1,2,3,4,5,6\right\}$ for $N=10$.}
	\def\arraystretch{0.8}\begin{tabular}{crr}
				\hline
				Assortment & \#Different products & Frequency \\
				\hline
				$\{1,2,3,4,5\}$  &  1  & 0.328\\
				$\{1,2,3,4,5,6,8\}$  &  1  & 0.202\\
				$\{1,2,3,6\}$  &  2  &  0.173\\
				$\{1,2,4,6\}$  &  2  & 0.097\\
				$\{1,2,4,5,6,7,10\}$  &  3  &  0.074\\
				$\{1,3,4,6,10\}$  &  3  &  0.051\\
				$\{2,3,4,5,6,9\}$  & 2 & 0.038\\
				$\{1,3,4,6,7\}$  & 3 & 0.032\\
				$\{1,2,5,6,9,10\}$  & 4 & 0.004\\
				$\{1,4,5,6,9\}$  &  3  & 0.001\\
				\hline
			\end{tabular}
			\label{tab:PNN_result}
			\vspace{-4mm}
	\end{table*}
\end{example}
The examples reveal the intrinsic connection between random forests and nearest neighbors.
Namely, if an assortment is not offered in the training data, then random forests would look for its neighboring assortments in the training data, by grouping them in the same leaf node, to predict the choice probabilities of the unseen assortment.
Unlike nearest neighbors, random forests don't always find the nearest one, as the chosen neighbor is determined by layers of mechanisms such as splitting and randomization.
Arguably, this difference leads to the improved performance of random forests, as well as its intractability.

{\revise
Based on this perspective, three crucial factors determine the performance of random forests when predicting unseen assortments:
\begin{enumerate}
	\item How ``far'' is the unseen assortment to the neighboring assortments in the training data? If all the assortments in the training data include many different products from the unseen assortment, then the extrapolation may not perform well.
	This is a property regarding what assortments are offered in the training data and the unseen assortment.
	\item How ``continuous'' is the DCM to be estimated? The information of the neighbors is not helpful if the choice probabilities vary wildly as the assortment ``travels'' to the neighbors.
	This is a property regarding the underlying DCM that generates the choice of the training data.
	\item How representative are the choices made by customers? To extrapolate to the unseen assortment, the estimated choice probabilities of the assortments offered in the training data need to be accurate.
	This is determined by the number of samples for each assortment.
\end{enumerate}

Next, we analyze the impact of the three factors on the performance of random forests.

\subsubsection{Distance to Neighbors in the Training Data}\label{sec:distance}
For two assortments (subsets) $S_1$ and $S_2$, the \emph{symmetric difference} is defined as $ S_1 \ominus S_2 \triangleq S_1 \cup S_2 - S_1 \cap S_2 = (S_1 - S_2) \cup (S_2-S_1)$.
We define the distance between $S_1$ and $S_2$, $d(S_1,S_2)$, by the cardinality of the symmetric difference:
\begin{equation}\label{eq:distance}
	d(S_1,S_2) \triangleq |S_1 \ominus S_2|.
\end{equation}
The distance defined in equation \eqref{eq:distance} is the number of different products mentioned in Examples~\ref{exmp:PNN_frequency} and \ref{exmp:PNN_frequency_table}.

We first attempt to study what assortments in the training data can potentially be a neighbor to the unseen assortment, i.e., they fall into the same leaf node in some trees in the forest.
Let $\mathcal{T} \subseteq 2^{[N]}$ denote the set of assortments observed in training data.
%Next we will define the \emph{dominance} relation of two assortments for given training set and an arbitrary unseen assortment.
For an unseen assortment $S$ and an assortment $S_1\in \mathcal T$, we define the concept of \emph{Potential Nearest Neighbor} (PNN). {\revise Let $\ell$ denote the maximum number of distinct assortments in a leaf node ($\ell \ge 1$). Note that $\ell$ differs from the leaf node size $l$, since multiple observations may correspond to the same assortment. Roughly speaking, $l = \ell \times \#$ observations in each assortment.}
%\begin{definition}
%	For an unseen assortment $S \notin \mathcal{T}$ and $S_1, S_2 \in \mathcal{T}$, we say $S_1$ dominates $S_2$ for $S$, denoted by $S_1 \succ^{(S)} S_2$ if and only if $S \ominus S_1 \subset S \ominus S_2$, and in this case we have $d(S, S_1) \leq d(S, S_2)$.
%\end{definition}
%
%Dominance in the above definition means that $S_1$ is strictly more similar to $S$ than $S_2$. Notice that two assortments may not dominate each other. In Example \ref{exmp:PNN_frequency}, $S_1$ and $S_2$ dominate $S_4$, but $S_1$ and $S_2$ cannot dominate each other.
\begin{definition}
	\label{def: PNN}
	Given the maximum number of distinct assortments $\ell \ge 1$, an assortment $S_1\in \mathcal T$ is an $\ell$-PNN of $S\notin \mathcal T$ if $\#\{S_2\in \mathcal T: S \ominus S_2 \subsetneq S \ominus S_1\} \le \ell-1$.
\end{definition}
In other words, $S_1$ is an $\ell$-PNN of $S$ if at most $\ell-1$ assortments observed in the training data
dominate $S_1$ in terms of the similarity to $S$ at the product level.
% The concept is closely related to the samples fall into the same terminal leaf node as $S$.
%The PNN definition is similar to $k$-PNN in the  random forest literature for a general classification or regression problem on $[0,1]^N$. Readers of interest are referred to  for detailed and more general definition.  In our problem, we restrict to corner points and 1-PNN is considered as we assume $l = 1$.
\begin{proposition}
	\label{prn: RF_PNN}
	\Copy{rev:prn1}{Suppose the the maximum number of distinct assortments is $\ell$ in Algorithm~\ref{alg:rf} ($\ell \ge 1$). An unseen assortment $S \notin \mathcal{T}$ and an assortment $S_i \in \mathcal{T}$ can fall into the same terminal leaf node for some trees if and only if $S_i$ is an $\ell$-PNN of $S$.}
\end{proposition}
% \begin{remark}\label{rmk:size-one-node-pnn}
% We study the case $l=1$ for the simplicity of the analysis and to provide technical insights.
% For $l>1$, the condition for PNN is much more involved.
% Moreover, $l=1$ is a popular choice for classification problems \citep{biau2016random}.
% \end{remark}
%\todo{Add a remark about when $l\neq 1$. The proof needs to be reworked as I changed some notations.}
%\textbf{Remark:} When the leaf node size $l>1$, then $S_1$ is a PNN if $\# \{S_2 \in \mathcal{T}: S \ominus S_2 \subsetneq S \ominus S_1\} \leq l-1$.

Based on Proposition~\ref{prn: RF_PNN}, the distance of an unseen assortment $S$ to the training data is essentially the distance to its $\ell$-PNNs.
We next provide an estimate of this distance when a fraction of assortments are observed in the training data.
In particular, we show that for all $\ell$-PNNs of an unseen assortment to be close, the training set only needs to have $\Big\lceil2^{N+2} \cdot c_0 \cdot \lceil\log_2 N\rceil \cdot \log N / (N-2)\Big\rceil$ assortments for constant $c_0 \ge 1$.
%\todo{I think we used $\log$ for $\ln$ in some proofs. Need to make consistent.}

%We will next consider when only a certain fraction $\epsilon:= |\mathcal{T}| / 2^N$ {\color{blue} ($\epsilon$ is also used in Theorem \ref{thm:consistency}, do we need a new notation?)} of assortments are observed in training set, the symmetric distance between an unseen assortment and its leaf node by random forest algorithm in three different cases.

%We consider the number of assortments in training set such that neighbor in each tree of an unseen assortment is within distance $r$.
%The following theorem shows that if each assortment is observed with certain frequency, then for an unseen assortment, neighbors in every tree is within distance $O(\log N)$ with high probability.
\begin{proposition}
	\label{prop:distance_worst_case}
	\Copy{rev:prn2}{Suppose $M = \Big\lceil2^{N+2} \cdot c_0 \cdot \lceil\log_2 N\rceil \cdot \log N / (N-2)\Big\rceil$ assortments are drawn randomly with replacement as the training data and set $\ell \le c_0\log_2 N$ in Algorithm~\ref{alg:rf}, where $c_0 \ge 1$ is a constant.
	For an arbitrary assortment $S$, with probability no less than $1 - 1/(\lceil \log_2 N \rceil)!$,\footnote{We can show that $(\lceil \log_2 N \rceil)!>\lceil\log_2 N\rceil \cdot (\lceil\log_2 N\rceil -1)\cdot \ldots \cdot (\lceil\log_2 N /2\rceil) > (\log_2 N / 2)^{\log_2 N/2}$, so all bounds hold when we replace $\lceil\log_2 N\rceil!$ by $(\log_2 N/2)^{\log_2 N/2}$.}
	the distance $d(S,S')\leq \lceil\log_2 N\rceil-1$ for all PNNs $S'$.}
\end{proposition}
In other words, if a fraction $O(\log(N)^2/N)$ of assortments appear in the training set and $\ell = O(\log N)$, then the distance of $S$ to its $\ell$-PNNs is guaranteed (in a probabilistic sense) to be of order $O(\log N)$.
However, this is a strong condition.
% for \emph{all the trees} in the forest, the assortment $S'$ that falls into the same terminal leaf node as $S$ is close to, or equivalently, has a similar set of products to $S$.
% It guarantees the choice probabilities of $S'$ can predict those of $S$ reasonably well.
In practice, even if there are a few trees that predict $S$ by a PNN that is ``far'' from the unseen assortment (including a very different set of products),
random forests are able to mitigate the bias by averaging them with other closer trees.
This is one of the reasons why random forests perform much better than CARTs in general.

To explore such an effect, we explore the \emph{expected} distance of a PNN.
To simplify the analysis, we consider a stylized splitting rule called random splitting,
that is, $m=1$ in Algorithm~\ref{alg:rf}.
Random splitting is commonly used to shed theoretical insights about the performance of random forests \citep{biau2016random}.
It also satisfies the ``honest tree'' assumption in \cite{wager2014asymptotic, wager2018estimation}.
%In this case, we assume the random forest algorithm adopts randomly split, i.e., randomly select a product as split point instead of selecting the product with greatest reduction on Gini index in each iteration.
%Following \cite{wager2014asymptotic, wager2018estimation},
%%\cite{meinshausen2006quantile, biau2008consistency, biau2012analysis}
%we also consider a special case of \emph{honest} tree to reduce bias.
%A tree is honest if it does not use labels for both determining split points and making nearest neighbor  prediction.
%Randomly split is a special case of honest tree because no label is used for splitting.
%CART is considered dishonest since training labels are used for both prediction and splitting rule.
%Although CART induces subtle bias, CART has a good performance in most scenarios.
We can derive the following result.
\begin{proposition}
	\label{thm:dis_upper_bound}
	\Copy{rev:prn3}{Suppose we draw $M = \lceil 2^N / N\rceil$ assortments with replacement as the training data and set $m=1$ in Algorithm~\ref{alg:rf}.
	For an arbitrary assortment, its expected distance to an $\ell$-PNN in the training data is bounded above by
	$\log_2 \ell +\log_2 N + 2.56$, where $\ell$ is the maximum number of distinct assortments in a leaf node.}
\end{proposition}
Although the order of magnitude seems to be similar to Proposition~\ref{prop:distance_worst_case}, i.e.,
with $\tilde{O}(2^N/N)$ assortments\footnote{The notation $\tilde O$ represents the asymptotic order neglecting the logarithmic factors.} and $\ell = O(\log_2(N))$, the distance is bounded 
by $\tilde{O}(\log_2(N))$,
the constants and logarithmic factors may play a role in explaining why random forests (the average distance) perform better than individual trees.
We have also conducted numerical studies showing the bound in Proposition~\ref{thm:dis_upper_bound} is more or less tight:
If $M$ grows at a slower rate, for example, polynomially in $N$ or $M=O(N^{\log N})$, then the average distance cannot be bounded by $O(\log_2(N))$. See Online Appendix~\ref{sec:distance_numerical}.

If the $M$ assortments in the training data are selected by the firm to minimize the distance to all unseen assortments, then how do the results change?
This is similar to the setting of the experimental design.
For this question, we refer to the literature on the so-called \emph{Covering Code} problem, see \cite{cohen1997covering, ostergaard1998new}.
\label{page:covering-code}
The Covering Code problem aims to find the minimum number of binary vectors in $\{0,1\}^N$, such that every other element in the set is within distance $r$ to some selected ones.
To our knowledge, the covering code problem is still an open problem, and numerous bounds are established.
In general, when $M = O(2^N / N)$, these binary vectors can cover all the binary vectors within distance 1.
This improves the distance in Proposition~\ref{thm:dis_upper_bound} from $\log_2(N)$ to $1$.
This literature may provide some useful heuristics and algorithms on how to design assortments in a new market in order to explore customers' choices efficiently.
		
We summarize the main results in this section below.
When $M =\Big\lceil2^{N+2} \cdot c_0 \cdot \lceil\log_2 N\rceil \cdot \log N / (N-2)\Big\rceil$ and $\ell \le c_0 \log_2 (N)$, where $c_0 \ge 1$, it is guaranteed that the distance of all PNNs for an unseen assortment to be less than $\lceil \log_2 N \rceil-1$ with high probability.
Under random splitting instead of the Gini index, then for $M = \lceil2^N/N\rceil$, the average distance is bounded by $\log_2 \ell +\log_2 N + 2.56$.
When the firm can design assortments to minimize the distance, $M = O(2^N / N)$ assortments are sufficient to guarantee that every other assortment has a PNN of distance 1. }

\subsubsection{Continuity of DCMs}\label{sec:dcm-continuity}
Having established bounds on the distance to PNNs of an unseen assortment, we next explore the continuity of DCMs.
The continuity is a crucial property for DCMs: the estimated choice probabilities for an assortment appearing frequently in the training data can be used to extrapolate an unseen assortment, only if choice probabilities do not vary significantly when the assortment changes slightly.
However, the notion of continuity is specific to DCMs and deviates from the literature on random forests.
First, the $\bm x$-space of our problem is not continuous and consists of extreme points of a hypercube.
Second, the ``$y$'' variable is a vector of choice probabilities.

To formalize the notion, define the following quantity between the choice probabilities of two assortments $S_1, S_2$ under a given DCM $p(i,S)$:
\begin{equation}
	\Phi(S_1, S_2) \triangleq \frac{\sum_{i \in [N]_+} \Big|p(i, S_1) - p(i, S_2)\Big|}{d(S_1,S_2)}.
\end{equation}
The quantity is similar to $(f(x)-f(y))/(x-y)$ for continuous functions: how close are their choice probabilities relative to the distance of $S_1$ and $S_2$?
If $\Phi(S_1, S_2)$ can be bounded, then the DCM is more or less ``Lipschitz continuous'' and the proximity of the unseen assortments to PNNs (results in Section~\ref{sec:distance}) leads to a good performance of random forests.
\begin{definition}\label{def:dcm-continuity}
	The DCM is $c$-continuous if for all $\emptyset \subsetneq S_1,S_2 \subset [N]$ and $S_1 \neq S_2$,
	\begin{equation}
		\label{eq:c-continuous}
		\Phi(S_1,S_2) \leq c / N.
	\end{equation}
\end{definition}
By the triangular inequality, it is sufficient to show a DCM is $c$-continuous if \eqref{eq:c-continuous} holds for all $S_1, S_2$ such that $d(S_1, S_2) = 1$.
If a DCM is $c$-continuous, then roughly speaking, the difference in the choice probabilities of two neighboring assortments with distance one is $c/N$.
Combining the results with Section~\ref{sec:distance}, if the distance between two assortments is $O(\log N)$ (e.g., Proposition~\ref{thm:dis_upper_bound}), then the error in the extrapolation of the choice probabilities is $O(\log N/N)$.
We next consider the continuity of popular DCMs.

\emph{The MNL model.}
Suppose $p(i,S)= v_i/(v_0+\sum_{j\in S}v_j)$ where $\{v_l\}_{l \in [N]_+}$ represents the attraction of the products.
Then for $\emptyset \subsetneq S \subsetneq [N]$ and $j \notin S$ we have that
\begin{equation*}
	\Phi(S, S \cup \{j\}) = \frac{2v_j}{v_0+v_j+\sum_{l \in S} v_l},
\end{equation*}
If the size of $S$ is $\epsilon N$, then it is easy to verify that $\Phi(S, S \cup \{j\})\le 2v_{\max}/(v_0+v_{\max}+\epsilon v_{min}N)$.

\emph{Rank-based DCMs. }
Suppose $\pi$ is a permutation of $[N]$ and $\pi(i)$ denotes the rank of product $i$ in $\pi$.
Define $\pi^*(S)\triangleq \argmin_{i \in S_+} \pi(i)$ to be the top choice in $\pi$ when $S$ is offered.
The rank-based model is represented by $w_{\pi}$, the weight of customers whose preference is consistent with $\pi$ in the population.
One can show that
\begin{equation*}
	\Phi(S, S \cup \{j\}) = \sum_{j = \pi^*(S \cup \{j\})} w_\pi.
\end{equation*}
If the fraction of customers who rank $j$ as the top choice in $S\cup \left\{j\right\}$ is small, then the DCM is more continuous according to Definition~\ref{def:dcm-continuity}.

\subsubsection{Sampling Error}\label{sec:sampling}
Another source of error stems from the empirical distribution used to estimate the choice probabilities of the assortments in the training data.
In the terminal leaf node, suppose a tree uses the choice probabilities of $S'$ to predict those of the unseen $S$.
Moreover, let $k$ denote the number of observations for $S'$ in the training data.
For each $i \in S \cap S'$, the frequency of customers choosing $i$ from the assortment $S'$ can be approximated by a normal distribution with mean $p(i,S')$ as the sample size increases.
The standard deviation is bounded by $\sqrt{p(i,S')(1-p(i,S'))/k} \leq 1/(2\sqrt{k})$.
Fortunately, the sampling error is more or less independent of the other two sources of errors articulated in Section~\ref{sec:distance} and~\ref{sec:dcm-continuity} and can be controlled using the standard concentration inequalities.
In particular, when an assortment has $O(|S|^2 \cdot N^2 / (\log N)^2)$ samples, the error in using the frequencies to approximate the choice probabilities is at most $O(\log N / (N \cdot |S|))$.

{\revise
\subsubsection{Combining the Errors}\label{sec:total-error}
In this section, we provide a unified bound combining the three sources mentioned above.
It provides a finite sample result for the performance of random forests predicting unseen assortment.
The error of prediction in the following theorem is defined as $\sum_{i \in [N]_+} \big|p(i, S) - \hat{p}(i, S)\big|$, where $S$ is the unseen assortment and $\hat{p}(i, S)$ is the estimated probability by random forests. 
\begin{theorem}
	\label{thm:total_error}
	\Copy{rev:thm2} {Suppose the DCM satisfies $c$-continuity and $\ell$ is the maximum number of distinct assortments in Algorithm~\ref{alg:rf}. $c_1 > 0$ is a constant.
	\begin{itemize}
		\item Suppose that $\ell \le c_0 \cdot \log_2 (N)$, where $c_0 \ge 1$ is a constant. If we draw $\Big\lceil2^{N+2} \cdot c_0 \cdot \lceil\log_2 N\rceil \cdot \log N / (N-2)\Big\rceil$ assortments with replacement in the training data and each assortment $S'$ has at least $\lceil\frac{N^3 \cdot (|S'|+1)^2}{c_1^2(\log_2 N)^2}\rceil$ transactions, then the error of predicting an unseen assortment using Algorithm~\ref{alg:rf} is bounded by $(c+c_1) (\log_2 N) / N$ with probability no less than $1-1/N-1/(\lceil \log_2 N \rceil)!$;
		\item If $m=1$ (random splitting) and we draw $\lceil 2^N / N\rceil$ assortments with replacement and each assortment $S' $ has at least $\lceil\frac{N^2 \cdot (|S'|+1)}{ c_1^2(\log_2 N)^2}\rceil$ transactions for each assortment in the training data, then the expected error of predicting an unseen assortment using Algorithm~\ref{alg:rf} is bounded by $\Big((c+c_1) \cdot \log_2 N + c \log_2 \ell + 2.56c\Big) / N$.
	\end{itemize}}
\end{theorem}
Roughly speaking, when the maximum number of distinct assortments in a leaf node is $O(\log N)$, the number of assortments in the training data is $\tilde O(2^N/N)$
and the transactions of each assortment is $O(N^2)$,
the estimation error of random forests is $\tilde O(1/N)$.
In Theorem \ref{thm:total_error}, we assume that only one assortment from the training data is used to predict an unseen assortment. However, if there are $\ell' \le \ell$ distinct assortments in the same leaf node, the error bound still holds, and the number of transactions required for each assortment $S'$ is reduced to $1/ \ell'$ of that required in Theorem \ref{thm:total_error}.
Note that the result only provides an upper bound for the error and we have seen much better performance of random forests in practice. 

\label{page:thm2_compare}
\Copy{rev:thm2_compare}{ 
\cite{chen2019decision} use mixed-integer programming to show that decision forests can represent historical assortments, whereas Theorem \ref{thm:total_error} analytically establishes an error bound for \emph{unseen} assortments.
Theorem \ref{thm:total_error} differs from existing random forests literature in several key aspects: (1) prior work studies regression for continuous $\bm{x} \in [0,1]^d$, whereas we address classification for discrete $\bm{x} \in \{0,1\}^N$; (2) instead of assuming continuity, positive density, and honest trees, we introduce $c$-continuity tailored to DCMs; (3) we bound prediction error for unseen assortments in classification, $\sum_{i \in [N]_+} |p(i, S) - \hat{p}(i, S)|$, while prior work focuses on MSE for regression; (4) methodologically, we exploit the combinatorial structure of DCMs and apply concentration inequalities to handle the discrete setting. }
}

\subsection{Splitting Criteria and the Recovery of Rank-based DCMs} \label{sec:gini_rank_list}
%{\revise \Copy{rev:sec3.3}{In this section, we examine the theoretical performance of different splitting criteria for classification and highlight the interpretability of the random forest algorithm for a special case of DCMs, showing how it can capture the underlying customer decision-making process.}}
In this section, we discuss a few commonly used splitting criteria for Step~\ref{step:criterion} in Algorithm~\ref{alg:rf} and their implications in a theoretical framework.
	Suppose a parent node with dataset $\mathcal{D}$ splits into $K$ child nodes $\{\mathcal{D}_k\}_{k=1}^K$ along dimension $i$, corresponding to $K$ regions $\{R_k\}_{k=1}^K$, i.e., $\mathcal{D}_k = \{(i_t, \bm x_t): \bm x_t \in R_k\}$.
	In Algorithm~\ref{alg:rf}, we choose $K=2$ and $i$ is the product on which the split is performed, thus grouping observations based on whether product $i$ is included in the assortment.
	With a slight abuse of notation, let $T$ denote the number of samples in the parent node, denoted as $\mathcal{D}$, and $t_k$ the number of observations in the child node $R_k$. 
	Let $\hat{p}_{jk}$ denote the empirical purchase frequency of product $j$ in $R_k$, i.e., $\hat{p}_{jk} = \sum_{\bm x_t \in R_k} \I{i_t=j}/t_k$.
	Similarly, we use $\hat p_{j0}$ to denote the empirical purchase frequency in the parent node.
	
	The Gini index is one of the most widely used criteria, defined as 
	\begin{definition}[Gini index]\label{def:gini}
		\begin{equation*}
		\texttt{GI}(\mathcal{D}, i) \triangleq \sum_{k=1}^K \frac{t_k}{T} \sum_{j=0}^N \hat{p}_{jk} (1-\hat{p}_{jk}).
		\end{equation*}
	\end{definition}
	Replacing the purity measure in the Gini index $\hat{p}_{jk} (1-\hat{p}_{jk})$ by entropy-related measures, $H(\mathcal{D}_k) \triangleq -\sum_{j=0}^N \hat{p}_{jk} \ln(\hat{p}_{jk})$,
	we can obtain another common splitting criterion, the information gain, measuring the reduction of entropy by the split:
	\begin{definition}[Information Gain]\label{def:info-gain}
		\begin{equation*}
			\texttt{IG}(\mathcal{D}, i) \triangleq H(\mathcal{D}) - \sum_{k=1}^K \frac{t_k}{T} \cdot H(\mathcal{D}_k) 
			=-\sum_{j=0}^N \hat{p}_{j0} \ln(\hat{p}_{j0}) - \Big(-\sum_{k=1}^K \frac{t_k}{T}\sum_{j=0}^N \hat{p}_{jk} \ln(\hat{p}_{jk})\Big).
		\end{equation*}
	\end{definition}	
	A similar criterion normalizes the information gain by the number of samples in the child nodes.
	\begin{definition}[Information Gain Ratio]\label{def:info-gain-ratio}
		\begin{equation*}
		\texttt{IGR}(\mathcal{D}, i) \triangleq \frac{\texttt{IG}(\mathcal{D}, i)}{-\sum_{k=1}^K \frac{t_k}{T} \ln(\frac{t_k}{T})}
		\end{equation*}
	\end{definition}
	The above criteria use slightly different notions of the purity of a region.
	More details about the splitting criteria can be found in \citet{zhou2021machine}.
	
	To explore the impact of the splitting criteria, we consider a specific DCM that generates the data.
	In particular, we assume that the underlying DCM is a rank-based model characterized by a single \emph{preference ranking}: the consumers always prefer product $i$ to $i+1$, for $i=1,\dots,N-1$, and product $N$ to the no-purchase option.
	The following result demonstrates that the ranking can be recovered from the random forest using the Gini index with high probability.

{\revise
	\begin{theorem}\label{thm:single tree}
		\Copy{rev:thm3}{ Suppose the underlying DCM is a preference ranking and the assortments in the training data are sampled uniformly and independently: each assortment includes product $j$ with probability $1/2$ for all $j=1,\dots,N$.
		The random forest algorithm with sub-sample size $z=T$ (without replacement), $m=N$, terminal leaf size {\revise $l = 1.77 T / 2^k$} and $B=1$ correctly predicts the choices of more than $(1-\epsilon)2^N$ assortments with probability no less than
		\begin{equation*}
			1-\sum_{i=1}^k \left[13\exp\left(-\frac{T}{164\cdot 2^{i-1}}\right)+10(N-i-1)\exp\left(-\frac{T}{113 \cdot 2^{i-1}}\right)\right],
		\end{equation*}
		where $k=\lceil \log_2 \frac{1}{\epsilon} \rceil$.}
	\end{theorem}}
	Since the error probability converges to zero exponentially in $T$, the predictive accuracy improves tremendously with data size.
	The proof of the theorem reveals an intrinsic connection between the Gini index and the recovery of the preference ranking.
	We analyze the deterministic output of the \emph{theoretical} random forest when the data size is infinite.
	It allows us to show that the theoretical Gini index leads to a sequence of splits consistent with the ranking, i.e., the ranking is recovered under the theoretical random forest.
	For example, if the first split is on product $i$, then the resulting theoretical Gini index is
	$(2-2^{2-2i})/3$.
	In other words, the first split would occur on product one under the theoretical random forest.
	Then we analyze the difference between empirical and theoretical Gini indices and bound the probability of incorrect splits using concentration inequalities. The recursive splits are analyzed using the union bound.
	
	The proof provides the following insight into why random forests may work well in practice:
	The Gini index criterion tends to find the products that are ranked high in the rankings because they create ``purer'' splits that lower the Gini index.
	As a result, the topological structure of the decision trees trained in the random forest is likely to resemble that of the binary choice trees underlying the DCM generating the data.
	Moreover, we also complement the results in Theorem \ref{thm:single tree} by additional numerical studies in Online Appendix \ref{sec:additional-numerical}. We numerically show the result still holds when the training data is not uniform. We also provide the insights for random forests when the rank-based DCM consists of more than one rankings. %the random forest may output a tree that concatenates and merges multiple rankings.

	% \subsection{Information Gain Ratio Recovers the Ranking} \label{sec:IGR_rank_list}
	One restrictive assumption in Theorem \ref{thm:single tree}, reflecting the weakness of the Gini index,
	is the uniform sampling of the assortments. 
	In practice, it is likely that the most popular products could be offered more or less often than other products. 
	In this case, even though it is the most preferred product, the random forest under the Gini index may fail to recognize it as the first split.
	Next, we show that using the information gain ratio in Definition~\ref{def:info-gain-ratio} as the splitting criterion,
	Algorithm~\ref{alg:rf} can recover the preference ranking with a much milder assumption.
	\begin{theorem} \label{thm:information_gain_ratio}
		Suppose the underlying DCM is a preference ranking. 
		Suppose the training data meet the following conditions: for $i = 1, \ldots, N-1$,
		\begin{enumerate}
			\item there exists $(i_t,S_t)$ such that $i \in S_t$, $i+1 \in S_t$ and $\{1, \ldots, i-1\} \cap S_t = \emptyset$;
			\item there exists $(i_t,S_t)$ such that $i \in S_t$ and $\{1, \ldots, i-1, i+1\} \cap S_t = \emptyset$.
		\end{enumerate}
		Then, the random forest algorithm with information gain ratio as the splitting criterion, sub-sample size $z=T$ (without replacement), $m = N$, terminal leaf size $l = 1$ and $B=1$ recovers the preference ranking almost surely. 
	\end{theorem}
	Theorem~\ref{thm:information_gain_ratio} highlights the benign theoretical property of splitting using the information gain ratio. 
	The conditions in Theorem~\ref{thm:information_gain_ratio} are rather mild:
	we only require at least one sample satisfying each of the conditions.
	As a result, the minimum sample requirement for the training data is $2N-1$. 
	That is, the data includes the offering of individual products $\{i\}_{i=1}^N$ and consecutive pairs $\{(i, i+1)\}_{i=1}^{N-1}$.
	The following numerical example shows that the information gain ratio can recover the DCM, while other criteria cannot. 
	The numerical example provides insights into the benefit of using the information gain ratio as a splitting criterion.
	\begin{example}
		Consider there are $N = 3$ products and the preference ranking of the consumers is $1 \succ 2 \succ 3 \succ 0$. 
		We generate 100 samples, shown in Table~\ref{tab:dataset_example}. We calculate the three criteria to determine the dimension of the first split in Table~\ref{tab:dataset_example}. 
		To recover the preference ranking, the first split should be performed on product one.
		However, both \texttt{GI} and \texttt{IG} choose product two as the first split, while \texttt{IGR} leads to a split on product one correctly.
		This is because the samples with an assortment showing product one is preferable to product two, $\{1,2\}$ or $\{1,2,3\}$, are relatively rare in the dataset (one observation).
		If split on product one, the two child nodes are extremely imbalanced (two observations in the child node with $\bm x_t(1)=1$).
		As a result, it doesn't contribute much to \texttt{GI} and \texttt{IG}.
		On the other hand, \texttt{IGR} normalizes the number of observations (see Definition~\ref{def:info-gain-ratio}) and offsets the imbalanced observations.
		In this example, the information gain ratio is the only criterion that recovers $1 \succ 2 \succ 3 \succ 0$.
		\begin{table}
			\caption{The dataset and the criteria of the first split. Note that for $\texttt{IG}$ and $\texttt{IGR}$, the algorithm selects the maximal to split, while for $\texttt{GI}$, it selects the minimal.}
			\begin{center}
				\def\arraystretch{1}\begin{tabular}{cccccc} 
					\hline
					Assortment & \{1\} & \{1,2\} & \{2\} & \{2,3\} & \{3\} \\
					\hline
					Choice & 1 & 1 & 2 & 2 & 3\\
					\hline
					\# of Samples & 1 & 1 & 25 & 24 & 49 \\
					\hline
				\end{tabular}
				\hspace{0.5in}
				\def\arraystretch{0.9}\begin{tabular}{crrr} 
					\hline
					Split  & $\texttt{IG}$ & $\texttt{GI}$ & $\texttt{IGR}$ \\
					\hline
					Prod. 1 & 0.0980 & 0.4900& \textbf{1.0000}  \\
					Prod. 2 & \textbf{0.6793} & \textbf{0.0392}& 0.9800  \\
					Prod. 3 & 0.2437 & 0.3592& 0.4179  \\
					\hline
				\end{tabular}
				\label{tab:dataset_example}
			\end{center}
			\vspace{-4mm}
		\end{table}
	\end{example}
	
Since the actual ranking is unknown, the condition in Theorem~\ref{thm:information_gain_ratio} that requires the pairs of $i$th and $(i+1)$th preferred products to be offered in an assortment may not be verifiable.
		A sufficient condition for Theorem~\ref{thm:information_gain_ratio} is that all pairs of products are offered in an assortment, as well as individual products. 
		This corresponds to at least ${N(N+1)}/{2}$ assortments in the training data to recover the preference ranking.
		On the other hand, this is almost necessary.

		\begin{proposition} \label{prop:IGR_unknown_rank}
			If there are less than ${N(N-1)}/{2}$ assortments offered in the training data, then there exist two preference rankings that exhibit the same choice behavior consistent with the training data, and thus they cannot be distinguished or recovered by any algorithm.
	\end{proposition}

\label{page:thm34_discussion}

{\revise \Copy{rev:sec3.3}{
In summary, Theorems \ref{thm:single tree} and \ref{thm:information_gain_ratio} in this section highlight the interpretability of the random forest algorithm and the theoretical performance of different splitting criteria for classification problems. They show that random forests can capture the underlying decision-making process of a single preference ranking customer, a feature not studied in the random forests literature. Unlike most literature studies on regression that use MSE as the splitting criterion, we examine the Gini index and information gain ratio for classification problems. Theorem \ref{thm:information_gain_ratio} demonstrates that the information gain ratio, rarely studied before, can recover the single preference ranking model with only $2N - 1$ samples.}}

\section{Flexibility and Practical Benefits of Random Forests}\label{sec:practical-considerations}
In this section, we demonstrate the flexibility of random forests and how the method can be adapted in practice to handle different situations.
\subsection{Behavioral Issues}\label{sec:behavoral}
Theorem~\ref{thm:consistency} shows random forests can estimate any DCMs.
For example, there is empirical evidence showing that behavioral considerations of consumers may distort their choice, e.g., the decoy effect \citep{ariely2008predictably}, comparison-based choices \citep{huber1982adding,russo1983strategies}, search cost \citep{weitzman1979optimal} and context effects \citep{yousefi2020choice}.
It implies that regular (see Section~\ref{sec:literature}) DCMs cannot predict the choice behavior well.
It is already documented in \citet{chen2019decision} that the decision forest can capture the decoy effect.
In this section, we use choice forests to model consumer search.

\citet{weitzman1979optimal} proposes a sequential search model with search costs.
Before the search process, consumers only know the distribution of $V_j$, the net utility of product $j \in [N]$, and the cost $c_j$ to learn the realization of $V_j$.
Let $z_j$ be the root of the equation $E[(V_j - z_j)^+]  = c_j$ and suppose that products are sorted in the descending order of the $z_j$s.
\citet{weitzman1979optimal} shows that it is optimal not to purchase if the realized value of the no-purchase alternative, $ V_0$, exceeds $z_1$. Otherwise, the consumer searches product one at a cost $c_1$ and $W_1 = \max(V_1, V_0)$ is computed.
The search process stops when $W_i$ exceeds $z_{i+1}$ for the first time, with the consumer selecting the best product among those that were searched.

We next show that binary choice trees can represent this search process.
Consider three products ($N=3$).
Suppose the products are sorted so that $z_1>z_2>z_3>0$, so the consumer searches in the order of product one $\to$ product two $\to$ product three. Suppose an arriving customer has realized utilities satisfying  $v_2> z_3 > v_1>v_3$. Then the decision process is illustrated by the tree in Figure~\ref{fig:search}.
For example, suppose products $\{1,3\}$ are offered. The customer first searches product 1 because the reservation price of product one $z_1$ is the highest. However, the realized valuation of product 1 is not satisfactory  ($v_1<z_3 < z_1$). Hence, the customer keeps searching for the product with the second-highest reservation price in the assortment, product 3.
However, the search process results in an even lower valuation of product 3, i.e., $v_3<v_1$.
As a result, the customer stops and chooses product one.
Clearly, a customer with different realized valuations would conduct a different search process, corresponding to another decision tree.
%Our tree structure can also represent the choice behavior when assortments are sequentially offered to customers \citep{liu2020assortment}. 
\begin{figure}[t]
	\centering
	\caption{The sequential search process when $N=3$ and the realized valuations and reservation prices satisfy $v_2>v_1>v_3$, $z_1>z_2>z_3>0$ and $v_2>z_3>v_1$.}
	\scalebox{0.8}{
	\begin{forest}
		for tree={l sep+=.5cm,s sep+=.2cm,shape=rectangle, rounded corners,
			draw, align=center,
			top color=white, bottom color=gray!20}
		[Has product 1
		[Has product 2, edge label={node[midway,left]{Y}}
		[Choose 2,edge label={node[midway,left]{Y}}]
		[Has product 3,edge label={node[midway,right]{N}}
		[Choose 1,edge label={node[midway,left]{Y}}]
		[Choose 1,edge label={node[midway,right]{N}}]
		]
		]
		[Has product 2,edge label={node[midway,right]{N}}
		[Choose 2,edge label={node[midway,left]{Y}}]
		[Has product 3,edge label={node[midway,right]{N}}
		[Choose 3,edge label={node[midway,left]{Y}}]
		[No purchase, edge label={node[midway,right]{N}}]
		]
		]
		]
	\end{forest}
	}
	\label{fig:search}
\end{figure}
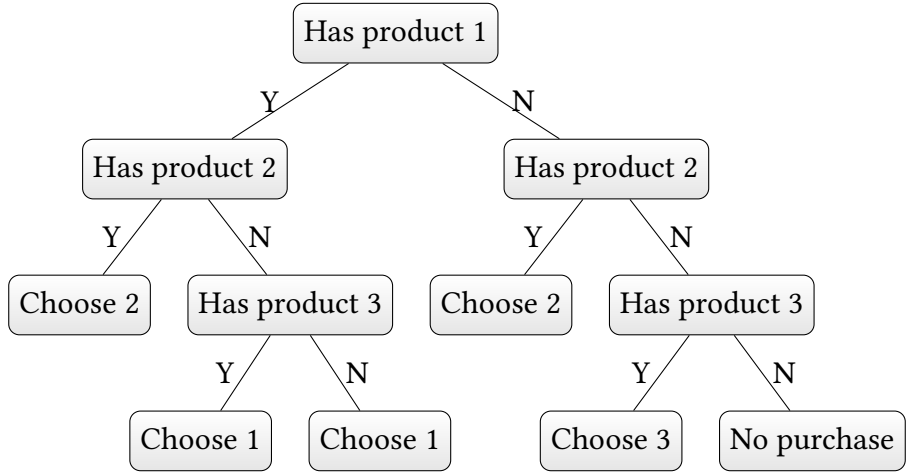

{\revise 
\subsection{Product Importance}\label{sec:importance}
\Copy{rev:MDI}{Random forests can be used to assign scores to each product and rank the importance of products.
A common score, mean decrease impurity (MDI), is based on the total decrease in node impurity from splitting on the product, averaged over all trees \citep{biau2016random}.
The score for product $m$ is defined as
\begin{equation*} \begin{aligned}
	\text{MDI}(m) &= \frac{1}{B} \sum_{b=1}^B \sum_{\substack{\text{all splits } s\\\text{ in the $b$th tree}}} (\text{fraction of data in the parent node of } s)\\
	&\quad \times(\text{reduction in the Gini index caused by }s)\times\I{s\text{ splits on }m}.
\end{aligned} \end{equation*}
In other words, if consumers make decisions frequently based on the presence of product $m$ (a lot of splits occur on product $m$), or
their decisions are more consistent after observing the presence of product $m$ (the Gini index is reduced significantly after splitting on $m$), then the product gains more score in MDI and is regarded as important.
To illustrate this measure, we provide examples in Online Appendix~\ref{sec:prod-imp-exp}.

The identification of important products provides simple yet powerful insights into the behavioral patterns of consumers.
Consider the following use cases:
(1) An retailer wants to promote its ``flagship'' products that significantly increase the conversion rate.
By computing the MDI from the historical data, important products can be identified without extensive A/B testing.
(2) Due to limited capacity, a firm plans to reduce the available types of products to cut costs.
According to the historical data, it could simply remove the products with low sales.
However, some products, while not looking attractive themselves, serve as decoys or references and boost the demand for other products.
Removing these products would distort consumers' choice behavior and may lead to unfavorable consequences.
The importance score provides an ideal solution: if a product is ranked low based on MDI, then it does not strongly influence the decision-making of consumers.
Therefore, it is safe to leave them out.
(3) When designing a new product, a firm attempts to decode the impact of various product features on customer choices.
Which product feature is drawing the most attention?
What do attractive products have in common?
To conduct successful product engineering, first, it needs to use historical data to nail down a set of attractive products.
Moreover, a numerical score of product importance is necessary to quantify and separate the contribution of various features.
The importance score is a more reasonable criterion than sales volume because the latter cannot capture the synergy between the products.
}}

{\revise
\subsection{Incorporating Price Information}\label{sec:price}
\Copy{rev:pricing}{One benefit of a parametric DCM, such as the MNL or nested logit model, is the ability to account for covariates.
For example, in the MNL model, the firm can estimate the price sensitivity of each product, and extrapolate/predict the choice probability when the product is charged at a new price that has never been observed in the historical data.
Many nonparametric DCMs cannot easily be extended to new prices.
In this section, we show that while enjoying the benefit of a nonparametric formulation, random forests can also accommodate the price information.

Consider the data of the following format: $\left\{(i_t,\bm p_t)\right\}_{t=1}^T$, where $\bm p_t\in [0,+\infty]^N$ represent the prices of all products.
For product $j$ not included in the assortment offered to customer $t$, we set $\bm p_t(j)=+\infty$.
This is because when a product is priced at $+\infty$, no customer would be willing to purchase it, and it is equivalent to the scenario that the product is not offered at all.
Therefore, compared to the binary vector $\bm x_t$ that only records whether a product is offered, the price vector $\bm p_t$ encodes more information.

However, the predictor $\bm p$ can not be readily used in random forests.
The predictor space $[0,+\infty]^{N}$ is unbounded, and the value $+\infty$ added to the extended real number line is not implementable in practice.
To apply Algorithm~\ref{alg:rf}, we introduce link functions that map the input into a compact set.
\begin{definition}\label{def:link-func}
	A function $g(\cdot):[0,+\infty)\mapsto (0,1]$ is referred to as a link function, if (1) $g(x)$ is strictly decreasing, (2) $g(0)=1$, and (3) $\lim_{x\to +\infty} g(x)=0$.
\end{definition}
The link function can be used to transform a price $p\ge 0$ into $(0,1]$.
Moreover, because of property (3), we can naturally define $g(+\infty)=0$.
Thus, if product $j$ is not included in assortment $\bm x_t$, then $g(\bm p_t(j))=g(+\infty)=0=\bm x_t(j)$.
If product $j$ is offered at the low price, then $g(\bm p_t(j))\approx g(0)= 1$.
After the transformation of inputs, $\bm p_t\to g(\bm p_t)$\footnote{When $g(\cdot)$ is applied to a vector $\bm p$, it is interpreted as applied to each component of the vector.}, we introduce a continuous scale to the problem in Section~\ref{sec:data_estimation}.
Instead of binary status (included or not), each product now has a spectrum of presence, depending on the price of the product.
Now we can directly apply Algorithm~\ref{alg:rf} to the training data $\left\{(i_t,g(\bm p_t))\right\}_{t=1}^T$ after modifying Step~\ref{step:choose-direction}, because the algorithm needs to find not only the optimal product to split but also the optimal split location.
The slightly modified random forests are demonstrated in Algorithm~\ref{alg:rf-pricing}.}
{\SingleSpacedXI
	\begin{algorithm}
		\caption{\revise Random forests for DCM estimation with price information}\label{alg:rf-pricing}
		\begin{algorithmic}[1]
			\revise
			\State Data: $\left\{(i_t,\bm p_t)\right\}_{t=1}^T$
			\State Tunable parameters: number of trees $B$, sub-sample size $z\in\left\{1,\dots,T\right\}$,
			number of products to split $m\in \left\{1,\dots,N\right\}$, terminal leaf size $l\in\{1,\dots,z\}$, a link function $g(\cdot)$\label{step:parameter-p}
			\State Transform the training data to $\left\{(i_t,g(\bm p_t))\right\}_{t=1}^T$
			\For{$b=1$ to $B$}
			\State Select $z$ observations from the training data with replacement, denoted by $Z$\label{step:select-subsample-p}
			\State Initialize the tree $t_b(g(\bm p))\equiv 0$ with a single root node
			\While{some leaf has greater than or equal to $l$ observations belonging to $Z$ and can be split}\label{step:while-split-p}
			\State Select $m$ products without replacement among $\left\{1,\dots,N\right\}$\label{step:choose-direction-p}
			\State Select the optimal one among the $m$ products and the optimal position to split that minimize the Gini index \label{step:criterion-p}
			\State Split the leaf node into two
			\EndWhile
			\State Denote the partition associated with the leaves of the tree by $\left\{R_1,\dots,R_M\right\}$;
			let $c_i$ be the class label of a randomly chosen observation in $R_i$ from the training data\label{step:random-label-p}
			\State Define $t_b(g(\bm p)) = \sum_{i=1}^{M} c_i\I{g(\bm p)\in R_i}$
			\EndFor
			\State The choice probability of product $i$ given price vector $\bm p$ is $\sum_{b=1}^B \frac{1}{B}\I{t_b(g(\bm p))=i}$
		\end{algorithmic}
\end{algorithm}
\vspace{-2mm}
}

\Copy{rev:pricing2}{
Because of the nature of the decision trees, the impact of prices on the choice behavior is piecewise constant.
For example, Figure~\ref{fig:pricing} illustrates a possible binary choice tree with $N=3$.}
\begin{figure}[t]
	\centering
	\caption{\revise A possible decision tree when the price information is incorporated for $N=3$. $g(\bm p(i))>a$ is equivalent to $\bm p(i)<g^{-1}(a)$, i.e., product $i$ is included in the assortment and its price is less than $g^{-1}(a)$.}
	\scalebox{0.8}{
	\begin{forest}
		for tree={l sep+=.5cm,s sep+=.2cm,shape=rectangle, rounded corners,
			draw, align=center,
			top color=white, bottom color=gray!20}
		[$g(\bm p(1))>0.3$
		[$g(\bm p(1))>0.9$, edge label={node[midway,left]{Y}}
		[1,edge label={node[midway,left]{Y}}]
		[$g(\bm p(2))>0.5$,edge label={node[midway,right]{N}}
		[2,edge label={node[midway,left]{Y}}]
		[0,edge label={node[midway,right]{N}}]
		]
		]
		[$g(\bm p(3))>0.4$,edge label={node[midway,right]{N}}
		[$g(\bm p(3))>0.8$,edge label={node[midway,left]{Y}}
		[3,edge label={node[midway,left]{Y}}]
		[0,edge label={node[midway,right]{N}}]
		]
		[$g(\bm p(2))>0.3$,edge label={node[midway,right]{N}}
		[2,edge label={node[midway,left]{Y}}]
		[0,edge label={node[midway,right]{N}}]
		]
		]
		]
	\end{forest}
	}
	\label{fig:pricing}
\end{figure}
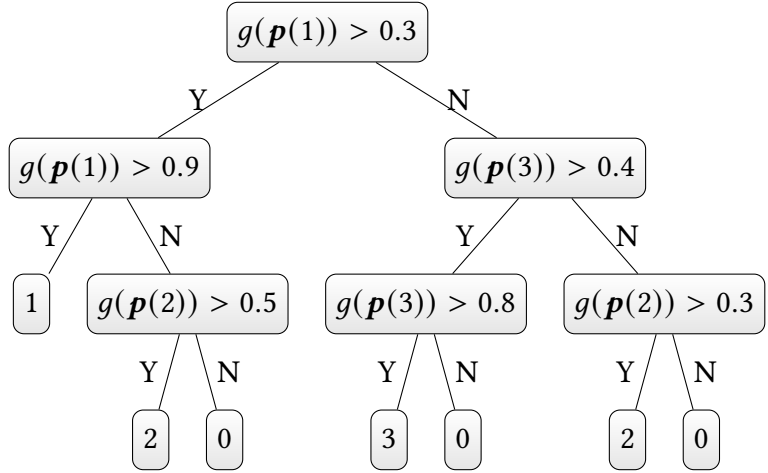

\Copy{rev:pricing3}{
It is not surprising that there are numerous link functions to choose from.
For instance, $g(x) = e^{-x}$ or $g(x) = 1- \frac{2}{\pi}\arctan(x)$.
In fact, the survival function of any non-negative random variables with positive PDF is a candidate for the link function.
This extra degree of freedom may concern some academics and practitioners:
How sensitive is the estimated random forest to the choice of link functions?
What criteria may be used to pick a ``good'' link function?
Our next result guarantees that the choice of link functions does not matter.
For any two link functions $g_1(x)$ and $g_2(x)$, we can run Algorithm~\ref{alg:rf-pricing} for training data $\left\{(i_t,g_1(\bm p_t))\right\}_{t=1}^T$ and $\left\{(i_t,g_2(\bm p_t))\right\}_{t=1}^T$.
We use $t_b^{(j)}(\bm x)$ to denote the returned $b$th tree of the algorithm for link function $g_j(x)$, $j=1,2$.}
\begin{proposition}\label{prop:link-func}
	\Copy{rev:prop_link_func}{If we equalize
	\begin{itemize}
		\item the choice of parameters in Step~\ref{step:parameter-p} except for the link function
		\item the internal randomizers in Step~\ref{step:select-subsample-p}, \ref{step:choose-direction-p}, and \ref{step:random-label-p}
	\end{itemize}
	in Algorithm~\ref{alg:rf-pricing}, then the trees of both link functions return the same class label for an observation in the training data: $t_b^{(1)}(g_1(\bm p_t)) = t_b^{(2)}(g_2(\bm p_t)) $ for all $t=1,\dots,T$ and $b=1,\dots,B$.}
	%    Given the same training data and choice of parameters in Step~\ref{step:parameter}, Algorithm~\ref{alg:rf} outputs the same classifier for any link function if the internal randomizers in Step~\ref{step:select-subsample}, \ref{step:choose-direction}, and \ref{step:random-label} have the same realization.
\end{proposition}
\Copy{rev:pricing4}{
It is worth pointing out that although the random forests using two link functions output identical class labels for $\bm p_t$ in the training data,
they may differ when predicting a new price vector $\bm p$.
This is because the splitting operation that minimizes the Gini index in Step~\ref{step:choose-direction-p} is not unique.
Any split between two consecutive observations\footnote{If the algorithm splits on product $m$, then $\bm p_{t_1}$ and $\bm p_{t_2}$ are consecutive if there does not exist $\bm p_{t_3}$ in the same leaf node such that $(\bm p_{t_1}(m)-\bm p_{t_3}(m))(\bm p_{t_2}(m)-\bm p_{t_3}(m))<0$.} results in an identical class composition in the new leaves and thus the same Gini index.
Usually, the algorithm picks the middle between two consecutive observations to split, which may differ for different link functions. %Moreover, as we show in Online Appendix~\ref{sec:customer_features}, random forests algorithm can also incorporate customer features.
}}

{\revise 
\subsection{Incorporating Customer Features} \label{sec:customer_features}
\Copy{rev:feature1}{A growing trend in online retailing and e-commerce is personalization.
Due to the increasing access to personal information and computational power, retailers are able to implement personalized policies, including pricing and recommendation, for different customers based on his/her observed features.
Leveraging personal information can greatly increase the garnered revenue of the firm.

The first step to offering a personalized assortment is incorporating the feature information into the choice model.
It has been considered in many classic DCMs by including a linear term in the features. See \citet{train2009discrete} for a general treatment.
In this section, we demonstrate that it is natural for random forests to capture customer features and return a binary choice forest that is aware of such information.
Suppose the collected data of the firm have the form $(i_t, \bm x_t, \bm f_t)$ for customer $t$, where in addition to $(i_t, \bm x_t)$, the choice made and the offered set, the customer feature $\bm f_t\in [0,1]^M$ is also recorded (possibly normalized).
The procedure in Section~\ref{sec:data_estimation} can be extended naturally.
In particular, we may append $\bm f_t$ to $\bm x_t$, so that the predictor $(\bm x,\bm f)\in [0,1]^{M+N}$.
Algorithm~\ref{alg:rf} can be modified accordingly.

The resulting binary choice forest consists of $B$ binary choice trees.
The splits of the binary choice tree now encode not only whether a product is offered, but also predictive feature information of the customer.
For example, a possible binary choice tree illustrated in Figure~\ref{fig:tree-model} may result from the algorithm.}

\begin{figure}[h]
	\centering
	\caption{\revise A possible binary choice tree after incorporating customer features.}
	\begin{forest}
		for tree={l sep+=.5cm,s sep+=.2cm,shape=rectangle, rounded corners,
			draw, align=center,
			top color=white, bottom color=gray!20}
		[Has product 1
		[Has product 3, edge label={node[midway,left]{Y}}
		[Choose 3,edge label={node[midway,left]{Y}}]
		[Choose 1,edge label={node[midway,right]{N}}]
		]
		[Age $\ge 30$,edge={dashed}, edge label={node[midway,right]{N}}
		[Married,edge={dashed},edge label={node[midway,left]{Y}}
		[Choose 4,edge={dashed},edge label={node[midway,left]{Y}}]
		[Choose 2,edge label={node[midway,right]{N}}]
		]
		[No purchase, edge label={node[midway,right]{N}}]
		]
		]
	\end{forest}
	\vspace{-4mm}
	\label{fig:tree-model}
\end{figure}
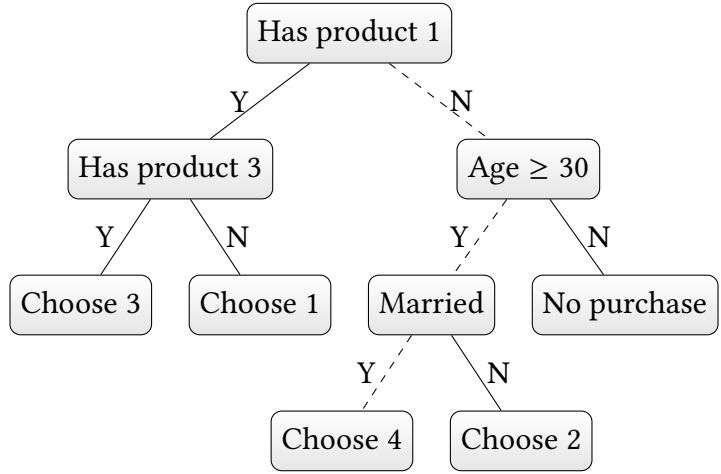
\Copy{rev:feature2}{
Compared with other DCMs with linear features,
the framework introduced in this paper has the following benefits:
(1) The estimation is straightforward (same as the algorithm without customer features) and can be implemented efficiently.
(2) The nonparametric nature of the model allows for capturing complex interaction between products and customer features, and among customer features. For example, ``offering a high-end handbag'' may become a strong predictor when the combination of features ``female'' and ``age$\ge 30$'' are activated.
In a binary choice tree, the effect is captured by three splits (one for the product and two for the customer features) along a branch.
It is almost impossible to capture in a parametric (linear) model.
(3) The framework can be combined with the aforementioned adjustments, such as pricing and product importance. For example, the measure MDI introduced in Section~\ref{sec:importance} can be used to identify predictive customer features.	}}

\section{Numerical Experiments}\label{sec:numerical}
{\revise
We conduct a comprehensive numerical study based on both synthetic and real datasets.
We find that (1) random forests are quite robust and the performance is stable for different underlying DCMs varying in complexity. 
In particular, random forests only underperform the correctly specified parametric models by a small margin and do not overfit;
(2) the standard error of random forests is small compared to other estimation procedures;
(3) random forests scale well to hundreds of products without requiring excessively large training samples; 
(4) random forests benefit tremendously from increasing sample size compared to other DCMs;
(5) the computation time of random forests almost does not scale with the size of the training data;
(6) random forests have a robust performance even if the training set only includes less than $1/30$ of all available assortments;
(7) random forests handle training data with nonstandard format reasonably well, such as price information and aggregated data (see Section~\ref{sec:price} and Appendix~\ref{sec:aggregate-choice} and for more details), which cannot be handled easily by other frameworks.}

For the experiments, we compare the estimation results of random forests with the MNL model, the Markov chain model \citep{blanchet2016markov}, and the decision forest \citep{chen2019decision}\footnote{The MNL model is estimated using the maximum likelihood estimator. The Markov chain model is estimated using the EM algorithm, the same as the implementation in \citet{csimcsek2018expectation}. The random forest is estimated using the Python package ``scikit-learn''.}. 
We mainly choose the MNL and the Markov Chain models as benchmarks because the MNL model is one of the most widely used DCM and
the Markov chain model has been shown \citep{berbeglia2022comparative} to have outstanding empirical performance compared to MNL, the nested logit, the mixed logit, and rank-based DCMs.
For the decision forest, our implementation is based on the code repository of \cite{chen2019decision} on GitHub. %However, we update the package versions and make a number of modifications to run the code successfully.

When conducting numerical experiments, we set the hyper-parameters of the random forest as follows: $B=1000$, $z=T$, $m=\sqrt{N}$, $l=50$.
The investigation of the sensitivity to hyper-parameters is shown in Online Appendix~\ref{sec:hyper-parameter}.
Choosing the hyper-parameters optimally using cross-validation would further improve the performance of random forests.

\subsection{Real Data: IRI Academic Dataset} \label{sec:IRI_numerical}
In this section, we compare different models on the IRI Academic Dataset \citep{bronnenberg2008database}.
The IRI Academic Dataset collects weekly transaction data from 47 U.S. markets from 2001 to 2012,
covering more than 30 product categories.
Each transaction includes the week and the store of purchase, the universal product code (UPC) of the purchased item, the number of units purchased, and the total paid dollars.

The preprocessing follows the same steps as in \citet{jagabathula2018limit} and \citet{chen2019decision}.
In particular, we regard the products sharing the same vendor code as the same product.
Each assortment is defined as a unique combination of stores and weeks. Such an assortment includes all the products available in the store during that week.
We conduct the analysis for 31 categories separately using the data for all weeks in 2007.
We only focus on the top nine purchased products from all stores in each category throughout the year and treat all other products as the no-purchase alternative.

All sales data after preprocessing has the form $\{(i_t, S_t)\}_{t=1}^T$, where $S_t$ is the offered product set and $i_t \in S_t \cup \{0\}$ is the purchased product in period $t$. Let $\mathcal{S}$ denote the set of unique assortments.
The dataset can be aggregated by the same assortment $S$ that $\{n_t, \bm \tilde\PR_t, S_t\}_{t=1}^{|\mathcal{S}|}$, where $n_t$ is the total sales under assortment $S_t$, and $\bm \tilde\PR_t$ is the empirical choice probability for each product in $S_t \cup \{0\}$.
We only focus on assortments with more than 30 transactions, i.e., $n_t \geq 30$.
However, the sales data for most categories are too large for the EM algorithm to estimate the Markov chain model.
For computational efficiency, we sample 30 samples with replacement for each unique assortment.
This re-sampling procedure does not significantly increase the sampling variability, as most transactions in the original data are repeated entries.

To compare different estimation procedures, we use five-fold cross-validation for each unique assortment to examine the out-of-sample performance.
We follow \cite{chen2019decision} and divide the unique assortments into five (approximately) equal-sized subsets $\mathcal{S}_1, \ldots, \mathcal{S}_5$.
For each $i = 1, \ldots, 5$, the testing set is $\mathcal{S}_i$ and the remaining four subsets serve as the training data. 
Following \cite{berbeglia2022comparative}, the empirical performance is evaluated by the empirical root mean squared error (RMSE) in the testing set.
That is, for estimated choice probabilities $\hat\PR$ and testing set $\mathcal{S}_i$, we define
\begin{equation}
	\label{RMSE_realdata}
	R M S E\left(\hat\PR, \tilde\PR, \mathcal{S}_i\right)=\sqrt{\frac{\sum_{S \in \mathcal{S}_i} \sum_{j \in S \cup\{0\}}\left(\tilde\PR(j, S)-\hat\PR(j , S)\right)^{2}}{\sum_{S \in \mathcal{S}_i}(|S|+1)}}.
\end{equation}

The result is shown in Table~\ref{tab:IRI KFold result} comparing random forests, MNL, the Markov chain model and the decision forest (based on column generation). Random forests outperform the others in 20 of 31 categories.
According to \citet{berbeglia2022comparative}, the Markov chain choice model has already shown strong performance in synthetic and real-world studies.
Table~\ref{tab:IRI KFold result} fully demonstrates the potential of random forests as a framework to model and estimate consumer behavior in practice.
\label{page:IRI}{\revise \Copy{rev:IRI}{For robustness, we also test the result when considering the top seven/fifteen products instead of nine.
The results are shown in Tables~\ref{tab:IRI KFold top7} and \ref{tab:IRI KFold top15} in Online Appendix \ref{sec:iri-additional}.
Random forests perform the best among the three models in 18 and 27 out of 31 categories, respectively. 
(Due to the computational issues, we only compare our model with MNL and the Markov chain model in the additional experiments.)
In general, random forests show stronger performance when applied to more products.}}

\begin{table}[t]
	\begin{center}
		\caption{The summary statistics (the data size, the number of unique assortments in the data, and the average number of products in an assortment) of the IRI dataset after preprocessing and the average and standard deviation of the out-of-sample RMSE \eqref{RMSE_realdata} for each category when considering the top 9 products.}
		\scalebox{0.9}{
			\begin{tabular}{lrrcccccc} % {p{4cm}>{\raggedleft}p{1.5cm}>{\raggedleft}p{1.5cm}>{\raggedleft}p{1.5cm}>{\centering}p{3cm}>{\centering}p{3cm}>{\centering}p{3cm}}
				\hline
				Product category & \#Data  &\#Unique &\#Avg & RF & MNL & MC & DF \\
				& & assort & prod \\
				\hline
				Beer & 2,614 & 87 & 5.05 & \textbf{0.075} (0.011) & 0.082 (0.008) & 0.083 (0.012) & 0.079 (0.017)\\
				Blades & 2,996 & 100 & 4.68 & 0.058 (0.014) & 0.058 (0.009) & \textbf{0.051} (0.012) & 0.079 (0.026)\\
				Carbonated Beverages & 869 & 29 & 5.14 & 0.089 (0.039) & 0.104 (0.019) & 0.089 (0.021) & \textbf{0.086} (0.034)\\
				Cigarettes & 3,546 & 119 & 4.88 & 0.056 (0.016) & 0.070 (0.019) & 0.059 (0.013) & \textbf{0.051} (0.017)\\
				Coffee & 2,539 & 85 & 5.71 & \textbf{0.094} (0.016) & 0.109 (0.018) & 0.104 (0.015) & 0.098 (0.013)\\
				Cold Cereal & 894 & 30 & 6.47 & \textbf{0.055} (0.025) & 0.078 (0.025) & 0.068 (0.030) & 0.068 (0.027)\\
				Deodorant & 3,432 & 114 & 5.45 & 0.051 (0.009) & 0.051 (0.013) & \textbf{0.048} (0.008) & 0.059 (0.013)\\
				Diapers & 989 & 33 & 3.79 & \textbf{0.074} (0.031) & 0.080 (0.029) & 0.077 (0.016) & 0.154 (0.101)\\
				Facial Tissue & 1,647 & 55 & 4.16 & \textbf{0.099} (0.015) & 0.112 (0.020) & \textbf{0.099} (0.018) & 0.121 (0.013)\\
				Frozen Dinners/Entrees & 1,765 & 59 & 5.90 & \textbf{0.084} (0.008) & 0.110 (0.008) & 0.104 (0.009) & 0.089 (0.008)\\
				Frozen Pizza & 3,175 & 106 & 4.92 & \textbf{0.109} (0.029) & 0.124 (0.022) & 0.117 (0.019) & 0.111 (0.024)\\
				Household Cleaners & 1,947 & 65 & 6.18 & 0.076 (0.015) & \textbf{0.069} (0.004) & 0.073 (0.004) & 0.078 (0.014)\\
				Hotdogs & 3,546 & 118 & 4.53 & \textbf{0.113} (0.019) & 0.130 (0.018) & 0.120 (0.010) & 0.121 (0.017)\\
				Laundry Detergent & 4,012 & 134 & 5.42 & \textbf{0.092} (0.008) & 0.123 (0.015) & 0.114 (0.011) & 0.102 (0.008)\\
				Margarine/Butter & 1,019 & 34 & 5.76 & \textbf{0.097} (0.020) & 0.119 (0.015) & 0.107 (0.017) & 0.101 (0.020)\\
				Mayonnaise & 1,665 & 56 & 5.05 & \textbf{0.077} (0.027) & 0.109 (0.024) & 0.101 (0.016) & 0.118 (0.028)\\
				Milk & 1,763 & 59 & 4.36 & \textbf{0.116} (0.027) & 0.125 (0.010) & 0.120 (0.016) & 0.130 (0.022)\\
				Mustard & 2,395 & 80 & 5.70 & \textbf{0.063} (0.026) & 0.078 (0.020) & 0.073 (0.022) & 0.065 (0.018)\\
				Paper Towels & 2,092 & 70 & 4.97 & \textbf{0.094} (0.023) & 0.114 (0.013) & 0.100 (0.013) & 0.096 (0.015)\\
				Peanut Butter & 1,700 & 57 & 4.67 & \textbf{0.083} (0.026) & 0.095 (0.024) & 0.093 (0.021) & 0.091 (0.029)\\
				Photography supplies & 4,200 & 140 & 3.54 & \textbf{0.076} (0.009) & 0.091 (0.013) & \textbf{0.076} (0.018) & 0.152 (0.043)\\
				Razors & 2,482 & 83 & 3.27 & 0.088 (0.029) & 0.066 (0.014) & \textbf{0.055} (0.008) & 0.203 (0.062)\\
				Salt Snacks & 1,402 & 47 & 5.23 & \textbf{0.071} (0.019) & 0.084 (0.017) & 0.080 (0.019) & 0.074 (0.019)\\
				Shampoo & 4,547 & 151 & 5.25 & 0.082 (0.014) & \textbf{0.072} (0.011) & 0.073 (0.011) & 0.083 (0.016)\\
				Soup & 1,763 & 59 & 6.37 & \textbf{0.079} (0.017) & 0.118 (0.009) & 0.113 (0.007) & 0.089 (0.013)\\
				Spaghetti/Italian Sauce & 2,332 & 78 & 5.85 & 0.083 (0.013) & 0.096 (0.011) & 0.088 (0.008) & \textbf{0.079} (0.012)\\
				Sugar Substitutes & 2,880 & 96 & 4.92 & 0.050 (0.009) & 0.056 (0.011) & 0.051 (0.006) & \textbf{0.049} (0.004)\\
				Toilet Tissue & 1,620 & 54 & 5.24 & \textbf{0.116} (0.017) & 0.127 (0.024) & 0.121 (0.017) & 0.141 (0.024)\\
				Toothbrushes & 6,642 & 221 & 5.11 & 0.072 (0.004) & 0.072 (0.007) & \textbf{0.070} (0.004) & 0.074 (0.005)\\
				Toothpaste & 2,211 & 74 & 5.51 & 0.098 (0.028) & \textbf{0.078} (0.017) & 0.079 (0.021) & 0.085 (0.028)\\
				Yogurt & 2,856 & 95 & 4.74 & \textbf{0.113} (0.020) & 0.122 (0.014) & 0.117 (0.010) & 0.121 (0.013)\\
				\hline
		\end{tabular}}
		\label{tab:IRI KFold result}
	\end{center}
	\vspace{-4mm}
\end{table}

\subsection{Real Data: Hotel} \label{sec:Hotel_numerical}
In this section, we apply the random forest algorithm to a public dataset \citep{bodea2009data}.
%\new{The description of the data needs to be more accurate, so the next few sentences need to be reworked: the source of the data, the data format (each record is a transaction, a customer click? what information is available, such as price?), also cite the papers using bibtex. Also need to explain why we process the data this way (or why they have to use alternative methods), otherwise people may think that we are cherry picking.}
The dataset includes transient customers (mostly business travelers) who stayed in one of five continental U.S. hotels between March 12, 2007, and April 15, 2007.
The minimum booking horizon for each check-in date is four weeks.
Rate and room type availability and reservation information are collected via the hotel and/or customer relationship officers (CROs), the hotel's websites, and offline travel agencies.
Since there is no direct competition among these five hotels, we process the data separately.
A product is uniquely defined by the room type (e.g., suite).
For each transaction, the purchased room type and the assortment offered are recorded.

When processing the dataset, we remove products with less than 10 transactions.
We also remove the transactions whose offered assortments are unavailable for technical reasons.
For the transactions that none of the products in the available sets are purchased by the customer, we assume customers choose the no-purchase alternative.
%We do not add dummy transactions with no-purchases to uncensor the data like \cite{van2014market}, \cite{csimcsek2018expectation} and \cite{berbeglia2022comparative}.

We use five-fold cross-validation and RMSE defined in \eqref{RMSE_realdata_hotel} to examine the out-of-sample performance.
That is, for estimated choice probabilities $\hat\PR$ and testing set $\Tscr \triangleq \{(i_t, S_t)\}_{t=1}^T$, we define
\begin{equation}
	\label{RMSE_realdata_hotel}
	R M S E\left(\hat\PR, \Tscr\right)=\sqrt{\frac{\sum_{(i,S) \in \Tscr} \sum_{j \in S \cup\{0\}}\left(\I{j=i}-\hat\PR(j , S)\right)^{2}}{\sum_{(i,S) \in \Tscr}(|S|+1)}}.
\end{equation}
In Table \ref{tab:basic-information}, we show the summary statistics of the five datasets after preprocessing.
We also show the out-of-sample RMSE for each hotel (average and standard deviation).
In addition, we show the performance of the independent demand model (ID), which does not incorporate the substitution effect and is expected to perform poorly, in order to provide a lower bound for the performance.

The random forest algorithm outperforms the parametric methods for large datasets (Hotel 1, 2 and 3).
For smaller data sizes (Hotel 4 and 5), the random forest is on par with the best parametric estimation procedure (Markov chain) according to \citet{berbeglia2022comparative}.
\begin{table}[t]
	\begin{center}
		\caption{The summary statistics (the total number of products, the data size for in-sample and out-of-sample tests, the number of unique assortments in the data, and the average number of products in an assortment) of the five hotel datasets and the average and standard deviation of the out-of-sample RMSE across five folds.}
		\def\arraystretch{0.9}\begin{tabular}{crrrrr}
			\hline
			& \#Prod & \#In-sample & \#Out-sample & \#Unique assort & \#Avg prod\\
			\hline
			Hotel 1 & 10 & 1271 & 318 & 50 & 5.94 \\
			Hotel 2 & 6 & 347 & 87 & 26 & 3.27 \\
			Hotel 3 & 7 & 1073 & 268 & 25 & 4.32 \\
			Hotel 4 & 4 & 240 & 60 & 12 & 2.33 \\
			Hotel 5 & 6 & 215 & 54 & 21 & 3.52 \\
			\hline
		\end{tabular}
		
		\def\arraystretch{0.9}\begin{tabular}{rcccc}
			\hline
			& RF & MNL & MC & ID \\
			\hline
			Hotel 1 & \textbf{0.3040} (0.0046) & 0.3098 (0.0031) & 0.3047 (0.0039) & 0.3224 (0.0043)\\
			Hotel 2 & \textbf{0.3034} (0.0120) & 0.3120 (0.0148) & 0.3101 (0.0124) & 0.3135 (0.0178)\\
			Hotel 3 & \textbf{0.2842} (0.0051) & 0.2854 (0.0065) & \textbf{0.2842} (0.0064) & 0.2971 (0.0035)\\
			Hotel 4 & 0.3484 (0.0129) & \textbf{0.3458} (0.0134) & 0.3471 (0.0125) & 0.3584 (0.0047)\\
			Hotel 5 & 0.3219 (0.0041) & 0.3222 (0.0069) & \textbf{0.3203} (0.0046) & 0.3259 (0.0058)\\
			\hline
		\end{tabular}
		\label{tab:basic-information}
	\end{center}
	\vspace{-4mm}
\end{table}

{\revise
\subsection{Generalizability to Unseen Assortments} \label{sec:unseen_numerical}
\Copy{rev:variation1}{
One of the major challenges in the estimation of the DCM, compared to other statistical estimation problems, is the limited coverage of the training data.
In particular, the seller tends to offer a few assortments that they believe are profitable.
As a result, in the training data $\{\bm x_t\}_{t=1}^T$ only makes up a small fraction of the total $2^N$ available assortments.
Any estimation procedure needs to address the following issue: can the DCM estimated from a few assortments generalize to the assortments that have never been offered in the training data?

While the theoretical foundation has been studied in Section~\ref{sec:nearest-neighbor},
we show the numerical performance in this section.
%: theoretically, random forests adaptively choose nearest neighbors, and the choice probability of an assortment can be generalized to ``neighboring'' assortments (those with one more or one less product), as long as the underlying DCM possesses a certain degree of continuity in terms of the offered set $\bm x$.
Consider $N=10$ products.
We randomly choose $T_1$ assortments to offer in the training set and thus there are $T/T_1$ transactions for each assortment on average.
%``Large'' assortments refer to those with many products ($7 \leq |S| \leq 10$).
We use the rank-based DCM to generate the data with $k=4$ and 10 customer types.
The rank-based DCM is shown to be equivalent to RUM \citep{block1959random}.
Consumers are divided into $4$ or $10$ different types, each with a random preference permutation of all the products and the no-purchase alternative (see, e.g., \citealt{farias2013nonparametric}).
%For a given assortment of products, each type of consumer will purchase the product ranked the highest in her preference rank.
%If the no-purchase option is ranked higher than all the products in the assortment, then the customer does not purchase anything.
We randomly generate the fractions of customer types as follows:
draw uniform random variables $u_i$ between zero and one for $i=1,...,k$, and then set ${u_i}/{\sum_{j=1}^k u_j}$ to be the proportion of type $i$, $i=1,...,k$.

The performance is evaluated by the root mean squared error (RMSE), which is also used in \citet{berbeglia2022comparative}:
\begin{equation}
	\label{RMSE_soft}
	R M S E\left(\PR, \hat\PR\right)=\sqrt{\frac{\sum_{S \subseteq [N]} \sum_{j \in S \cup\{0\}}\left(\PR(j , S)-\hat\PR(j , S)\right)^{2}}{\sum_{S \subseteq [N]}(|S|+1)}},
\end{equation}
where $\PR$ denotes the actual choice probability and $\hat\PR$ denotes the estimated choice probability.
Because the actual choice probability is known, we can compute the RMSE theoretically by enumerating all the assortments.
For each setting, we generate 100 independent training datasets and compute the average and standard deviation of the RMSE.

The results are shown in Tables~\ref{tab:rank_list_4} and \ref{tab:rank_list_10}.
%\begin{table}[]
%	\centering
%	\begin{tabular}{rccc}
%		\hline
%		$\tilde{T}$ & Rank-based $k=4$ & Rank-based $k=10$ & MNL \\
%		\hline
%		$5$ & 0.193 (0.064) & 0.156 (0.034) & 0.133 (0.041) \\
%		$10$ & 0.158 (0.034) & 0.128 (0.026) & 0.111 (0.035)\\
%		$5$ (large) & 0.181 (0.056) & 0.124 (0.028) & 0.038 (0.017)\\
%		$10$ (large) & 0.150 (0.047) & 0.109 (0.027) & 0.034 (0.014)\\
%		$50$ & 0.087 (0.025) & 0.073 (0.014) & 0.054 (0.008)\\
%		$100$ & 0.068 (0.014) & 0.060 (0.007) & 0.042 (0.004)\\
%		$600$ & 0.045 (0.006) & 0.046 (0.004) & 0.037 (0.002)\\
%		\hline
%	\end{tabular}
%	\caption{The average and standard deviation of RMSE using random forests when there are a few assortments in the training data. The column represents different ground-truth models.}
%	\label{tab:sparse assortment}
%\end{table}
Notice that there are $2^N-1=1023$ possible available assortments.
Therefore, for example, $T_1=30$ implies that less than $1/30$ of the total assortments have been offered in the training data.
In general, the random forest outperforms the MNL model and the decision forest, and is on par with the Markov chain DCM when $T$ is large.
Except for the MNL, which does not improve significantly with higher assortment variation (larger $T_1$), all other models benefit from it.
For small sample sizes, the Markov chain model performs better.
It is likely due to the similarity between the rank-based model and the Markov chain model, e.g., both are regular choice models (definition in Section~\ref{sec:literature}).
As we shall see, when the underlying model is irregular, the random forest tends to have the best performance (see Table~\ref{tab:comparison result}).
%and the RMSE still converges as $T$ increases, while the MNL and Markov chain both have misspecification errors .
Moreover, in the real datasets, when the underlying model is unknown and likely to be irregular, the random forest performs better than the Markov chain model (see Table~\ref{tab:IRI KFold result}).}

%When the actual DCM is the MNL model, training random forests with 10 large assortments performs better than training with $600$ randomly chosen assortments.

%We also remark that the generalizability of random forests does not only depend on the estimator, but also the actual DCM.
%Some DCMs are more accessible to generalization to unseen assortments.
%It remains an exciting future research to formalize the statement and theoretically quantify the generalizability of a DCM to unseen data in the framework of random forests.

\begin{table}[t]
	\begin{center}
		\caption{\revise The average RMSE and the standard deviation for random forests, MNL, the Markov chain model, and the decision forest when the training data is generated by the rank-based model with 4 rankings.}
		\scalebox{0.85}{ \revise
			\begin{tabular}{rlcccclcccclcccc}
				\hline
				$T$     &  & \multicolumn{4}{c}{$T_1 = 30$} &  & \multicolumn{4}{c}{$T_1 = 100$} &  & \multicolumn{4}{c}{$T_1 = 300$} \\ \cline{3-6} \cline{8-11} \cline{13-16}
				&  & RF      & MNL     & MC   & DF   &  & RF       & MNL     & MC  & DF    &  & RF       & MNL     & MC  & DF    \\ \hline
				$ 300$   &  & 0.103   & 0.117   & 0.077 & 0.127  &  & 0.094    & 0.113   & 0.065  & 0.108 &  & 0.092    & 0.111   & 0.062 & 0.104  \\
				&  & (0.014) & (0.016) & (0.021) & (0.028) &  & (0.011)  & (0.016) & (0.017) & (0.021) &  & (0.011)  & (0.014) & (0.015) & (0.021) \\ \hline
				%            $ 600$   &  & 0.094   & 0.113   & 0.071   &  & 0.080    & 0.109   & 0.054   &  & 0.076    & 0.108   & 0.052   \\
				%                          &  & (0.014) & (0.018) & (0.020) &  & (0.010)  & (0.017) & (0.016) &  & (0.008)  & (0.017) & (0.015) \\ \hline
				%            $ 1500$  &  & 0.091   & 0.116   & 0.068   &  & 0.067    & 0.111   & 0.052   &  & 0.060    & 0.110   & 0.049   \\
				%                          &  & (0.017) & (0.015) & (0.022) &  & (0.010)  & (0.014) & (0.016) &  & (0.007)  & (0.014) & (0.015) \\ \hline
				$ 3000$  &  & 0.090   & 0.114   & 0.063 & 0.111  &  & 0.060    & 0.109   & 0.048 & 0.092  &  & 0.050    & 0.108   & 0.044 & 0.083  \\
				&  & (0.016) & (0.017) & (0.024) & (0.029) & & (0.009)  & (0.015) & (0.017) & (0.026) &  & (0.007)  & (0.015) & (0.017) & (0.030)\\ \hline
				%            $ 6000$  &  & 0.083   & 0.108   & 0.058   &  & 0.056    & 0.105   & 0.045   &  & 0.043    & 0.105   & 0.041   \\
				%                          &  & (0.018) & (0.021) & (0.024) &  & (0.012)  & (0.020) & (0.018) &  & (0.007)  & (0.020) & (0.017) \\ \hline
				$ 20000$ &  & 0.084   & 0.110   & 0.063 & 0.109  &  & 0.053    & 0.107   & 0.048 & 0.081  &  & 0.038    & 0.108   & 0.043 & 0.075   \\
				&  & (0.019) & (0.022) & (0.025) & (0.034) & & (0.011)  & (0.019) & (0.017) & (0.031) &  & (0.005)  & (0.016) & (0.017) & (0.030) \\ \hline
		\end{tabular}}
		\label{tab:rank_list_4}
	\end{center}
	\vspace{-4mm}
\end{table}

\begin{table}[t]
	\begin{center}
		\caption{\revise The average RMSE and the standard deviation for random forests, MNL, the Markov chain model and the decision forest when the training data is generated by the rank-based model with 10 rankings.}
		\scalebox{0.85}{\revise
			\begin{tabular}{rlcccclcccclcccc}
				\hline
				$T$     &  & \multicolumn{4}{c}{$T_1 = 30$} &  & \multicolumn{4}{c}{$T_1 = 100$} &  & \multicolumn{4}{c}{$T_1 = 300$} \\ \cline{3-6} \cline{8-11} \cline{13-16}
				&  & RF      & MNL     & MC   & DF   &  & RF       & MNL     & MC  & DF    &  & RF       & MNL     & MC  & DF    \\ \hline
				$ 300$   &  & 0.084   & 0.080   & 0.073  & 0.111  &  & 0.079    & 0.079   & 0.067 & 0.102  &  & 0.077    & 0.079   & 0.064 & 0.101  \\
				&  & (0.009) & (0.009) & (0.012) & (0.013) & & (0.009)  & (0.008) & (0.010) & (0.010) & & (0.009)  & (0.008) & (0.009) & (0.010)\\ \hline
				$ 3000$  &  & 0.071   & 0.077   & 0.050 & 0.086  &  & 0.054    & 0.074   & 0.043 & 0.068  &  & 0.048    & 0.074   & 0.041 & 0.064  \\
				&  & (0.009) & (0.009) & (0.009) & (0.018) & & (0.006)  & (0.008) & (0.007) & (0.017) & & (0.004)  & (0.008) & (0.006) & (0.013)\\ \hline
				$ 20000$ &  & 0.067   & 0.075   & 0.047 & 0.081  &  & 0.046    & 0.074   & 0.039 & 0.057  &  & 0.039    & 0.072   & 0.038 & 0.056  \\
				&  & (0.009) & (0.010) & (0.009) & (0.019) & & (0.005)  & (0.007) & (0.005) & (0.018) & & (0.003)  & (0.008) & (0.005) & (0.017)\\ \hline
		\end{tabular}}
		\label{tab:rank_list_10}
	\end{center}
\vspace{-4mm}
\end{table}

\Copy{rev:variation2}{
When the DCM is outside the scope of RUM and the regularity is violated, the Markov chain and MNL model may fail to specify the choice behavior correctly.
Next, we generate choice data using the comparison-based DCM \citep{huber1982adding}, described below.
Consumers implicitly score various attributes of the products in the assortment.
Then they undergo an internal round-robin tournament of all the products.
When comparing two products from the assortment, the customer checks their attributes and counts the number of preferable attributes of both products.
Eventually, the customer counts the total number of ``wins'' in the pairwise comparisons.
Here we assume that customers choose with equal probability if there is a tie.

In the experiment, we consider $N = 10$ products.
Consumers are divided into $k=2$ different types, whose proportions are randomly generated between 0 and 1.
Each type assigns uniform random variables between 0 and 1 to the five attributes of all the products (including the no-purchase option).
Again we use the RMSE in \eqref{RMSE_soft} to compare the predictive accuracy.
Like in the previous experiment, each setting is simulated 100 times. The result is shown in Table~\ref{tab:comparison result}.}

\begin{table}[t]
	\begin{center}
		\caption{\revise The average RMSE and the standard deviation using random forests, MNL and Markov chain Model under comparison-based DCM with different numbers of observed assortments in the training data ($T_1$).}
		\revise
		\def\arraystretch{0.9}\begin{tabular}{rlccclccclccc}
			\hline
			$T$     &  & \multicolumn{3}{c}{$T_1 = 30$} &  & \multicolumn{3}{c}{$T_1 = 100$} &  & \multicolumn{3}{c}{$T_1 = 300$} \\ \cline{3-5} \cline{7-9} \cline{11-13}
			&  & RF      & MNL     & MC      &  & RF       & MNL     & MC      &  & RF       & MNL     & MC      \\ \hline
			$ 300$   &  & 0.153   & 0.164   & 0.149   &  & 0.140    & 0.155   & 0.128   &  & 0.139    & 0.153   & 0.125   \\
			&  & (0.029) & (0.030) & (0.035) &  & (0.024)  & (0.029) & (0.033) &  & (0.023)  & (0.028) & (0.032) \\ \hline
			$ 3000$  &  & 0.141   & 0.158   & 0.142   &  & 0.099    & 0.147   & 0.121   &  & 0.084    & 0.145   & 0.116   \\
			&  & (0.034) & (0.041) & (0.045) &  & (0.023)  & (0.037) & (0.038) &  & (0.020)  & (0.037) & (0.038) \\ \hline
			$ 20000$ &  & 0.135   & 0.154   & 0.135   &  & 0.098    & 0.145   & 0.124   &  & 0.063    & 0.138   & 0.109   \\
			&  & (0.032) & (0.032) & (0.034) &  & (0.023)  & (0.030) & (0.033) &  & (0.017)  & (0.037) & (0.037) \\ \hline
		\end{tabular}
		\label{tab:comparison result}
	\end{center}
	\vspace{-4mm}
\end{table}

\Copy{rev:variation3}{
Because of the irregularity, both the MNL and the Markov chain DCM are outperformed by the random forest, especially when the data size increases.
Notice that as $T\to\infty$, the random forest is able to achieve diminishing RMSE, while the other two models do not improve because of the misspecification error.
Like the previous experiment, the random forest achieves stable performances with small standard deviations.

We run our algorithm on a server with 2.50GHz dual-core Inter Xeon CPU E5-2680 and 256GB memory.
The running time is shown in Table~\ref{tab:Running time}.
In terms of computation time, the random forest is the most efficient, while the EM algorithm used to estimate the Markov chain model takes much longer.
When $T = 20000$, the random forest spends 1/160 of the computation time of the Markov chain model.
Notice that the running time of random forests only increases slightly for large training sets.}}
\begin{table}[t]
	\begin{center}
		\caption{\revise The average running time of random forests, MNL, the Markov chain Model and the decision forest.}
		\revise
		\def\arraystretch{1}\begin{tabular}{rcccc}
			\hline
			$T$ & RF & MNL & MC & DF\\
			\hline
			$300$ & 1.4s & 0.3s & 14.5s & 19.7s\\
			%		$750$ & 1.5s & 0.6s & 24.1s \\
			%		$1500$ & 1.7s & 1.6s & 60.2s \\
			$3000$ & 1.9s & 3.2s & 120.9s & 58.1s\\
			%		$6000$ & 2.4s & 5.9s & 226.0s \\
			$20000$ & 5.1s & 22.6s & 819.9s & 67.3s \\
			\hline
		\end{tabular}
		\label{tab:Running time}
	\end{center}
\vspace{-4mm}
\end{table}

{\revise
\subsection{Scalability to a Large Number of Products} \label{sec:scalability}
\Copy{rev:scalability1}{
A major challenge in discrete choice models (DCMs) is scalability, as many models become increasingly complex and difficult to estimate when the number of products $N$ is large. In this section, we demonstrate that random forests remain practical and effective even when applied to settings with a large number of products.

We test the random forest when $N = 30$ and $50$. Here we assume the number of observed assortments $T_1 = 300$. All assortments cannot be enumerated to compute the RMSE. We thereby randomly sample 10,000 assortments as a testing set $\mathcal{T}_{test}$ and approximate the RMSE as follows.
\begin{equation}
	\label{eq:RMSE_MC}
	R M S E\left(\PR, \hat\PR, \mathcal{T}_{test}\right)=\sqrt{\frac{\sum_{S \subseteq \mathcal{T}_{test}} \sum_{j \in S \cup\{0\}}\left(\PR(j , S)-\hat\PR(j , S)\right)^{2}}{\sum_{S \subseteq \mathcal{T}_{test}}(|S|+1)}}.
\end{equation}

Tables~\ref{tab:rank_list_4_N} and \ref{tab:rank_list_10_N} show the result from 100 simulated datasets under the rank-based model with 4 and 10 rankings, respectively.
The random forest significantly outperforms MNL and shows an improvement in performance for larger $N$.
Notice that the RMSE is typically smaller when $N$ is large because the denominator of \eqref{eq:RMSE_MC} increases in $N$.
We do not test the Markov chain model and the decision forest due to %its
high computational cost for the problem scale. The EM algorithms converge slowly for more than 900 and 2,500 parameters under the Markov chain model.}
}

\begin{table}[t]
	\begin{center}
		\caption{\revise The average RMSE and the standard deviation for random forests and MNL for different $N$ when the training data is generated by the rank-based model with 4 rankings.}
		\revise
		\def\arraystretch{0.9}\begin{tabular}{rlcclcclcc}
			\hline
			$T$     &  & \multicolumn{2}{c}{$N= 10$} &  & \multicolumn{2}{c}{$N= 30$} &  & \multicolumn{2}{c}{$N = 50$} \\ \cline{3-4} \cline{6-7} \cline{9-10}
			&  & RF      & MNL         &  & RF       & MNL         &  & RF       & MNL         \\ \hline
			$ 300$    &  &  0.103    &  0.117   & &  0.069    &  0.078   & &  0.056   & 0.062   \\
			&  & (0.014)  & (0.016)  & & (0.004)  & (0.006)  & & (0.003) & (0.005) \\ \hline
			$ 3000$  &  & 0.090     & 0.114     & & 0.050     & 0.075    & &  0.043   & 0.060 \\
			&  & (0.016)  & (0.017)  & & (0.003)   & (0.006) & & (0.002)  & (0.004) \\ \hline
			$ 20000$ &  & 0.084    & 0.110     & & 0.045     & 0.075   & & 0.040     & 0.060   \\
			&  & (0.019)  & (0.022)  & & (0.004)  & (0.007) & & (0.003)  & (0.005) \\ \hline
		\end{tabular}
		\label{tab:rank_list_4_N}
	\end{center}
	\vspace{-4mm}
\end{table}

\begin{table}[t]
	\begin{center}
		\caption{\revise The average RMSE and the standard deviation for random forests and MNL for different $N$ when the training data is generated by the rank-based model with 10 rankings.}
		\revise
		\def\arraystretch{0.9}\begin{tabular}{rlcclcclcc}
			\hline
			$T$     &  & \multicolumn{2}{c}{$N= 10$} &  & \multicolumn{2}{c}{$N= 30$} &  & \multicolumn{2}{c}{$N = 50$} \\ \cline{3-4} \cline{6-7} \cline{9-10}
			&  & RF      & MNL         &  & RF       & MNL         &  & RF       & MNL         \\ \hline
			$ 300$    &  &  0.077    &  0.079   & &  0.052    &  0.054    & &  0.041   & 0.043   \\
			&  & (0.009)  & (0.008)  & & (0.003)  & (0.003)  & & (0.002) & (0.003) \\ \hline
			$ 3000$  &  & 0.048     & 0.074     & & 0.040     & 0.051    & &  0.033   & 0.040 \\
			&  & (0.004)  & (0.008)  & & (0.002)   & (0.004) & & (0.002)  & (0.003) \\ \hline
			$ 20000$ &  & 0.039    & 0.072     & & 0.036     & 0.051   & & 0.030     & 0.041   \\
			&  & (0.003)  & (0.008)  & & (0.002)  & (0.004) & & (0.001)  & (0.003) \\ \hline
		\end{tabular}
		\label{tab:rank_list_10_N}
	\end{center}
	\vspace{-4mm}
\end{table}

{\revise
\Copy{rev:scalability2}{
We also consider $N \in \{13, 25, 50, 100, 200, 400\}$ products and generate data using a rank-based DCM with $k = 10$ customer types. For the training set, we test $T_1 \in \{130, 250, 500, 1000, 2000, 4000\}$, generating 10 transactions per assortment, so that the total number of transactions is $T = 10T_1$. Performance is evaluated using RMSE in \eqref{RMSE_soft}, with a test set of 10,000 randomly generated assortments. For each configuration, we generate 100 independent training and test datasets and report the mean and standard deviation of the RMSE in Table~\ref{tab:RMSE_large_NT}. The results show that RMSE decreases with larger $T$ for all $N$, clearly illustrating convergence. We also find that random forests scale well to hundreds of products without requiring excessively large training samples.  
Table~\ref{tab:run_time} reports the running times. Random forests remain highly efficient even for large $N$ and $T$. In summary, for large-scale problems with large $N$ and $T$, random forests outperform MNL in predictive accuracy while maintaining comparable computational efficiency.
}}

\begin{table}[t]
	\begin{center}
		\caption{\revise The average RMSE and the standard deviation for random forests for large scale of $N$ when the training data is generated by the rank-based model with 10 rankings.} \label{tab:RMSE_large_NT}
		\revise
		\def\arraystretch{1}\begin{tabular}{rcccccc}
			\hline
			$T$ & $N = 13$ & $N = 25$ & $N = 50$ & $N = 100$ & $N = 200$ & $N = 400$\\
			\hline
			1300 & 0.059 (0.005) & 0.050 (0.004) & 0.039 (0.003) & 0.030 (0.002) & 0.022 (0.002) & 0.016 (0.001) \\
			2500 & 0.051 (0.004) & 0.044 (0.003) & 0.036 (0.002) & 0.028 (0.002) & 0.021 (0.002) & 0.016 (0.001) \\
			5000 & 0.044 (0.003) & 0.040 (0.002) & 0.033 (0.002) & 0.026 (0.002) & 0.020 (0.002) & 0.015 (0.001) \\
			10000 & 0.038 (0.003) & 0.035 (0.002) & 0.030 (0.001) & 0.024 (0.001) & 0.019 (0.001) & 0.014 (0.001) \\
			20000 & 0.033 (0.002) & 0.032 (0.002) & 0.028 (0.001) & 0.023 (0.001) & 0.018 (0.001) & 0.014 (0.001) \\
			40000 & 0.031 (0.002) & 0.029 (0.002) & 0.026 (0.001) & 0.022 (0.001) & 0.017 (0.001) & 0.013 (0.001) \\
			\hline
		\end{tabular}
	\end{center}
\end{table}

\begin{table}[t]
	\begin{center}
		\caption{\revise The average running time of random forests for large scale of $N$ when the training data is generated by the rank-based model with 10 rankings.} \label{tab:run_time}
		\revise
		\def\arraystretch{1}\begin{tabular}{rcccccc}
			\hline
			$T$ & $N = 13$ & $N = 25$ & $N = 50$ & $N = 100$ & $N = 200$ & $N = 400$\\
			\hline
			1300 & 1.6s & 1.9s & 2.1s & 2.2s & 3.3s & 5.4s \\
			2500 & 1.8s & 2.1s & 2.4s & 3.0s & 3.6s & 6.6s \\
			5000 & 2.5s & 3.4s & 3.6s & 4.3s & 5.4s & 9.6s \\
			10000 & 3.7s & 4.8s & 6.1s & 7.8s & 10.2s & 16.1s \\
			20000 & 6.3s & 8.6s & 11.5s & 14.7s & 20.8s & 30.8s \\
			40000 & 11.4s & 16.9s & 21.9s & 31.1s & 43.7s & 60.2s \\
			\hline
		\end{tabular}
	\end{center}
\end{table}

{\revise
\subsection{Numerical Experiments for Incorporating Pricing Information} \label{sec:pricing_numerical}
\Copy{rev:pricing_numerical1}{
	In this section, we test the performance of random forests when the price information is incorporated.
	We use the MNL model to generate the choice data. Let $\bm u$ denote the expected utility of the products and $\bm p$ their prices. For given assortment $S$, the choice probabilities of product $i \in S$ and the outside option are:
	\begin{equation}
		\label{eq:MNL choice probability}
		p(i, S) = \frac{\exp(u_i-p_i)}{1+\sum_{j \in S} \exp(u_j-p_j)},\quad p(0, S) = \frac{1}{1+\sum_{j \in S} \exp(u_j-p_j)}.
	\end{equation}
	
	Consider $N=10$ products.
	We generate $u_i$ as uniform random variables between 0 and 5 for each product.
	Then we generate a price uniformly on $[0,5]$ for each product.
	Note that in this experiment, all products are available in an assortment.
	As explained in Section~\ref{sec:price}, we use the link function $g(x) = \exp(-x)$.
	The customer's choice then follows the choice probability \eqref{eq:MNL choice probability}.
	The RMSE in \eqref{RMSE_soft} is no longer applicable because the assortments and prices cannot be exhausted.
	To evaluate the performance, we randomly generate $1000$ assortments and prices according to the same distribution as the training data.
	Then we evaluate the empirical RMSE in the test data.
	
	In addition to the MNL model, we also use linear demand as a benchmark.
	Under the linear demand, a customer purchases product $i$ with probability
	$p(i,\bm p)=(a_i + \sum_{j\in N} b_{ij} p_j)^+$ for some coefficients $a_i$ and $b_{ij}$.
	The no-purchase probability is thus $1-\sum_{i=1}^N p(i, \bm p)$.
	We estimate the coefficients using linear regression.
	Note that linear demand is very popular in modeling demand for multiple products with price information.
	We investigate the performance of random forests and the two benchmarks for different sizes of training data $T \in \{300,1500,5000\}$.}

\begin{table}[t]
	\begin{center}
		\caption{\revise The RMSE of random forests, MNL and linear demand with price information when the underlying model is MNL.}
		\revise
		\begin{tabular}{rccc}
			\hline
			$T$   &  RF  & MNL & Linear \\
			\hline
			300  &  0.075 (0.004)  &  0.020 (0.005)  &  0.064 (0.004) \\
			1500 & 0.050 (0.002)  &  0.009 (0.002)  &  0.053 (0.003) \\
			5000 & 0.039 (0.002)  &  0.005 (0.001)  &  0.050 (0.003) \\
			\hline
		\end{tabular}
		\label{tab:pricing result}
	\end{center}
\end{table}

\Copy{rev:pricing_numerical2}{
	From Table~\ref{tab:pricing result}, it is not surprising that MNL has the best performance, because the data is generated by the MNL model.
	However, when $T=1500$ and 5000, random forests are able to outperform the linear model, which is believed to be fairly robust and used widely.
	We believe it is due to the model misspecification of the linear model. It further demonstrates the benefit of random forests.

	We also test the performance when the ground truth model is linear demand, as shown in Table~\ref{tab:pricing result_linear_ground}.
	The coefficient $a_i$ is uniformly generated on $[0, 0.3]$ and price sensitivity $b_{ii}$ is uniform on $[-0.1, 0]$.
	The cross-sensitivity $b_{ij}, j \neq i$ is uniform on $[-0.03, 0.03]$ to capture the possibility of substitution or complementary.
	The linear model has the best performance, followed by random forests. Moreover, random forests significantly outperform the MNL as $T$ increases.
	From the cross-validation results of Tables~\ref{tab:pricing result} and \ref{tab:pricing result_linear_ground}, we can conclude that random forests have robust performance and benefit from large data sizes for different underlying models.} 
	Moreover, we also numerically show random forests can handle aggregated data in Online Appendix~\ref{sec:aggregated_numerical}.

\begin{table}[t]
	\begin{center}
		\caption{\revise The RMSE of random forests, MNL and linear demand with price information when the underlying model is linear demand.}
		\revise
		\begin{tabular}{rccc}
			\hline
			$T$   &  RF  & MNL & Linear \\
			\hline
			300  &  0.057 (0.008)  &  0.106 (0.014)  &  0.053 (0.005) \\
			1500 & 0.045 (0.005)  &  0.105 (0.017)  &  0.032 (0.004) \\
			5000 & 0.037 (0.005)  &  0.104 (0.015)  &  0.026 (0.005) \\
			\hline
		\end{tabular}
		\label{tab:pricing result_linear_ground}
	\end{center}
	\vspace{-4mm}
\end{table}}

\section{Concluding Remarks}\label{sec:conclusion}
This paper demonstrates some theoretical and practical benefits of using random forests to estimate discrete choice models, especially when the data is relatively abundant.
We also provide comprehensive numerical experiments in Section \ref{sec:numerical}.
It opens up a series of exciting new research questions:
\begin{itemize}
	\item What type of DCMs can be estimated well by random forests and have higher generalizability to unseen assortments?
	\item As we use the choice forest to approximate DCMs, how can we translate the properties of a DCM to the topological structure of decision trees?
	\item Can we provide finite-sample error bounds for the performance of random forests, with or without the price information?
	\item What properties does the product importance index MDI have?
	%\item Given a binary choice forest, possibly estimated by random forests, can we compute the optimal assortment and prices efficiently?
\end{itemize}
We hope to address some of these questions in future research.

\vspace{-2mm}
\theendnotes

\vspace{-2mm}
\bibliographystyle{informs2014} % outcomment this and next line in Case 1
\bibliography{ref.bib}

%\bibliographystyle{ormsv080} % outcomment this and next line in Case 1
%\bibliography{ref.bib}

\newpage

\begin{appendices}
{\SingleSpacedXI
\noindent {\large \textbf{Appendix}}
\section{Proofs} \label{sec:proof}

\proof{Proof of Theorem~\ref{thm:consistency}.}
We first prove that for a single decision tree, there is a high probability that the number of observations chosen in Step~\ref{step:select-subsample} in which $\bm x$ is offered is large.
More precisely, let $X_t = \I{\bm x_t= \bm x}$. It is easy to see that $\sum_{t=1}^T X_t = k_T$.
Step~\ref{step:select-subsample} randomly selects $z_T$ observations out of the $T$ with replacement.
Denote the bootstrap sample of $\left\{X_1,\dots,X_T\right\}$ by $\left\{Y_1,\dots,Y_{z_T}\right\}$.
By Hoeffding's inequality, we have the following concentration inequality
\begin{equation}\label{eq:sum-Yj-large}
	\PR\left( \left|\frac{\sum_{j=1}^{z_T}Y_j}{z_T} - \frac{k_T}{T}\right|> \epsilon\right)\le 2\exp\left( - 2z_T\epsilon^2\right)
\end{equation}
for any $\epsilon>0$.
In other words, the bootstrap sample in Step~\ref{step:select-subsample} does not deviate too far from the population as long as $z_T$ is large.
As we choose $\epsilon<\lim_{T\to\infty}k_T/T$, it implies that $\sum_{j=1}^{z_T}Y_j\to\infty$ and in particular
\begin{equation}\label{eq:greater-than-lT}
	\lim_{T\to\infty}\PR(\sum_{j=1}^{z_T}Y_j>l_T) = 1.
\end{equation}

Next we show that given $\sum_{j=1}^{z_T}Y_j>l_T$ for a decision tree, the leaf node that contains $\bm x$ only contains observations with $Y_j=1$.
That is, the terminal leaf containing $\bm x$ is a single corner of the unit hypercube.
If the terminal leaf node containing an observation with predictor $\bm x$, then it has no less than $\sum_{j=1}^{z_T}Y_j$ observations, because all the $\sum_{j=1}^{z_T}Y_j$ samples used to train the tree fall on the same corner in the predictor space.
If another observation with a different predictor is in the same leaf node, then it contradicts
Step~\ref{step:while-split} in the algorithm, because it would imply that another split could be performed.
Suppose $\left\{R_1,\dots,R_M\right\}$ is the final partition corresponding to the decision tree.
As a result, in the region $R_j$ such that $\bm x\in R_j$, we must have that $t_b(\bm x)$ is a random sample from the $\sum_{j=1}^{z_T}Y_j$ customer choices, according to Step~\ref{step:random-label}.

Now consider the estimated choice probability of a given assortment $\bm x$ from the random forest: $\sum_{b=1}^{B_T} \frac{1}{B_T}\I{t_b(\bm x)=i}$.
Notice that $t_b(\bm x)$, $b=1,\dots,B_T$, are IID given the training set.
By Hoeffding's inequality, conditional on $\left\{(i_t,\bm x_t)\right\}_{t=1}^T$,
\begin{equation}\label{eq:hoeffding1}
	\PR\left(\left|\sum_{b=1}^{B_T} \frac{1}{B_T}\I{t_b(\bm x)=i}-\PR(t_b(\bm x)=i|\left\{(i_t,\bm x_t)\right\}_{t=1}^T)\right|>\epsilon_1\bigg|\left\{(i_t,\bm x_t)\right\}_{t=1}^T\right)\le 2e^{-2B_T\epsilon_1^2},
\end{equation}
for all $\epsilon_1>0$.
Next we analyze the probability $\PR(t_b(\bm x)=i|\left\{(i_t,\bm x_t)\right\}_{t=1}^T)$ for a single decision tree.
By the previous paragraph, conditional on $\sum_{j=1}^{z_T}Y_j>l_T$, the output of a single tree $t_b(\bm x)$ is randomly chosen from the class labels of $\sum_{j=1}^{z_T}Y_j$ observations whose predictor is $\bm x$.
Let $Z_j$ be the class label of the $j$th bootstrap sample $Y_j$
in Step~\ref{step:select-subsample}.
Therefore, conditional on the event $\sum_{j=1}^{z_T}Y_j>l_T$ and the training data, we have
\begin{equation}\label{eq:iid-tree}
	\PR(t_b(\bm x)=i|\left\{(i_t,\bm x_t)\right\}_{t=1}^T,\sum_{j=1}^{z_T}Y_j>l_T) = \sum_{j=1}^{z_T} \frac{Y_j\I{Z_j=i}}{\sum_{j=1}^{z_T} Y_j}.
\end{equation}
Because $\{Y_j\I{Z_j=i}\}_{j=1}^{z_T}$ is a bootstrap sample, having IID distribution
\begin{equation*}
	\PR(Y_j\I{Z_j=i}=1) = \frac{\sum_{t=1}^T\I{i_t=i,\bm x_t=\bm x}}{T}
\end{equation*}
given the training data, we apply Hoeffding's inequality again
\begin{equation}\label{eq:hoeffding2}
	\PR\left( \left| \frac{\sum_{j=1}^{z_T}Y_j\I{Z_j=i}}{z_T} - \frac{\sum_{t=1}^T\I{i_t=i,\bm x_t=\bm x}}{T}\right|>\epsilon_2\bigg|\left\{(i_t,\bm x_t)\right\}_{t=1}^T \right)\le 2\exp(-2z_T \epsilon_2^2)
\end{equation}
for all $\epsilon_2>0$.
Now applying Hoeffding's inequality to $\sum_{t=1}^T\I{i_t=i,\bm x_t=\bm x}$ again, and because of Assumption~\ref{asp:independent-choice}, we have that
\begin{equation}\label{eq:hoeffding4}
	\PR\left(\left| \frac{\sum_{t=1}^T\I{i_t=i,\bm x_t=\bm x}}{k_T}-p(i,\bm x)\right|> \epsilon_3\right)\le 2\exp(-2k_T\epsilon_3^2)
\end{equation}
for all $\epsilon_3>0$.

With the above results, we can bound the target quantity
\begin{equation*} \begin{aligned}
	&\PR\left(\left|\sum_{b=1}^{B_T} \frac{1}{B_T}\I{t_b(\bm x)=i}-p(i,\bm x)\right|>\epsilon\right)\\
	&= \E\left[\PR\left(\left|\sum_{b=1}^{B_T} \frac{1}{B_T}\I{t_b(\bm x)=i}-p(i,\bm x)\right|>\epsilon\bigg| \left\{(i_t,\bm x_t)\right\}_{t=1}^T\right)\right]\\
	&\le \E\left[\PR\left(\left|\sum_{b=1}^{B_T} \frac{1}{B_T}\I{t_b(\bm x)=i}-\PR(t_b(\bm x)=i|\{(i_t,\bm x_t)\}_{t=1}^T)\right|>\epsilon/2\bigg| \{(i_t,\bm x_t)\}_{t=1}^T\right)\right]\\
	&\quad +\E\left[\PR\left(\left|p(i,\bm x)-\PR(t_b(\bm x)=i|\{(i_t,\bm x_t)\}_{t=1}^T)\right|>\epsilon/2\bigg| \{(i_t,\bm x_t)\}_{t=1}^T\right)\right]
\end{aligned} \end{equation*}
By \eqref{eq:hoeffding1}, the first term is bounded by $2\exp(-B_T\epsilon^2/2)$ which converges to zero as $B_T\to\infty$.
To bound the second term, note that
\begin{align}\label{eq:second-term-split}
	&\PR\left(\left|p(i,\bm x)-\PR(t_b(\bm x)=i|\{(i_t,\bm x_t)\}_{t=1}^T)\right|>\epsilon/2\bigg| \{(i_t,\bm x_t)\}_{t=1}^T\right)\notag\\
	&\le \PR\left(\left|p(i,\bm x)- \frac{\sum_{t=1}^T\I{i_t=i,\bm x_t=\bm x}}{k_T}\right|>\epsilon/6\bigg| \{(i_t,\bm x_t)\}_{t=1}^T\right)\notag\\
	&\quad+ \PR\left(\left|\frac{\sum_{t=1}^T\I{i_t=i,\bm x_t=\bm x}}{k_T}-\sum_{j=1}^{z_T} \frac{Y_j\I{Z_j=i}}{\sum_{j=1}^{z_T} Y_j}\right|>\epsilon/6\bigg| \{(i_t,\bm x_t)\}_{t=1}^T\right)\notag\\
	&\quad + \PR\left(\left|\sum_{j=1}^{z_T} \frac{Y_j\I{Z_j=i}}{\sum_{j=1}^{z_T} Y_j}-\PR(t_b(\bm x)=i|\{(i_t,\bm x_t)\}_{t=1}^T)\right|>\epsilon/6\bigg| \{(i_t,\bm x_t)\}_{t=1}^T\right)
\end{align}
The expected value of the first term in \eqref{eq:second-term-split} is bounded by $2\exp(-k_T\epsilon^2/18)$ by \eqref{eq:hoeffding4}, which converges to zero as $k_T\to\infty$.
For the second term of \eqref{eq:second-term-split}, we have that
\begin{align}
	&\PR\left(\left|\frac{\sum_{t=1}^T\I{i_t=i,\bm x_t=\bm x}}{k_T}-\sum_{j=1}^{z_T} \frac{Y_j\I{Z_j=i}}{\sum_{j=1}^{z_T} Y_j}\right|>\epsilon/6\bigg| \{(i_t,\bm x_t)\}_{t=1}^T\right)\notag\\
	%        &= \PR\left( \sum_{j=1}^{z_T}Y_j\le l_T\bigg|\{(i_t,\bm x_t)\}_{t=1}^T\right)\PR\left(\left|\frac{\sum_{t=1}^T\I{i_t=i,\bm x_t=\bm x}}{k_T}-\sum_{j=1}^{z_T} \frac{Y_j\I{Z_j=i}}{\sum_{j=1}^{z_T} Y_j}\right|>\epsilon/6\bigg| \{(i_t,\bm x_t)\}_{t=1}^T, \sum_{j=1}^{z_T}Y_j\le l_T\right) \notag\\
	%        &\quad + \PR\left( \sum_{j=1}^{z_T}Y_j> l_T\bigg|\{(i_t,\bm x_t)\}_{t=1}^T\right)\PR\left(\left|\frac{\sum_{t=1}^T\I{i_t=i,\bm x_t=\bm x}}{k_T}-\sum_{j=1}^{z_T} \frac{Y_j\I{Z_j=i}}{\sum_{j=1}^{z_T} Y_j}\right|>\epsilon/6\bigg| \{(i_t,\bm x_t)\}_{t=1}^T, \sum_{j=1}^{z_T}Y_j> l_T\right)\notag \\
	%        &\le \PR\left( \sum_{j=1}^{z_T}Y_j\le l_T\bigg|\{(i_t,\bm x_t)\}_{t=1}^T\right)\notag \\
	&\le \PR\left(\left|\frac{\sum_{t=1}^T\I{i_t=i,\bm x_t=\bm x}}{k_T}-\sum_{j=1}^{z_T} \frac{TY_j\I{Z_j=i}}{z_Tk_T}\right|>\epsilon/12\bigg| \{(i_t,\bm x_t)\}_{t=1}^T\right)\notag\\
	&\quad + \PR\left(\left|\sum_{j=1}^{z_T} \frac{TY_j\I{Z_j=i}}{z_Tk_T}-\sum_{j=1}^{z_T} \frac{Y_j\I{Z_j=i}}{\sum_{j=1}^{z_T} Y_j}\right|>\epsilon/12\bigg| \{(i_t,\bm x_t)\}_{t=1}^T\right)\label{eq:term2of8}
	%        &\quad + \PR\left(\left|\frac{\sum_{t=1}^T\I{i_t=i,\bm x_t=\bm x}}{k_T}-\sum_{j=1}^{z_T} \frac{Y_j\I{Z_j=i}}{\sum_{j=1}^{z_T} Y_j}\right|>\epsilon/6\bigg| \{(i_t,\bm x_t)\}_{t=1}^T, \sum_{j=1}^{z_T}Y_j> l_T\right)\\
	%        &\le \PR\left( \sum_{j=1}^{z_T}Y_j\le l_T\right) + \PR\left(\left|\frac{\sum_{t=1}^T\I{i_t=i,\bm x_t=\bm x}}{k_T}-\sum_{j=1}^{z_T} \frac{TY_j\I{Z_j=i}}{z_Tk_T}\right|>\epsilon/12\bigg| \{(i_t,\bm x_t)\}_{t=1}^T, \sum_{j=1}^{z_T}Y_j> l_T\right)\\
\end{align}
For the first term in \eqref{eq:term2of8}, note that by \eqref{eq:iid-tree}
\begin{equation*} \begin{aligned}
	&\PR\left(\left|\frac{\sum_{t=1}^T\I{i_t=i,\bm x_t=\bm x}}{k_T}-\sum_{j=1}^{z_T} \frac{TY_j\I{Z_j=i}}{z_Tk_T}\right|>\epsilon/12\bigg| \{(i_t,\bm x_t)\}_{t=1}^T\right)\\
	&=\PR\left(\left|\frac{\sum_{t=1}^T\I{i_t=i,\bm x_t=\bm x}}{T}-\sum_{j=1}^{z_T} \frac{Y_j\I{Z_j=i}}{z_T}\right|>k_T\epsilon/12\bigg| \{(i_t,\bm x_t)\}_{t=1}^T\right)\le 2\exp(- z_Tk_T^2\epsilon^2/72)\to 0
\end{aligned} \end{equation*}
as $T\to\infty$.
For the second term in \eqref{eq:term2of8}, we have
\begin{equation*} \begin{aligned}
	&\PR\left(\left|\sum_{j=1}^{z_T} \frac{TY_j\I{Z_j=i}}{z_Tk_T}-\sum_{j=1}^{z_T} \frac{Y_j\I{Z_j=i}}{\sum_{j=1}^{z_T} Y_j}\right|>\epsilon/12\bigg| \{(i_t,\bm x_t)\}_{t=1}^T\right)\\
	&\le\PR\left(\frac{\sum_{j=1}^{z_T}Y_j\I{Z_j=i}}{z_T}\left| \frac{T}{k_T}-\frac{z_T}{\sum_{j=1}^{z_T} Y_j}\right|>\epsilon/12\bigg| \{(i_t,\bm x_t)\}_{t=1}^T\right)\\
	&\le\PR\left(\left| \frac{T}{k_T}-\frac{z_T}{\sum_{j=1}^{z_T} Y_j}\right|>\epsilon/12\bigg| \{(i_t,\bm x_t)\}_{t=1}^T\right)\\
	&= \PR\left( \frac{Tz_T}{k_T\sum_{j=1}^{z_T} Y_j}\left| \frac{k_T}{T}-\frac{\sum_{j=1}^{z_T} Y_j}{z_T}\right|>\epsilon/12\bigg| \{(i_t,\bm x_t)\}_{t=1}^T\right)
\end{aligned} \end{equation*}
It is easy to see that $\frac{Tz_T}{k_T\sum_{j=1}^{z_T} Y_j}$ converges almost surely to a constant as $T\to\infty$.
Therefore, by \eqref{eq:sum-Yj-large} the last term converges to zero.
Finally, we move on to the third term of \eqref{eq:second-term-split}.
By \eqref{eq:iid-tree}, we have
\begin{equation*} \begin{aligned}
	&\PR\left(\left|\sum_{j=1}^{z_T} \frac{Y_j\I{Z_j=i}}{\sum_{j=1}^{z_T} Y_j}-\PR(t_b(\bm x)=i|\{(i_t,\bm x_t)\}_{t=1}^T)\right|>\epsilon/6\bigg| \{(i_t,\bm x_t)\}_{t=1}^T\right)\notag\\
	&=\PR\left(\left|\PR(t_b(\bm x)=i|\{(i_t,\bm x_t)\}_{t=1}^T, \sum_{j=1}^{z_T}Y_j>l_T)-\PR(t_b(\bm x)=i|\{(i_t,\bm x_t)\}_{t=1}^T)\right|>\epsilon/6\bigg| \{(i_t,\bm x_t)\}_{t=1}^T\right)\notag\\
	&\le \PR\left(2\PR\left( \sum_{j=1}^{z_T}Y_j\le l_T\bigg|\{(i_t,\bm x_t)\}_{t=1}^T\right)>\epsilon/6\bigg| \{(i_t,\bm x_t)\}_{t=1}^T\right).
\end{aligned} \end{equation*}
Notice that we are focusing on a fixed-design case, and $\{Y_j\}$ and $\{i_t\}$ are independent given Assumption~\ref{asp:independent-choice}.
Therefore,
\begin{equation*} \begin{aligned}
	\PR\left( \sum_{j=1}^{z_T}Y_j\le l_T\bigg|\{(i_t,\bm x_t)\}_{t=1}^T\right) = \PR\left( \sum_{j=1}^{z_T}Y_j\le l_T\right)\to 0
\end{aligned} \end{equation*}
by \eqref{eq:greater-than-lT}.
This completes the proof.
\endproof

{\revise 
\proof{Proof of Proposition~\ref{prn: RF_PNN}.}
To show the if part, we can construct a tree that splits at all products in $[N] \backslash (S \ominus S_i)$, then $S_i$ is one of at most $\ell$ assortments in the leaf node by the definition of $\ell$-PNN.

To show the only if part, we assume that there exists $\ell$ distinct assortments $S'_1, S'_2, \ldots, S'_{\ell} \in \mathcal{T}$ such that $S \ominus S'_j \subsetneq S \ominus S_i$, $j = 1, \ldots, \ell$. Then $\{S'_j\}_{j=1}^{\ell}$ are always in the same leaf node as $S_i$ when predicting $S$ because they are strictly more similar. Since a leaf node includes at most $\ell$ distinct assortments, $S_i$ will never be assigned.
\endproof

\proof{Proof of Proposition~\ref{prop:distance_worst_case}.}
Denote $M \triangleq \Big\lceil2^{N+2} \cdot c_0 \cdot \lceil\log_2 N\rceil \cdot \log N / (N-2)\Big\rceil$.
Let $\mathcal{S}(S,r)\triangleq \{S' \in 2^{[N]}: d(S,S') = r\}$ denote the set of assortments with distance equal to $r$ to assortment $S$.
Note that the cardinality of $\mathcal{S}(S,r)$ is $\binom Nr$.
Note that $\#\{S''\in 2^{[N]}: S''\ominus S\subsetneq S'\ominus S, S''\neq S\}=2^r-2$.
By Definition~\ref{def: PNN}, for an assortment $S' \in \mathcal{S}(S,r)$, $S'$ is the $\ell$-PNN of $S$ if at most $\ell-1$ out of $2^r-2$ assortments that \emph{dominate} $S'$
are included in $\mathcal T$.
Let $X$ be a binomial random variable representing the number of assortments in $\mathcal{T}$ that dominates $S'$. Then, $X \sim B(M, \frac{2^r-2}{2^N})$. Let $\mu = \E X = M \cdot \frac{2^r-2}{2^N}$. By Chernoff bound $\Pr(X < (1-\delta)\mu) < \exp(-\frac{\mu\delta^2}{2})$ for $0 < \delta < 1$, we have 
$\PR(X < \ell) < \exp(-(\mu-\ell)^2 / (2\mu)) < 1 / N^{\frac{2^r-2}{N-2} \cdot \lceil \log_2 N\rceil}$ for $r \ge \lceil \log_2 N \rceil$ $c_0 \ge 1$, and $\ell \le c_0 \log_2 N$, since %we can show $(\mu - \ell)^2 / (2\mu) > \lceil \log_2 N \rceil \cdot \log N \cdot \frac{2^r-2}{N-2}$.

$\frac{(\mu - \ell)^2}{2\mu} \ge \frac{(4c_0 \lceil \log_2 N\rceil \cdot \log N \cdot \frac{2^r-2}{N-2} - c_0 \log_2 N)^2}{8 c_0 \lceil \log_2 N\rceil \cdot \log N \cdot \frac{2^r-2}{N-2}} > \frac{9c_0}{8} \cdot \lceil \log_2 N\rceil \cdot \log N \cdot \frac{2^r-2}{N-2} > \lceil \log_2 N \rceil \cdot \log N \cdot \frac{2^r-2}{N-2}.$

%It is thus easy to see the probability that $S'$ is a PNN is $(1-\frac{2^r-2}{2^N})^M$ by Proposition \ref{prn: RF_PNN}.
%Thus, the probability that $S'$ is an $\ell$-PNN is 
%\begin{equation*}
%	\PR(S' \in \mathcal{P} | d(S, S') = r, S' \in \mathcal{T}) = \PR(X < \ell) \le  \exp\bigg(-\frac{(\mu-\ell)^2}{2\mu}\bigg) < \frac{1}{N^{\frac{2^r-2}{N-2} \cdot \lceil \log_2 N\rceil}}
%\end{equation*}

Let $\mathcal{P}_{\ell}$ denote the set of all $\ell$-PNNs of $S$. Then we have the probability that $S$ has PNNs equals distance $r$ is
\begin{equation*} \begin{aligned}
	&\PR (\mathcal{S}(S,r) \cap \mathcal{P}_{\ell} \neq \emptyset)
	< \binom{N}{r} \cdot \PR(S' \in \mathcal{P}_{\ell} | d(S, S') = r, S' \in \mathcal{T}) =\binom{N}{r} \cdot \PR(X < \ell) < \binom Nr \cdot \frac{1}{N^{\frac{2^r-2}{N-2} \cdot \lceil \log_2 N\rceil}},
\end{aligned} \end{equation*}
where the first inequality follows from the union bound, and the equality follows from Definition \ref{def: PNN}.
We write the upper bound of the above probability as $\PR_r\triangleq \binom Nr \big/  N^{\frac{2^r-2}{N-2} \cdot \lceil \log_2 N\rceil}$ for simplicity.

We first consider for $r = \lceil\log_2 N\rceil$:
\begin{equation*} \begin{aligned}
	&\PR_r < \binom {N}{\lceil\log_2 N\rceil} \cdot \frac{1}{N^{\lceil \log_2 N\rceil} }
	= \frac{N \cdot (N-1) \cdot \ldots \cdot (N-\lceil \log_2 N\rceil +1)}{(\lceil \log_2 N\rceil)! \cdot N^{\lceil \log_2 N\rceil}}\\
	&< \frac{1}{(\lceil \log_2 N \rceil)!} \cdot \frac{N-\lceil \log_2 N\rceil +1}{N}.
\end{aligned} \end{equation*}

Next, we consider the case $r > \lceil \log_2 N \rceil$.
We bound $\PR_r$ in this case by a geometric sequence, because
\begin{equation}
	\label{eq:P_r+1/P_r}
	\frac{\PR_{r+1}}{\PR_r}
	=\frac{N-r}{r+1} \cdot \frac{1}{N^{\lceil\log_2 N\rceil \cdot \frac{2^r}{N-2}}}
	< \frac{N^{1-\lceil \log_2 N \rceil}}{\lceil \log_2 N\rceil+1}.
\end{equation}
From the above two inequalities we can conclude:
\begin{equation*} \begin{aligned}
	&\PR((\cup_{r = \lceil \log_2 N \rceil}^N \mathcal{S}_{S,r} ) \cap \mathcal{P}_{\ell} \neq \emptyset)
	< \sum_{r = \lceil \log_2 N \rceil}^N \PR_{r} \\
	& < \frac{1}{(\lceil \log_2 N \rceil)!} \cdot \frac{N-\lceil \log_2 N\rceil +1}{N} \cdot \frac{1}{1-N^{1-\lceil \log_2 N \rceil}/(\lceil \log_2 N\rceil+1)}\\
	&= \frac{1}{(\lceil \log_2 N \rceil)!} \cdot \frac{1-(\lceil \log_2 N\rceil+1) / N}{1-N^{1-\lceil \log_2 N \rceil} / (\lceil \log_2 N\rceil+1)}
	< \frac{1}{(\lceil \log_2 N \rceil)!},
\end{aligned} \end{equation*}
where the first inequality follows from union bound, the second inequality from \eqref{eq:P_r+1/P_r} and sum of an infinite geometric sequence, and the last inequality follows from $(\lceil \log_2 N\rceil+1)^2>N^{2-\lceil \log_2 N\rceil}$.
\endproof

\proof{Proof of Proposition~\ref{thm:dis_upper_bound}.}
Note that for an individual tree, it never splits at the same product more than once.
Therefore, under random split, one can treat the sequence of attempted splits of an individual tree as a permutation of $(1,\dots,N)$.
Note that not all products in the permutation show up in the tree due to two reasons: (1) the split corresponding to some product in the permutation may result in empty leaves containing no training data, in which case the split is redrawn (moving to the next one in the permutation) and (2) terminal leaves containing less than $l$ data points (or less than $\ell$ distinct assortments) cannot be further split before reaching the end of the permutation, in which case the training of this tree is completed.
Similar to the setup in Proposition~\ref{prop:distance_worst_case}, the $M$ assortments in the training data are randomly drawn from $2^{[N]}$ with replacement.
Therefore, by symmetry, we focus on the sequence $1,2,\dots$ from now on.

We encode the $M$ assortments in the training data as binary vectors $\bm x_i\in \{0,1\}^{N}$, $i=1,\dots,M$, as mentioned in Section~\ref{sec:data_estimation}.
If we fix the unseen assortment $S$, or equivalently, $\bm x$, to be $\bm x=(1,\dots,1)$ by symmetry, then
the distance $d(S,S_i)=N-\|\bm x_i\|_1$ is the number of $0$s in $\bm x_i$.
Whenever a split is performed on the $j$th product, only the assortments among $\{\bm x_i\}_{i=1}^M$ whose $j$th digit is 1 may still be in the same leaf node as $\bm x$.
Under the splitting order of $1,2,\dots$ mentioned above, for instance,
consider $N = 3$ and the three assortments in the training data $\bm x_1 = (1,0,1), \bm x_2 = (1,0,0), \bm x_3 = (0,1,1)$.
In the first split on product one, $\bm x_1$ and $\bm x_2$ are still in the same leaf as $\bm x=(1,1,1)$.
The second split on this leaf is discarded because both $\bm x_1$ and $\bm x_2$ do not include product two and the split creates empty leaves.
In the third split, $\bm x_1$ is the only product remaining in the terminal leaf node of $\bm x$.
From this example, it is easy to see that the assortment in the same terminal node of $\bm x$ must be the largest $\ell$ among $\{\bm x_i\}_{i=1}^M$, which are interpreted as binary numbers.

With this interpretation, the average distance from $S$ to a PNN in the training data is equivalent to the following problem.
Consider $M$ random binary numbers drawn with replacement from $\{0,1,\dots,2^N-1\}$.
What is the expected number of zeros in the $\ell$-th largest number?
Note that the number of zeros is precisely the distance from the assortment corresponding to this binary number and the unseen $S$, or $\bm x=(1,\dots,1)$, according to \eqref{eq:distance}. Next we derive the expected number of zeros in the $\ell$-th largest number by Lemma \ref{lm:expected_zeros}.

\begin{lemma} \label{lm:expected_zeros}
	Let $X_1, \dots, X_M$ be i.i.d.\ random variables uniformly distributed on $\{0,1,\dots,2^N-1\}$, and let $X_{(\ell)}$ denote the $\ell$-th largest sample. 
	Write $X_{(\ell)}$ in its $N$-bit binary expansion (including leading zeros), and let $Z$ denote the number of zeros among these $N$ bits. 
	Then for all $N, \ell$, 
	$
	\E[Z] \le \log_2 \ell + \log_2 N + 2.56.
	$
\end{lemma}
The key idea is that the $\ell$-th largest number among $M$ samples approximately corresponds to the $\ell/M$-quantile of the set $\{0,1,\dots,2^N-1\}$. 
Hence, the number of leading $1$s before the first $0$ in its binary representation is roughly 
$
\log_2 (M / \ell) \approx N - \log_2 \ell - \log_2 N.
$
Consequently, the expected number of zeros should be at most $\log_2 \ell + \log_2 N + C$ for some constant $C$. 
We now provide a rigorous proof.
\proof{Proof of Lemma \ref{lm:expected_zeros}.}
	If $Z \ge i$, i.e., there are at least $i$ zeros in $X_{(\ell)}$, then the first $(N-i)$ bits of $X_{(\ell)}$ cannot all be $1$s. 
	Let $Y_i$ denote the number of samples among $\{X_1,\dots,X_M\}$ whose first $(N-i)$ bits are all $1$s. 
	Then $Y_i \sim \text{Binomial}(M, 1/2^{N-i})$, and
	$
	\E[Y_i] = \mu_i = M / 2^{N-i} \ge 2^i / N.
	$
	Since $Z \ge i$ implies $Y_i < \ell$,
	\begin{equation} \label{eq:EZ}
	\E[Z] = \sum_{i=1}^N \Pr(Z \ge i) 
	\le \sum_{i=1}^N \Pr(Y_i < \ell)
	\le \lceil \log_2 (2\ell N)\rceil - 1 + \sum_{i=\lceil \log_2 (2\ell N)\rceil}^N \Pr(Y_i < \ell).
	\end{equation}
	
	By the Chernoff bound, for any $0 < \delta < 1$,
	$
	\Pr(Y_i < (1-\delta)\mu_i) \le \exp(-\mu_i \delta^2 / 2).
	$
%	Choosing $\delta$ such that $(1-\delta)\mu_i = \ell$, i.e., $\delta = 1 - \ell / \mu_i$, yields
%	\[
%	\Pr(Y_i < \ell) \le \exp(-\mu_i \delta^2 / 2).
%	\]
	For $i = \lceil \log_2 (2\ell N)\rceil$, we have $\mu_i \ge 2^i / N \ge 2\ell$, hence 
	$
	\Pr(Y_i < \ell) \le \exp(-\mu_i / 8) \le \exp(-\ell / 4).
	$
	As $i$ increases, $\mu_i$ doubles with each increment of $i$, so $\Pr(Y_i < \ell)$ decays exponentially. 
	Therefore, the second term in~\eqref{eq:EZ} is bounded by
	\[
	\sum_{i=\lceil \log_2 (2\ell N)\rceil}^N \Pr(Y_i < \ell)
	\le \sum_{i=\lceil \log_2 (2\ell N)\rceil}^N \exp(-\mu_i / 8)
	< 2 \exp(-\ell / 4) < 1.56.
	\]
	Combining these bounds, we obtain
	$
	\E[Z] < \lceil \log_2 (2\ell N)\rceil - 1 + 1.56 
	\le \log_2 \ell + \log_2 N + 2.56.
	$
	This completes the proof of Lemma \ref{lm:expected_zeros}.
\endproof

The same bound also holds for the first to $(\ell - 1)$-th largest numbers. Thus, the expected distance between an assortment and one of its $\ell$-PNNs is also upper bounded by $\log_2 \ell + \log_2 N + 2.56$.

\endproof

\proof{Proof of Theorem~\ref{thm:total_error}.}
Consider the assortment to predict $S$ and the assortment in the same leaf node $S^*$.
We first prove the first part of the theorem.
From Proposition~\ref{prop:distance_worst_case} we know that $\PR (d(S, S^*) > \lceil \log_2 N\rceil-1) < 1/(\lceil \log_2 N\rceil)!,$ and by $c$-continuity we have
\begin{equation}
	\label{eq:continue_error}
	\PR\Big(\sum_{i \in [N]_+} \big|p(i, S) - p(i, S^*)\big|>\frac{c\log_2 N}{N}\Big) < 1/(\lceil \log_2 N\rceil)!.
\end{equation}
Suppose $Q$ is the number of transactions for assortment $S$ in the training data.
Let $\hat{p}(i, S^*)$ be the empirical frequency of choosing product $i$ among $Q$ samples.
Conditional on $Q$, we know $\hat{p}(i, S^*)$ has mean $p(i, S^*)$ and variance $p(i,S^*) (1-p(i, S^*))/Q$. Then by Chebyshev's inequality and $Q\ge \lceil\frac{N^3 \cdot (|S'|+1)^2}{c_1^2(\log_2 N)^2}\rceil$, we have
$$\PR \Big(\big|\hat{p}(i, S^*)-p(i, S^*)\big|\geq \frac{c_1 \cdot \log_2 N}{N \cdot (|S^*|+1)}\Big)
\leq \frac{p(i,S^*) (1-p(i, S^*))}{Q} \cdot \frac{N^2 \cdot (|S^*|+1)^2}{c_1^2 \cdot (\log_2 N)^2} < \frac{p(i, S^*)}{N}.$$
The total sampling error can be bounded by
\begin{equation}
	\label{eq:sample_error}
	\PR \Big(\sum_{i \in S^*\cup \left\{0\right\}} \big|\hat{p}(i, S^*)-p(i, S^*)\big|\geq \frac{c_1 \cdot \log_2 N}{N}\Big) < \frac{\sum_{i \in S^*\cup \left\{0\right\}} p(i, S^*)}{N} = 1/N.
\end{equation}
Combining \eqref{eq:continue_error} and \eqref{eq:sample_error} we have that
$$\PR \Big(\sum_{i \in [N]_+} \big|p(i,S) - \hat{p}(i, S^*)\big|>\frac{(c+c_1)\cdot \log_2 N}{N}\Big) <  \frac{1}{(\lceil \log_2 N\rceil)!}+\frac{1}{N}.$$
Notice that random forests use the empirical frequency of $S^*$ to estimate unseen assortment $S$, i.e. $\hat{p}(i, S^*) = \hat{p}(i, S), i \in [N]_+$. This completes the proof for the first part.

For the second part, similarly, let $S^*$ denote the assortment in the same leaf node as $S$. Let $M$ denote the number of assortments drawn with replacement. Let $Q$ denote the number of transactions for each assortment.
From Proposition~\ref{thm:dis_upper_bound} we know $\E[d(S,S^*)] \leq \log_2 \ell +\log_2 N + 2.56$, and by $c$-continuity we have
\begin{equation}
	\label{eq:continue_error_random_split}
	\E\Big[\sum_{i \in [N]_+} \big|p(i,S)-p(i, S^*)\big|\Big] < \frac{c \cdot (\log_2 \ell +\log_2 N + 2.56)}{N}.
\end{equation}
Recall that the empirical distribution $\hat{p}(i, S^*)$ has mean $p(i, S^*)$ and variance $p(i,S^*) (1-p(i, S^*))/Q$, then we have total sampling error is bounded by
\begin{align}
	\label{eq:sample_error_randomsplit}
	&\E\Big[\sum_{i \in S^*\cup\{0\}} \big|\hat{p}(i, S^*)-p(i, S^*)\big|\Big]
	= \sum_{i \in S^*\cup\{0\}} \E\Big[\big|\hat{p}(i, S^*)-p(i, S^*)\big|\Big]
	\leq \sum_{i \in S^*\cup\{0\}} \sqrt{\E \Big[\big(\hat{p}(i, S^*)-p(i, S^*)\big)^2\Big]}\nonumber\\
	&=\sum_{i \in S^*\cup\{0\}} \sqrt{\VaR\big[\hat{p}(i, S^*)-p(i, S^*)\big]}
	< \sum_{i \in S^*\cup\{0\}} \sqrt{p(i, S^*)/ Q}
	\leq \sqrt{(|S^*|+1) / Q}
	\leq \frac{c_1 \log_2 N}{N}.
\end{align}
Combining (\ref{eq:continue_error_random_split}) and (\ref{eq:sample_error_randomsplit}) we have
$$\E\Big[\sum_{i \in [N]_+} \big|p(i,S) - \hat{p}(i, S^*)\big|\Big] < \frac{(c+c_1) \cdot \log_2 N + c \log_2 \ell + 2.56c}{N}.$$
This completes the second part.
\endproof
}

\proof{Proof of Theorem~\ref{thm:single tree}.}
We first provide a road map of the proof.
The analysis of the first split is the most important, as the subsequent splits can be analyzed similarly with similar error probabilities.
Therefore, we prove two lemmas to analyze the first split.
In the first lemma, we show that when $T\to +\infty$, the first split is on product one almost surely.
In the second lemma, we control the probability errors when $T$ is finite.
The remaining part of the proof applies the same technique to the remaining splits.

\begin{lemma}[Theoretical Gini index] \label{lm:theoretical_cut}
	When the data size is sufficiently large, $T \rightarrow +\infty$, then the first split converges to product 1 almost surely, i.e., the theoretical Gini index of splitting at product one is the smallest almost surely.
\end{lemma}

\proof{Proof of Lemma~\ref{lm:theoretical_cut}.}
Given the training data, we define the random variables $n_k^j$ ($n_k^{-j}$), which represent the number of assortments in the training data where $j$ is (not) in the assortment and product $k$ is chosen.
Also, let $n^j$ denote the number of assortments in the training data including product $j$.
To simplify the notation, we use product $N+1$ to denote the no-purchase option.
Our first step is to compute the Gini index $G_j$ if the first split is on product $j$.
Recall the definition of the Gini index: $\sum_{R_j} \frac{t_j}{T}\sum_{k=0}^N \hat p_{jk}(1-\hat p_{jk})$.
If the first split is on product $j$, then the left node (assortments without product $j$) has $n^j$ data points while the right node has $T-n^j$.
In the left node, %the frequency of label $k$, i.e.,
the fraction of assortments resulting in a purchase of product $k$, is $n_k^{-j}/(T-n^j)$.
Similarly, in the right node, the frequency is $n_k^j/n^j$ for product $k\le j$ and 0 for product $k>j$.
Therefore, we have that
\begin{align}
	\label{Gini_j}
	G_j &= \frac{n^j}{T}\sum_{k=1}^j \frac{n_k^j}{n^j}(1-\frac{n_k^j}{n^j})+\frac{T-n^j}{T}\sum_{k=1}^{j-1} \frac{n_k^{-j}}{T-n^j}(1-\frac{n_k^{-j}}{T-n^j})+\frac{T-n^j}{T}\sum_{k=j+1}^{N+1} \frac{n_k^{-j}}{T-n^j}(1-\frac{n_k^{-j}}{T-n^j})\nonumber\\
	&= \frac{1}{T} \left(n^j-\frac{\sum_{k=1}^j (n_k^j)^2}{n^j}+T-n^j-\frac{\sum_{k=1}^{j-1} (n_k^{-j})^2 +\sum_{k=j+1}^{N+1} (n_k^{-j})^2}{T-n^j}\right)\nonumber\\
	&=1-\frac{\sum_{k=1}^j (n_k^j)^2}{T n^j}-\frac{\sum_{k=1}^{j-1} (n_k^{-j})^2 +\sum_{k=j+1}^{N+1} (n_k^{-j})^2}{T(T-n^j)}.
\end{align}
Notice that the second term in \eqref{Gini_j} only depends on assortments including product $j$, and the third term only depends on assortments without product $j$.

%Next, we are going to show $G_1<G_j$ with high probability.
We define $H_j$ for the quantity in \eqref{Gini_j} for simplicity:
\begin{equation}\label{eq:Hj-formula}
	H_j \triangleq \frac{\sum_{k=1}^j (n_k^j)^2}{n^j}+\frac{\sum_{k=1}^{j-1} (n_k^{-j})^2 +\sum_{k=j+1}^{N+1} (n_k^{-j})^2}{T-n^j}.
\end{equation}
It's easy to see that $G_1<G_j$ if and only if $H_1>H_j$.

%To analyze $\PR(H_1>H_j)$, we want to show that $H_1$ and $H_j$ are concentrated around their expectations $\E[H_1]$ and $\E[H_j]$.
%For such concentration results
To compute and compare Gini indices $G_j$, we need to analyze the probability distributions of $n^j$, $n_k^j$, and $n_k^{-j}$,
Note that the randomness of $n^j$, $n_k^j$, and $n_k^{-j}$ are caused by the randomness of sampling assortments uniformly in the training data.
It is clear that $n^j$ has a binomial distribution $B(T,1/2)$, because each assortment includes product $j$ with probability $1/2$.
The random variables $n_k^j$ and $n_k^{-j}$ we defined above have the following binomial distribution for $1\leq j \leq N$:
$$n_k^j\sim\left\{\begin{aligned}
	& B(T,1/2^{k+1}) & & {1\leq k<j}\\
	& B(T,1/2^j) & & {k=j}\\
	&0 & & {j<k\leq N+1}
\end{aligned}\right. ,\quad
n_k^{-j}\sim\left\{\begin{aligned}
	& B(T,1/2^{k+1}) & & {1\leq k<j}\\
	&0 & & {k=j}\\
	& B(T,1/2^k) & & {j<k\leq N}\\
	& B(T,1/2^N) & & {k=N+1}
\end{aligned}\right. .$$
To understand it, note that an assortment is counted in $n_k^j$ when products $1$ to $k-1$ are not in the assortment while product $k$ and $j$ are in the assortment, according to the preference ranking.
So the probability is $1/2^{k+1}$.
The other probabilities follow a similar argument.
From the above distributions, we can show that when $T \rightarrow \infty$, $H_j$ converges to
$$H_j \rightarrow \frac{T}{3} \cdot \Big(1+\frac{1}{2^{2j-2}}+ \frac{1}{2^{2N}}\Big).$$
The theoretical Gini index for product $j$ is
\begin{equation*}
	G_j \rightarrow \frac{2}{3}-\frac{1}{3\cdot 2^{2j-2}}-\frac{1}{3\cdot 2^{2N-2}}.%+O(1/T).
\end{equation*}
Obviously, $G_1 < G_j, j \geq 2$, so the theoretical cut of the first split is on product one.
\endproof
Next, for finite $T$, we are going to bound the probability of incorrect splits.
We write $H_j$ as
\begin{equation*} \begin{aligned}
	H_j &= n^j \cdot \Bigg(\sum_{k = 1}^j \hat{p}_{k,j}^2\Bigg) + (T-n^j) \cdot \Bigg(\sum_{k = 1}^{N+1} \hat{p}_{k,-j}^2\Bigg) \\
	&=\frac{T}{2} \cdot \Bigg(\sum_{k = 1}^j \hat{p}_{k,j}^2\Bigg)+ \frac{T}{2} \cdot \Bigg(\sum_{k = 1}^{N+1} \hat{p}_{k,-j}^2\Bigg) + \Bigg(n^j-\frac{T}{2}\Bigg) \cdot \Bigg(\sum_{k = 1}^j \hat{p}_{k,j}^2 - \sum_{k = 1}^{N+1} \hat{p}_{k,-j}^2\Bigg)\\
	&=\frac{T}{2} \cdot \Bigg(\sum_{k = 1}^j \hat{p}_{k,j}^2\Bigg)+ \frac{T}{2} \cdot \Bigg(\sum_{k = 1}^{N+1} \hat{p}_{k,-j}^2\Bigg)+O_T(1) \cdot \Bigg|n^j-\frac{T}{2}\Bigg|,
\end{aligned} \end{equation*}
where $(\hat{p}_{k,j})_k$ and $(\hat{p}_{k,-j})_k$ are the empirical distributions of two multinomial random variables over $n^j$ and $T-n^j$  IID samples, respectively.
By comparing these multinomial distributions with concentration inequalities, we can get $H_1 > H_j$ if $j > 1$ with high probability.
Before formally proving it in Lemma \ref{lm:empirical_cut},
we introduce the Chernoff inequality to show the concentration of these random variables.

\textbf{The Chernoff inequality:} Let $X_1,X_2,...,X_n$ be independent Bernoulli random variables. Denote $X = \sum_{i=1}^n X_i$ with $\mu=E[X]$.
For all $0<\delta<1$, we have that $\PR \left(X<(1-\delta)\mu\right)<\exp(-\frac{\mu \delta^2}{2})$.

Applying the Chernoff inequality to the Binomial random variables
$n^j \sim B(T,1/2)$ for $1 \leq j \leq N$, we have
$\PR(n^j<(1-\delta)T/2)=\PR(n^j>(1+\delta)T/2)<\exp(-T\delta^2/4)$.
Therefore,
\begin{equation}
	\label{n^j bound}
	\PR\left(\left|n^j-\frac{1}{2}T\right|>\frac{\delta}{2}T\right) < 2\exp\bigg(-\frac{\delta^2 T}{4}\bigg)
\end{equation}

Now we are ready to bound the probabilities for the empirical Gini indices.
Since $\PR(G_j < G_1), j \geq 3$ are much smaller than $\PR(G_2 < G_1)$, we separately consider these two probabilities in the following lemma to get a better bound.
\begin{lemma}[Empirical Gini index] \label{lm:empirical_cut}
	The probability that first cut is not on product one is at most $$12\exp\left(-\frac{T}{145}\right)+10(N-2)\exp\left(-\frac{T}{100}\right).$$
\end{lemma}
\proof{Proof of Lemma~\ref{lm:empirical_cut}.}
We first bound the probability that $H_1$ is small compared to its limit when $T$ is sufficiently large. %mean $\E[H_1]$.
By~\eqref{eq:Hj-formula}, because $n^1_1=n^1$ and $n^1_k=0$ for $k\ge 2$, we can rewrite $H_1$ as
$H_1 = n^1+\frac{\sum_{k=2}^{N+1} (n_k^{-1})^2}{T-n^1}$.
The concentration of the first term $n^1$ follows from \eqref{n^j bound}.
Conditional on $n^1$, $n_2^{-1} \sim B(T-{n}^1,1/2)$ and $n_3^{-1} \sim B(T-{n}^1,1/4)$. Define $\delta_1\triangleq\delta/\sqrt{1-\delta}$. By the Chernoff inequality we have:
\begin{equation*} \begin{aligned}
	&\PR\left((n_2^{-1})^2<\frac{1}{4} (1-\sqrt{2} \delta_1)^2 (T-n^1)^2 \bigg|n^1\right)\\
	=&\PR\left(n_2^{-1}<\frac{1}{2} (1-\sqrt{2} \delta_1) (T-n^1)\bigg|n^1\right)
	<\exp\bigg(-\frac{\delta_1^2(T-{n}^1)}{2}\bigg),\\
	&\PR\left((n_3^{-1})^2<\frac{1}{16} (1-2 \delta_1)^2 (T-n^1)^2\bigg|n^1\right)\\
	=&\PR\left(n_3^{-1}<\frac{1}{4} (1-2 \delta_1) (T-n^1)\bigg|n^1\right)
	<\exp\bigg(-\frac{\delta_1^2(T-{n}^1)}{2}\bigg).
\end{aligned} \end{equation*}
Since $(n_k^{-1})^2\ge 0$, we can bound the probability of the sum $\sum_{k=2}^{N+1} (n_k^{-1})^2$ in the following way:
\begin{equation*} \begin{aligned}
	&\PR\left(\sum_{k=2}^{N+1} (n_k^{-1})^2<\left[\frac{1}{4} (1-\sqrt{2}\delta_1)^2+\frac{1}{16}(1-2\delta_1)^2\right](T-n^1)^2\bigg|n^1\right)\nonumber\\
	<\,&\PR\left((n_2^{-1})^2+(n_3^{-1})^2<\left[\frac{1}{4} (1-\sqrt{2}\delta_1)^2+\frac{1}{16}(1-2\delta_1)^2\right](T-n^1)^2\bigg|n^1\right)\nonumber\\
	<\,&\PR\left((n_2^{-1})^2<\frac{1}{4} (1-\sqrt{2} \delta_1)^2 (T-n^1)^2\bigg|n^1\right)
	+\PR\left((n_3^{-1})^2<\frac{1}{16} (1-2 \delta_1)^2 (T-n^1)^2\bigg|n^1\right)\nonumber\\
	<\,&2\exp\bigg(-\frac{\delta_1^2(T-{n}^1)}{2}\bigg)
\end{aligned} \end{equation*}
Therefore, conditional on $n^1$, we can bound $H_1$ as below
\begin{align}
	\label{H_1|n_1}
	&\PR\left(H_1<n^1+\left[\frac{1}{4} (1-\sqrt{2}\delta_1)^2+\frac{1}{16}(1-2\delta_1)^2\right](T-n^1)\bigg|n^1\right)\nonumber\\
	=\,&\PR\left(\sum_{k=2}^{N+1} (n_k^{-1})^2<\left[\frac{1}{4} (1-\sqrt{2}\delta_1)^2+\frac{1}{16}(1-2\delta_1)^2\right](T-n^1)^2\bigg|n^1\right)\nonumber\\
	<\,&2\exp\bigg(-\frac{\delta_1^2(T-{n}^1)}{2}\bigg)
\end{align}
Combined with \eqref{n^j bound} when $j=1$, we can bound the unconditional probability:
\begin{align}
	\label{H_1}
	&\PR\left(H_1<\frac{(1-\delta)T}{2}+\left[\frac{1}{4} (1-\sqrt{2}\delta_1)^2+\frac{1}{16}(1-2\delta_1)^2\right]\frac{(1+\delta)T}{2}\right)\nonumber\\
	<\,&\PR\left(\left|n^1-\frac{1}{2}T\right|>\frac{\delta}{2}T\right)\nonumber\\
	&+\sum_{k:|k-T/2|<\delta T/2}\PR(n^1=k)\PR\left(H_1<\frac{(1-\delta)T}{2}+\left[\frac{1}{4} (1-\sqrt{2}\delta_1)^2+\frac{1}{16}(1-2\delta_1)^2\right]\frac{(1+\delta)T}{2}\bigg|n^1=k\right)\nonumber\\
	<\,&2\exp\bigg(-\frac{\delta^2 T}{4}\bigg)+\sum_{k:|k-T/2|<\delta T/2}\PR(n^1=k)\times 2\exp\bigg(-\frac{\delta_1^2(T-k)}{2}\bigg)\nonumber\\
	<\,&2\exp\bigg(-\frac{\delta^2 T}{4}\bigg)+2\exp\bigg(-\frac{\delta_1^2(1-\delta)T}{4}\bigg)
	=4\exp\bigg(-\frac{\delta^2 T}{4}\bigg).
\end{align}
This establishes the bound for the probability that $H_1$ would be too large.

Next we bound the probability that $H_2$ is large compared to its mean.
Recall from \eqref{eq:Hj-formula} that
\begin{equation}
	\label{H_2 formulation}
	H_2 = \frac{(n_1^2)^2+(n_2^2)^2}{n^2}+\frac{(n_1^{-2})^2 +\sum_{k=3}^{N+1} (n_k^{-2})^2}{T-n^2}
\end{equation}
For the first term in \eqref{H_2 formulation}, conditional on $n^2$, both $n_1^2$ and $n_2^2={n}^2-n_1^2$ follows $B({n}^2,1/2)$.
Therefore, by the Chernoff bound, we have
\begin{equation*}
	\PR \left(\left|n_1^2-\frac{1}{2}n^2\right|>\frac{\sqrt{2}}{2}\delta_1 n^2\bigg|n^2\right)
	<2\exp\left(-\frac{\delta_1^2 {n}^2}{2}\right)
\end{equation*}
Moreover, we have
\begin{align}
	\label{H_2_1}
	&\PR \left((n_1^2)^2+(n^2-n_1^2)^2>\frac{1}{2}(1+2\delta_1^2)(n^2)^2\bigg|n^2\right)\nonumber\\
	=\,&\PR \left(\frac{1}{2} (n^2)^2 +2(n_1^2-\frac{1}{2} n^2)^2>\frac{1}{2}(1+2\delta_1^2)(n^2)^2\bigg|n^2\right)\nonumber\\
	=\,&\PR \left(\left|n_1^2-\frac{1}{2}n^2\right|>\frac{\sqrt{2}}{2}\delta_1 n^2\bigg|n^2\right)
	<\,2\exp\left(-\frac{\delta_1^2 {n}^2}{2}\right).
\end{align}
where the first equality above follows from $(n_1^2)^2+(n^2-n_1^2)^2=\frac{1}{2} (n^2)^2 +2(n_1^2-\frac{1}{2} n^2)^2$.

For the second term in \eqref{H_2 formulation}, conditional on $n^2$, we have $n_1^{-2} \sim B(T-{n}^2,1/2)$, $T-{n}^2-n_1^{-2} \sim B(T-{n}^2,1/2)$, $n_3^{-2} \sim B(T-{n}^2,1/4)$ and $T-{n}^2-n_1^{-2}-n_3^{-2} \sim B(T-{n}^2,1/4)$.
By the Chernoff bound we have:
\begin{equation*} \begin{aligned}
	\PR \left(n_3^{-2}<\frac{1}{4}(1-2\delta_1)(T-n^2)\bigg|n^2\right)&<\exp\left(-\frac{\delta_1^2(T-{n}^2)}{2}\right)\\
	\PR \left(T-{n}^2-n_1^{-2}-n_3^{-2}<\frac{1}{4}(1-2\delta_1)(T-n^2)\bigg|n^2\right)&<\exp\left(-\frac{\delta_1^2(T-{n}^2)}{2}\right)
\end{aligned} \end{equation*}
From the above two equations we have
\begin{equation*}
	\PR \left(n_3^{-2}(T-{n}^2-n_1^{-2}-n_3^{-2})<\frac{1}{16}(1-2\delta_1)^2(T-n^2)^2\bigg|n^2\right)<2\exp\left(-\frac{\delta_1^2(T-{n}^2)}{2}\right)
\end{equation*}
Similar to \eqref{H_2_1} we also have
\begin{equation*}
	\PR \left((n_1^{-2})^2+(T-n^2-n_1^{-2})^2>\frac{1}{2}(1+2\delta_1^2)(T-n^2)^2\bigg|n^2\right)<2\exp\left(-\frac{\delta_1^2 (T-{n}^2)}{2}\right)
\end{equation*}
Combining the above two inequalities we have
\begin{align}
	\label{H_2_2}
	&\PR \left((n_1^{-2})^2+(n_3^{-2})^2+\sum_{k=4}^{N+1} (n_k^{-2})^2>\left[\frac{1}{2}(1+2\delta_1^2)-\frac{1}{8}(1-2\delta_1)^2\right](T-n^2)^2\bigg|n^2\right)\nonumber\\
	<\,&\PR \left((n_1^{-2})^2+(n_3^{-2})^2+(T-n^2-n_1^{-2}-n_3^{-2})^2>\left[\frac{1}{2}(1+2\delta_1^2)-\frac{1}{8}(1-2\delta_1)^2\right](T-n^2)^2\bigg|n^2\right)\nonumber\\
	=\,&\PR \left((n_1^{-2})^2+(T-n^2-n_1^{-2})^2-2n_3^{-2}(T-n^2-n_1^{-2}-n_3^{-2})>\left[\frac{1}{2}(1+2\delta_1^2)-\frac{1}{8}(1-2\delta_1)^2\right](T-n^2)^2\bigg|n^2\right)\nonumber\\
	<\,&\PR \left((n_1^{-2})^2+(T-n^2-n_1^{-2})^2>\frac{1}{2}(1+2\delta_1^2)(T-n^2)^2\bigg|n^2\right)\nonumber\\
	&+\PR \left(n_3^{-2}(T-{n}^2-n_1^{-2}-n_3^{-2})<\frac{1}{16}(1-2\delta_1)^2(T-n^2)^2\bigg|n^2\right)\nonumber\\
	<\,&2\exp\left(-\frac{\delta_1^2(T-{n}^2)}{2}\right)+2\exp\left(-\frac{\delta_1^2(T-{n}^2)}{2}\right)
	= 4\exp\left(-\frac{\delta_1^2(T-{n}^2)}{2}\right)
\end{align}
The first inequality follows from $\sum_{k=4}^{N+1} n_k^{-2} = T-n^2-n_1^{-2}-n_3^{-2}$ and thus $\sum_{k=4}^{N+1} (n_k^{-2})^2 \leq (T-n^2-n_1^{-2}-n_3^{-2})^2$. The first equality follows from $(n_3^{-2})^2+(T-n^2-n_1^{-2}-n_3^{-2})^2=(T-n^2-n_1^{-2})^2-2n_3^{-2}(T-n^2-n_1^{-2}-n_3^{-2})$.

Combine \eqref{H_2_1} and \eqref{H_2_2} we have
\begin{align}
	\label{H_2|n_2}
	&\PR \left(H_2>\frac{1}{2}(1+2\delta_1^2)T-\frac{1}{8}(1-2\delta_1)^2(T-n^2)\bigg|n^2\right)\nonumber\\
	<\,&\PR \left( \frac{(n_1^2)^2+(n^2-n_1^2)^2}{n^2}>\frac{1}{2}(1+2\delta_1^2)n^2\bigg|n^2\right)\nonumber\\
	&+\PR \left( \frac{(n_1^{-2})^2+(n_3^{-2})^2+\sum_{k=4}^{N+1} (n_k^{-2})^2}{T-n^2}>\left[\frac{1}{2}(1+2\delta_1^2)-\frac{1}{8}(1-2\delta_1)^2\right](T-n^2)\bigg|n^2\right)\nonumber\\
	=\,&\PR \left((n_1^2)^2+(n^2-n_1^2)^2>\frac{1}{2}(1+2\delta_1^2)(n^2)^2\bigg|n^2\right)\nonumber\\
	&+\PR \left((n_1^{-2})^2+(n_3^{-2})^2+\sum_{k=4}^{N+1} (n_k^{-2})^2>\left[\frac{1}{2}(1+2\delta_1^2)-\frac{1}{8}(1-2\delta_1)^2\right](T-n^2)^2\bigg|n^2\right)\nonumber\\
	<\,&2\exp\left(-\frac{\delta_1^2 {n}^2}{2}\right)+4\exp\left(-\frac{\delta_1^2(T-{n}^2)}{2}\right).
\end{align}
Next we bound the unconditional probability based on the conditional probability.
We have
\begin{align}
	\label{H_2}
	&\PR \left(H_2>\frac{1}{2}(1+2\delta_1^2)T-\frac{1}{8}(1-2\delta_1)^2\frac{(1-\delta)T}{2}\right)\nonumber\\
	<\,&\PR\left(\left|n^2-\frac{1}{2}T\right|>\frac{\delta}{2}T\right)\nonumber\\
	&+\sum_{k:|k-T/2|\le \delta T/2}\PR(n^2=k)\PR \left(H_2>\frac{1}{2}(1+2\delta_1^2)T-\frac{1}{8}(1-2\delta_1)^2\frac{(1-\delta)T}{2}\bigg|n^2=k\right)\nonumber\\
	<\,&2\exp\bigg(-\frac{\delta^2 T}{4}\bigg)+\sum_{k:|k-T/2|\le \delta T/2}\PR(n^2=k)\left(2\exp\left(-\frac{\delta_1^2 k}{2}\right)+4\exp\left(-\frac{\delta_1^2(T-k)}{2}\right)\right) \nonumber\\
	\le\,&2\exp\bigg(-\frac{\delta^2 T}{4}\bigg)+\sum_{k:|k-T/2|\le \delta T/2}\PR(n^2=k)\times 6\exp\left(-\frac{\delta_1^2 (1-\delta)T}{4}\right)\nonumber\\
	=\,&2\exp\bigg(-\frac{\delta^2 T}{4}\bigg)+6\exp\bigg(-\frac{\delta_1^2(1-\delta) T}{4}\bigg)=\,8\exp\bigg(-\frac{\delta^2 T}{4}\bigg).
\end{align}
Next we choose a proper value for $\delta$.
By inequality \eqref{H_1} and \eqref{H_2}, we want to find $\delta$ such that with high probability, we have
\begin{equation*} \begin{aligned}
	H_1 \ge \frac{(1-\delta)T}{2}+\left[\frac{1}{4} (1-\sqrt{2}\delta_1)^2+\frac{1}{16}(1-2\delta_1)^2\right]\frac{(1+\delta)T}{2}>\frac{1}{2}(1+2\delta_1^2)T-\frac{1}{8}(1-2\delta_1)^2\frac{(1-\delta)T}{2}\ge  H_2
\end{aligned} \end{equation*}
where $\delta_1=\delta/\sqrt{1-\delta}$.
We also have the constraint that $0<2\delta_1<1$, which is equivalent to $0<\delta<\frac{\sqrt{17}-1}{2} \approx 0.39$. Solving the above inequality for $0<\delta<0.39$ we have $0<\delta\le 0.166185$.
Let $\delta=0.166185$. Then $4/\delta^2 \approx 145$.
Plugging into \eqref{H_1} and \eqref{H_2}, we have
\begin{align}
	\PR\left(H_1<0.512041T\right)&<4\exp\left(-\frac{T}{145}\right)
	\label{Prob_Small_H_1}
	\\
	\label{Prob_Large_H_2}
	\PR\left(H_2>0.512041T\right)&<8\exp\left(-\frac{T}{145}\right)
\end{align}
Therefore
$
	\PR \left(H_1<H_2\right)<12\exp(- \frac{T}{145}).
$
%	\begin{equation}
%	\PR \left(H_1<H_2\right)<\PR\left(H_1<0.512041T\right)+\PR\left(H_2>0.512041T\right)<12\exp\left(-\frac{T}{145}\right)
%	\end{equation}
This implies that $G_1<G_2$ with high probability.
Notice that the probability bound in the above equation doesn't depend on $N$.
Next we consider $j \geq 3$. Recall that
\begin{equation}
	\label{H_j formulation}
	H_j = \frac{\sum_{k=1}^j (n_k^j)^2}{n^j}+\frac{\sum_{k=1}^{j-1} (n_k^{-j})^2 +\sum_{k=j+1}^{N+1} (n_k^{-j})^2}{T-n^j}
\end{equation}
Consider some $\delta_2>0$. From \eqref{n^j bound} we have
\begin{equation}
	\PR\left(\left|n^j-\frac{1}{2}T\right|>\frac{\delta_2}{2}T\right) < 2\exp\bigg(-\frac{\delta_2^2 T}{4}\bigg)
\end{equation}
We investigate the second term of \eqref{H_j formulation}.
Conditional $n^j$, we have $n_1^{-j} \sim B(T-{n}^j,1/2)$, $T-{n}^j-n_1^{-j} \sim B(T-{n}^j,1/2)$, $n_2^{-j} \sim B(T-{n}^j,1/4)$ and $T-{n}^j-n_1^{-j}-n_2^{-j} \sim B(T-{n}^j,1/4)$. Define $\delta_3 \triangleq \delta_2/\sqrt{1-\delta_2}$. Then similar to (\ref{H_2_2}), we have
\begin{align}
	\label{H_j_2}
	&\PR \left((n_1^{-j})^2+(n_2^{-j})^2+\sum_{k=3}^{j-1} (n_k^{-j})^2 +\sum_{k=j+1}^{N+1} (n_k^{-j})^2>\left[\frac{1}{2}(1+2\delta_3^2)-\frac{1}{8}(1-2\delta_3)^2\right](T-n^j)^2\bigg|n^j\right)\nonumber\\
	<\,&\PR \left((n_1^{-j})^2+(n_2^{-j})^2+(T-n^j-n_1^{-j}-n_2^{-j})^2>\left[\frac{1}{2}(1+2\delta_3^2)-\frac{1}{8}(1-2\delta_3)^2\right](T-n^j)^2\bigg|n^j\right)\nonumber\\
	<\,&4\exp\left(-\frac{\delta_3^2(T-{n}^j)}{2}\right).
\end{align}
Then similarly we can bound the first term of \eqref{H_j formulation} since $n_1^{j} \sim B({n}^j,1/2)$, ${n}^j-n_1^{j} \sim B({n}^j,1/2)$, $n_2^{j} \sim B({n}^j,1/4)$ and ${n}^j-n_1^{j}-n_2^{j} \sim B(T-{n}^j,1/4)$.
\begin{align}
	\label{H_j_1}
	&\PR \left(\sum_{k=1}^j (n_k^j)^2>\left[\frac{1}{2}(1+2\delta_3^2)-\frac{1}{8}(1-2\delta_3)^2\right](n^j)^2\bigg|n^j\right)\nonumber\\
	<\,&\PR \left((n_1^j)^2+(n_2^j)^2+(n^j-n_1^j-n_2^j)^2>\left[\frac{1}{2}(1+2\delta_3^2)-\frac{1}{8}(1-2\delta_3)^2\right](n^j)^2\bigg|n^j\right)\nonumber\\
	<\,&4\exp\left(-\frac{\delta_3^2{n}^j}{2}\right).
\end{align}
Combining \eqref{H_j_2} and \eqref{H_j_1}, we have
\begin{align}
	\label{H_j|n_j}
	&\PR \left(H_j>[\frac{1}{2}(1+2\delta_3^2)-\frac{1}{8}(1-2\delta_3)^2]T\bigg|n^j\right)\nonumber\\
	<\,&\PR \left(\sum_{k=1}^j (n_k^j)^2>[\frac{1}{2}(1+2\delta_3^2)-\frac{1}{8}(1-2\delta_3)^2](n^j)^2\bigg|n^j\right)\nonumber\\
	&+\PR \left(\sum_{k=1}^{j-1} (n_k^{-j})^2 +\sum_{k=j+1}^{N+1} (n_k^{-j})^2>[\frac{1}{2}(1+2\delta_3^2)-\frac{1}{8}(1-2\delta_3)^2](T-n^j)^2\bigg|n^j\right)\nonumber\\
	<\,&4\exp\left(-\frac{\delta_3^2 {n}^j}{2}\right)+4\exp\left(-\frac{\delta_3^2(T-{n}^j)}{2}\right)
\end{align}
Using a similar argument, we can bound the unconditional probability:
\begin{align}
	\label{H_j}
	\PR \left(H_j>[\frac{1}{2}(1+2\delta_3^2)-\frac{1}{8}(1-2\delta_3)^2]T\right)
	%	<\,&\PR\left(\left|n^j-\frac{1}{2}T\right|>\frac{\delta_2}{2}T\right)+\PR \left(H_j>[\frac{1}{2}(1+2\delta_3^2)-\frac{1}{8}(1-2\delta_3)^2]T\bigg|n^j=\tilde{n}^j,\left|\tilde{n}^j-\frac{1}{2}T\right|<\frac{\delta_2}{2}T\right)\nonumber\\
	%	<\,&2\exp\bigg(-\frac{\delta_2^2 T}{4}\bigg)+8\exp\bigg(-\frac{\delta_2^2 T}{4}\bigg)\nonumber\\
	<\,10\exp\bigg(-\frac{\delta_2^2 T}{4}\bigg).
\end{align}
Similarly we can calculate $\delta_2$.
By \eqref{H_j} and \eqref{Prob_Small_H_1}, we have the following condition for $\delta_2$
$
	\frac{1}{2}(1+2\delta_3^2)-\frac{1}{8}(1-2\delta_3)^2<0.512041,
$
where $\delta_3=\delta_2/\sqrt{1-\delta_2}$. Again when we consider $0<\delta_2<0.39$, the above inequality is equivalent to $0<\delta_2<0.200261$.
Let $\delta_2=0.200261$, then $4/\delta_2^2 \approx 99.74$, so we have for all $j \geq 3$
\begin{equation}
	\label{Prob_Large_H_j}
	\PR \left(H_j > 0.512041T \right)<10\exp\left(-\frac{T}{100}\right)
\end{equation}
Now note that
\begin{align}
	\label{eq:first_split}
	&\PR \left(\text{first split not on product one}\right)=\,\PR \left(G_1>\min\{G_j|2 \leq j\leq N\}\right)\nonumber\\
	=\,&\PR \left(H_1<\max\{H_j|2 \leq j\leq N\}\right)
	<\,\PR \left(H_1<0.512041T\right)+\sum_{j=2}^N \PR \left(H_j >0.512041T\right)\nonumber\\
	<\,&4\exp\bigg(-\frac{\delta^2T}{4}\bigg)+8\exp\bigg(-\frac{\delta^2T}{4}\bigg)+(N-2)\times 10\exp\bigg(-\frac{\delta_2^2 T}{4}\bigg)\nonumber\\
	=\,&12\exp\left(-\frac{T}{145}\right)+10(N-2)\exp\left(-\frac{T}{100}\right).
\end{align}
This completes the proof of Lemma~\ref{lm:empirical_cut}.
\endproof

{\revise
Next we are going to show that the probability of the second split on product two, third split on product 3 and so on, can be bounded similarly.
Let $\Tscr$ denote the training set.
Define a sequence of subsets of $\Tscr$ as follows: $\Tscr_i \triangleq \{S \in \Tscr| 1,2,...,i-1 \notin S\}$.
That is, $\Tscr_i$ only contains assortments that include a subset of $\{i,i+1,\dots,N\}$.
Let $T_i \triangleq |\Tscr_i|$ denote the cardinality of set $\Tscr_i$.
Notice that $\Tscr_1=\Tscr$ and $T_1=T$. for $2 \leq i \leq N$, since $T_i \sim B(T,1/2^{i-1})$, by the Chernoff inequality we have
\begin{equation}
	\label{eq:T_i}
	\PR \left(T_i<(1-\delta_0)\frac{1}{2^{i-1}}T\right)<\exp(-\frac{\delta_0^2 T}{2^i}),
\end{equation}
where $0<\delta_0<1$.   

Define the event $A_i \triangleq \{\text{product } i \text{ is the best split for training set } \Tscr_i\}$, and let $\bar{A}_i$ denote the complement of event $A_i$.
Then conditional on $T_i$, we can bound the probability of event $\bar{A}_i$ using \eqref{eq:first_split} (with only $N-i+1$ products):
\begin{equation}
	\label{eq:A_i|T_i}
	\PR \left(\bar{A}_i\right|T_i) < 12\exp\left(-\frac{T_i}{145}\right)+10(N-i-1)\exp\left(-\frac{T_i}{100}\right)
\end{equation}
The unconditional probability can be bounded by combining \eqref{eq:T_i} and \eqref{eq:A_i|T_i}:
\begin{align}
	\label{eq:A_i}
	\PR \left(\bar{A}_i\right)<\,&\PR \left(T_i<(1-\delta_0)\frac{1}{2^{i-1}}T\right)+\sum_{k:k\ge (1-\delta_0)T/2^{i-1}}\PR(T_i=k)\PR\left(\bar{A}_i|T_i=k\right)\nonumber\\
	<\,&\exp(-\frac{\delta_0^2 T}{2^i})+\sum_{k:k\ge(1-\delta_0)T/2^{i-1}}\PR(T_i=k)\left(12\exp\left(-\frac{k}{145}\right)+10(N-i-1)\exp\left(-\frac{k}{100}\right)\right)\nonumber\\
	\leq\,&\exp(-\frac{\delta_0^2 T}{2^i})+12\exp\left(-\frac{(1-\delta_0)T}{145\cdot 2^{i-1}}\right)+10(N-i-1)\exp\left(-\frac{(1-\delta_0)T}{100 \cdot 2^{i-1}}\right)
\end{align}
Solving the equation $\frac{\delta_0^2 T}{2^i}=\frac{(1-\delta_0)T}{145\cdot 2^{i-1}}$, we get $\delta_0 \approx 0.1107$. Note that the leaf node size $l = (1-\delta_0) T / 2^{i-1} \approx 1.77 T / 2^i$, which ensures the split occurs whenever $T_i > (1-\delta_0)T/2^{i-1}$. Then we have
\begin{equation*}
	\PR \left(\bar{A}_i\right)\le13\exp\left(-\frac{T}{164\cdot 2^{i-1}}\right)+10(N-i-1)\exp\left(-\frac{T}{113 \cdot 2^{i-1}}\right).
\end{equation*}
If all the events $A_1, A_2, ..., A_m$ happen, we can get the right split for the first $m$ step.
That is, the first split is on product one, the second split is on product two, \dots, the $m$th split is on product $m$. We can bound the probability by the union bound:
\begin{align}
	\PR \left(\cap_{i=1}^m A_i\right)=\,&1-\PR \left(\cup_{i=1}^m \bar{A}_i\right)\nonumber\\
	\ge\,&1-\sum_{i=1}^m \PR(\bar{A}_i)\nonumber\\
	\ge \,& 1-\sum_{i=1}^m \left[13\exp\left(-\frac{T}{164\cdot 2^{i-1}}\right)+10(N-i-1)\exp\left(-\frac{T}{113 \cdot 2^{i-1}}\right)\right].
\end{align}
If the first $m$ splits match the products, then the assortments including at least one product among $\left\{1,\dots,m\right\}$ can be correctly classified.
Therefore, with a probability at least $\PR \left(\cap_{i=1}^m A_i\right)$, we can correctly predicts the choices of more than $(1-1/2^m)2^N$ assortments
Given $\epsilon>0$, letting $m=\lceil \log_2 \frac{1}{\epsilon} \rceil$ completes the proof.}
\endproof

\proof{Proof of Theorem~\ref{thm:information_gain_ratio}.}
We will start from proving the first split by employing information gain ratio is at product 1. The remaining split can be shown by induction. 

For a given dataset $\mathcal{D}$ and $i \in [N]$, let $\hat{q}_i$ and $\hat{p}_i$ denote the empirical frequency that product $i$ is offered and purchased, respectively. 
To describe the dataset in child node after a split, for $i, j \in [N]$, we define $\hat{p}_i^j$ and $\hat{p}_i^{-j}$ as the empirical frequency of purchasing product $i$ given that product $j$ is offered and not offered, respectively. 
Notice that $\hat{p}_j^i = 0$ and $\hat{p}_j^{-i} = \hat{p}_j / (1-\hat{q}_i)$ if $1 \leq i < j \leq N$;  $\hat{p}_i^i = \hat{p}_i / \hat{q}_i$ and $\hat{p}_i^{-i} = 0$ for all $i \in [N]$.

The entropy of the dataset $\mathcal{D}$ is 
$
	H(\mathcal{D}) = -\sum_{j=1}^N \hat{p}_j \ln(\hat{p}_j).
$
Let $\texttt{IV}(i) = -\hat{q}_i \ln(\hat{q}_i) - (1-\hat{q}_i) \cdot \ln(1-\hat{q}_i)$ for notational brevity, where $\texttt{IV}$ represents intrinsic value \citep{zhou2021machine}.
Notice that $\texttt{IGR}(\mathcal{D}, i) = \texttt{IG}(\mathcal{D},i) / \texttt{IV}(i)$.

The parent node splits into two child nodes after product $i$ is selected. Let $\mathcal{D}_i$ and $\mathcal{D}_{-i}$ denote the left node ($i \in S$) and right node ($i \notin S$), respectively. 
Suppose the first feature is split at product 1, then the information gain is
\begin{align}
	\texttt{IG}(\mathcal{D}, 1) &= H(\mathcal{D}) - \hat{q}_1 H(\mathcal{D}_1) - (1-\hat{q}_1) H(\mathcal{D}_{-1}) \\
	& = -\sum_{j=1}^N \hat{p}_j \ln(\hat{p}_j) + (1-\hat{q}_1) \sum_{j=2}^N \hat{p}_j^{-1} \ln(\hat{p}_j^{-1}) \nonumber \\
	& = -\sum_{j=1}^N \hat{p}_j \ln(\hat{p}_j) + (1-\hat{q}_1) \sum_{j=2}^N \frac{\hat{p}_j}{1-\hat{q}_1} \ln\Big(\frac{\hat{p}_j}{1-\hat{q}_1}\Big) \nonumber \\
	& = -\hat{p}_1 \ln(\hat{p}_1) - \sum_{j=2} \hat{p}_j \cdot \ln(1-\hat{q}_1) \nonumber \\
	& = -\hat{q}_1 \ln(\hat{q}_1) - (1-\hat{q}_1) \cdot \ln(1-\hat{q}_1) = \texttt{IV}(1) \nonumber 
\end{align}
where the second equality follows from $H(\mathcal{D}_1) = 0$, the third equality follows from $\hat{p}_j^{-1} = \hat{p}_j / (1-\hat{q}_1)$ for $j \geq 2$, and the second last equality follows from $\hat{p}_1= \hat{q}_1$.
Since $\texttt{IG}(\mathcal{D}, 1) = \texttt{IV}(1)$, we have the information gain ratio $\texttt{IGR}(\mathcal{D}, 1) = 1$ by definition. 

Next we consider if the first feature is split at product $i \neq 1$. 
For the left and right nodes, the entropy functions are:
\begin{align}
	H(\mathcal{D}_i) = -\sum_{j=1}^i \hat{p}_j^i \ln(\hat{p}_j^i), \quad 
	H(\mathcal{D}_{-i}) = -\sum_{j=1}^{i-1} \hat{p}_j^{-i} \ln(\hat{p}_j^{-i}) - \sum_{j=i+1}^N \hat{p}_j^{-i} \ln(\hat{p}_j^{-i})
\end{align}
The information gain is 
\begin{align}
	\texttt{IG}(\mathcal{D}, i) &= H(\mathcal{D}) - \hat{q}_i H(\mathcal{D}_i) - (1-\hat{q}_i) H(\mathcal{D}_{-i}) \\
	& = -\sum_{j=1}^N \hat{p}_j \ln(\hat{p}_j)+\hat{q}_i \sum_{j=1}^i \hat{p}_j^i \ln(\hat{p}_j^i)+(1-\hat{q}_i) \sum_{j=1}^{i-1} \hat{p}_j^{-i} \ln(\hat{p}_j^{-i})+(1-\hat{q}_i)\sum_{j=i+1}^N \frac{\hat{p}_j}{1-\hat{q}_i} \ln\Big(\frac{\hat{p}_j}{1-\hat{q}_i}\Big) \nonumber \\
	& = -\sum_{j=1}^{i} \hat{p}_j \ln(\hat{p}_j)+\hat{q}_i \sum_{j=1}^i \hat{p}_j^i \ln(\hat{p}_j^i)+(1-\hat{q}_i) \sum_{j=1}^{i-1} \hat{p}_j^{-i} \ln(\hat{p}_j^{-i})-\sum_{j=i+1}^N \hat{p}_j \ln(1-\hat{q}_i) \nonumber \\
	& = -\hat{p}_i \ln(\hat{q}_i)-\sum_{j=1}^{i-1} \hat{p}_j \ln(\hat{p}_j)+\hat{q}_i \sum_{j=1}^{i-1} \hat{p}_j^i \ln(\hat{p}_j^i)+(1-\hat{q}_i) \sum_{j=1}^{i-1} \hat{p}_j^{-i} \ln(\hat{p}_j^{-i})-\sum_{j=i+1}^N \hat{p}_j \ln(1-\hat{q}_i), \nonumber
\end{align}
where the second equality follows from $\hat{p}_j^{-i} = \hat{p}_j / (1-\hat{q}_i)$ for $j > i$, and the last equality follows from $-\hat{p}_i \ln(\hat{p}_i)+\hat{q}_i \hat{p}_i^i \ln(\hat{p}_i^i)=-\hat{p}_i \ln(\hat{q}_i)$.

To show the first split is at product 1 by employing information gain ratio, we only need to show $\texttt{IGR}(\mathcal{D}, i) < \texttt{IGR}(\mathcal{D}, 1) = 1$ for $i \geq 2$. Equivalently, we want to show $\texttt{IG}(\mathcal{D}, i) < \texttt{IV}(i)$ for $i \geq 2$. 
We have
\begin{align}
	&\texttt{IG}(\mathcal{D}, i) - \texttt{IV}(i) \\
	=& (\hat{q}_i-\hat{p}_i)\ln(\hat{q}_i)+\Big(1-\hat{q}_i-\sum_{j=i+1}^N \hat{p}_j\Big)\ln(1-\hat{q}_i) +\sum_{j=1}^{i-1} \Bigg[-\hat{p}_j \ln(\hat{p}_j)+\hat{q}_i \hat{p}_j^i \ln(\hat{p}_j^i)+(1-\hat{q}_i)\hat{p}_j^{-i}\ln(\hat{p}_j^{-i})\Bigg] \nonumber \\
	=&\hat{q}_i \ln(\hat{q}_i) \sum_{j=1}^{i-1} \hat{p}_j^i + (1-\hat{q}_i) \ln(1-\hat{q}_i) \sum_{j=1}^{i-1} \hat{p}_j^{-i} +\sum_{j=1}^{i-1} \Bigg[-\hat{p}_j \ln(\hat{p}_j)+\hat{q}_i \hat{p}_j^i \ln(\hat{p}_j^i)+(1-\hat{q}_i)\hat{p}_j^{-i}\ln(\hat{p}_j^{-i})\Bigg] \nonumber \\
	=& \sum_{j=1}^{i-1} \Bigg[-\hat{p}_j \ln(\hat{p}_j)+\hat{q}_i \hat{p}_j^i \ln(\hat{q}_i\hat{p}_j^i)+(1-\hat{q}_i)\hat{p}_j^{-i}\ln\big((1-\hat{q}_i)\hat{p}_j^{-i}\big)\Bigg], \nonumber
\end{align}
where the second equality follows from $\hat{q}_i -\hat{p}_i = \hat{q}_i \sum_{j=1}^{i-1} \hat{p}_j^i$ and $1-\hat{q}_i-\sum_{j=i+1}^N \hat{p}_j=(1-\hat{q}_i) (1-\sum_{j=i+1}^N \hat{p}_j^{-i})=(1-\hat{q}_i)\sum_{j=1}^{i-1} \hat{p}_j^{-i}$.
Since $\hat{p}_j = \hat{q}_i \hat{p}_j^i + (1-\hat{q}_i) \hat{p}_j^{-i}$, we can conclude that $\texttt{IG}(\mathcal{D}, i) \leq \texttt{IV}(i)$, i.e., $\texttt{IGR}(\mathcal{D}, i) \leq 1$ for all $i \geq 2$. Moreover, $\texttt{IGR}(\mathcal{D}, i) = 1$ only when for all $j = 1, \ldots, i-1$, $\hat{q}_i \hat{p}_j^i = 0$ or $(1 - \hat{q}_i) \hat{p}_j^{-i} = 0$. 

We have $0 < \hat{q}_i < 1$ for $i \in [N]$, otherwise splitting at product $i$ is invalid. By our assumptions, for $i \geq 2$, there exists $S \in \mathcal{D}$ s.t., $i-1 \in S, i \in S$ and $\{1, \ldots, i-2\} \cap S = \emptyset$, so $\hat{p}_{i-1}^i \neq 0$; similarly, there exists $S \in \mathcal{D}$ s.t., $i-1 \in S$ and $\{1, \ldots, i-2, i\} \cap S = \emptyset$, so $\hat{p}_{i-1}^{-i} \neq 0$. We can conclude that $\texttt{IGR}(\mathcal{D}, i) < 1 = \texttt{IGR}(\mathcal{D}, 1)$ for all $i=2, \ldots, N$, and the first split is at product 1. 

After the first split, the left node stops splitting because all samples have the same choice product 1. The right node with dataset $\mathcal{D}_{-1}$ will split on product 2 since $\texttt{IGR}(\mathcal{D}_{-1}, i) < \texttt{IGR}(\mathcal{D}_{-1}, 2) = 1$ for $i = 3, \ldots, N$. By induction, we can show that at $i$th iteration, every left node stops splitting, and the right node will split at product $i$. Therefore, the random forest algorithm recovers the preference ranking DCM. 
$\hfill \Box$
\endproof

\proof{Proof of Proposition~\ref{prop:IGR_unknown_rank}.}
% We assume that the preference ranking is $1 \succ 2 \succ \ldots \succ N \succ 0$ without loss of generality. We can check the conditions in Theorem \ref{thm:information_gain_ratio} hold as follows. 
We will show that the training data must include all assortments with 2 products, i.e., the assortment $\{i,j\}$ such that $i,j \in [N]$ and $i \neq j$. The total number of these assortments are $N(N-1) / 2$.
Suppose that the assortment $\{i, j\}$ is not observed in the training data. Without loss of generality,  we assume $i = N-1, j = N$. Then, the preference ranking $1 \succ 2 \succ \ldots \succ N-2 \succ N-1 \succ N \succ 0$ and $1 \succ 2 \succ \ldots \succ N-2 \succ N \succ N-1 \succ 0$ exhibit the same choice behavior for all assortments except $\{N-1, N\}$. Therefore, they cannot be distinguished or recovered by any algorithm. 
\endproof

\proof{Proof of Proposition~\ref{prop:link-func}.}
We need to show that the $b$th tree constructed by the algorithm of both link functions returns the same partition (in the sense that each region contains the same set of observations in the training data) of the predictor space $[0,1]^N$ and the same class labels in each region/leaf.
The class labels are guaranteed to be the same because we control the internal randomizer in Step~\ref{step:random-label-p}.
To show the partitions are the same, it suffices to show that each split creates regions that are identical for the two link functions in the sense that the resulting regions contain the same set of observations.
We prove this claim by induction.

Before the construction of the $b$th tree, because the internal randomizers in Step~\ref{step:select-subsample-p} are equalized, the root node $[0,1]^N$ for both link functions contains the same set of observations.
Now focusing on a leaf node in the middle of constructing the $b$th tree for both link functions.
We use $[l^{(j)}_1,u^{(j)}_1]\times\dots\times [l^{(j)}_N,u^{(j)}_N]\subset [0,1]^N$ to denote the region of the leaf node for link functions $j=1,2$.
By the inductive hypothesis, both regions contain the same set of observations.
Without loss of generality, we assume that the regions contain $\left\{g_1(\bm p_t)\right\}_{t=1}^{T_1}$ and $\left\{g_2(\bm p_t)\right\}_{t=1}^{T_1}$, respectively.
After Step~\ref{step:choose-direction-p}, the same set of candidate splitting products are selected.
To show that Step~\ref{step:criterion-p} results in the same split in the two regions, consider a given split product $m$ and split point $x^{j}$ for $j=1,2$.
If $[l^{(1)}_1,u^{(1)}_1]\times\dots[l^{(1)}_m, x^{(1)}]\times\dots\times [l^{(1)}_N,u^{(1)}_N]$ and $[l^{(2)}_1,u^{(2)}_1]\times\dots[l^{(2)}_m, x^{(2)}]\times\dots\times [l^{(2)}_N,u^{(2)}_N]$ contain the same set of observations, i.e., for $t=1,\dots,T_1$
\begin{equation*} \begin{aligned}
	&g_1(\bm p_t)\in [l^{(1)}_1,u^{(1)}_1]\times\dots[l^{(1)}_m, x^{(1)}]\times\dots\times [l^{(1)}_N,u^{(1)}_N]\\
	\iff &g_2(\bm p_t)\in [l^{(2)}_1,u^{(2)}_1]\times\dots[l^{(2)}_m, x^{(2)}]\times\dots\times [l^{(2)}_N,u^{(2)}_N],
\end{aligned} \end{equation*}
then the Gini indices resulting from the splits are equal for the two link functions.
This is because the Gini index only depends on the class composition in a region instead of the locations of the input, and the splits above lead to the same class composition in the sub-regions.
This implies that in Step~\ref{step:choose-direction-p}, both trees are going to find the optimal splits that lead to the same division of training data in the sub-regions.
By induction and the recursive nature of the tree construction, Algorithm~\ref{alg:rf-pricing} outputs the same partition in the $b$th tree for both link functions, i.e., the training data is partitioned equally.
This completes the proof.
\endproof

\section{Aggregated Choice Data}\label{sec:aggregate-choice}
One of the most pressing practical challenges in data analytics is data quality.
In Section~\ref{sec:data_estimation}, the historical data $\left\{(i_t,\bm x_t)\right\}_{t=1}^T$ is probably the most structured and granular form of data a firm can hope to acquire.
While most academic papers studying the estimation of DCMs assume this level of granularity, in practice, it is frequent to see data in a more aggregate format.  As an example, consider an airline offering three service classes E, T and Q of a flight, where data is aggregated over different sales channels over a specific time window during which there may be changes in the offered assortments.  The company records information at certain time clicks as in Table~\ref{tab:sample-data}.
\begin{table}[h]
	\centering
	\caption{A sample daily data of offered service classes and the number of bookings.}
	\def\arraystretch{0.9}\begin{tabular}{crr}
		\hline
		Class & Closure percentage & \#Booking\\
		\hline
		E & 20\% & 2\\
		T & 0\% &5\\
		Q & 90\% &1\\
		\hline
	\end{tabular}
	\label{tab:sample-data}
\end{table}
For each class, the closure percentage reflects the fraction of time that the class is not open for booking, i.e., included in the assortment.
Thus, 100\% would imply that the corresponding class is not offered during that time window.
The number of bookings for each class is also recorded.
There may be various reasons behind the aggregation of data.
The managers may not realize the value of high-quality data or are unwilling to invest in the infrastructure and human resources to reform the data collection process.
%One of the authors has encountered this situation in practice with aggregate datasets as in Table~\ref{tab:sample-data}.

Fortunately, random forests can deal with aggregated choice data naturally.
Suppose the presented aggregated data has the form
$\left\{(\bm p_s, \bm b_s)\right\}_{s=1}^S$, where $\bm p_s\in [0,1]^N$ denotes the closure percentage of the $N$ products in day $s$, $\bm b_s\in \Z_+^{N+1}$ denotes the number of bookings\footnote{Again, we do not deal with demand censoring in this paper and assume that $\bm b_s$ has an additional dimension to record the number of consumers who do not book any class.}, and the data spans $S$ time windows.
We  transform the data into the desired form as follows:
for each time window $s$, we create $ D_s\triangleq \sum_{k=0}^N \bm b_s(k)$ observations, $\left\{(i_{s,k}, \bm x_{s,k})\right\}_{k=1}^{D_s}$.
The predictor $\bm x_{s,k}\equiv \bm 1-\bm p_s \in[0,1]^N$ and let the choices $i_{s,k}$ be valued $j$ for $b_s(j)$ times, for $j=0,\dots,N$.

To explain the intuition behind the data transformation, notice that we cannot tell from the data which assortment a customer faced when she made the booking.
We simply take an \emph{average} assortment that the customer may have faced, represented by $\bm 1-\bm p_s$.
In other words, if $1-\bm p_s(j)\in[0,1]$ is large, then it implies that product $j$ is offered most of the time during the day, and the transformation leads to the interpretation that consumers see a larger ``fraction'' of product $j$.
As the closure percentage has a continuous impact on the eventual choice, it is reasonable to transform the input into a Euclidean space $[0,1]^N$,
and build a smooth transition between the two ends $\bm p_s(j)=0$ (the product is always offered) and $\bm p_s(j)=1$ (the product is never offered).

The transformation creates a training dataset for classification with continuous input.
The random forest can accommodate the data with minimal adaptation.
In particular, all the steps in Algorithm~\ref{alg:rf} can be performed.
The tree may have different structures: because the predictor $\bm x$ may not be at the corner of the unit hypercube anymore,
the split points may no longer be at 0.5.

\section{Sample Code}\label{sec:sample-code}
\begin{center}
	\includegraphics[width=\textwidth]{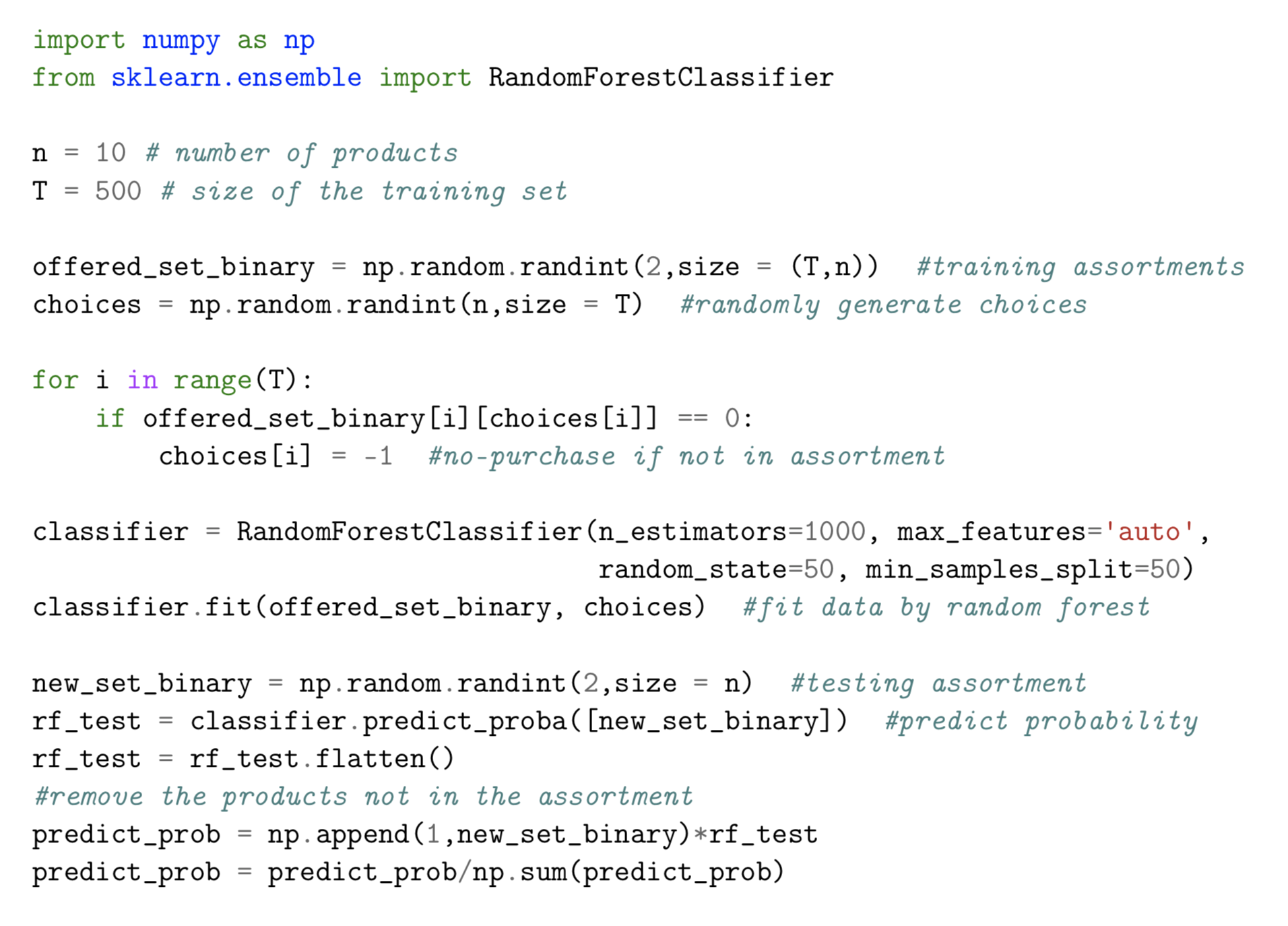}
\end{center}

}
\end{appendices}

% Acknowledgments here
\ACKNOWLEDGMENT{}	
		% CASE 1: BiBTeX used to constantly update the references
		%   (while the paper is being written).
		%\bibliographystyle{ormsv080} % outcomment this and next line in Case 1
		%\bibliography{<your bib file(s)>} % if more than one, comma separated
		
		% CASE 2: BiBTeX used to generate mypaper.bbl (to be further fine tuned)
		%\input{mypaper.bbl} % outcomment this line in Case 2
		
		%If you don't use BiBTex, you can manually itemize references as shown below.

%\newpage

%\begin{thebibliography}{}%
%
%\end{thebibliography}
	
\ECSwitch
\ECHead{Online Appendix: Extensions and Additional Experiments}
{\SingleSpacedXI

\section{Choice of Hyper-parameters} \label{sec:hyper-parameter}
In this section, we conduct an extensive numerical study on the hyper-parameters of Algorithms~\ref{alg:rf} to find the optimal choice and test the sensitivity.
The hyper-parameters investigated include the number of trees $B$,
the sub-sample size $z$, the number of products to split $m$, and the terminal leaf size $l$.
Note that in all the numerical studies in the paper, we use $B = 1000$, $z = T$, $m = \sqrt{N}$, and $l = 50$.
Such choice seems to be robust across all problem setups.

%We test when the ground truth model is MNL and 10 segments mixture of rank list model.
%The performance is in general robust on tunable parameters. Moreover, based on the RMSE result and running time we choose  in our further numerical studies.

\subsection{Rank-based DCM}\label{sec:hyper-para-rank}
In this section, we generate the data using the rank-based DCM with 10 customer types.
We consider $N \in \{10,30,50\}$ products.
We first generate the training data using the rank-based DCM, where the proportion/weight of 10 types follows a Dirichlet distribution.
For each type, the preference is a random permutation of $N$ products and the outside option.
The customer will choose the product (or the outside option) ranked the highest among the offered assortment.
If the outside option is ranked higher than all products in the assortment, then she leaves without purchase.
The training data consists of $\tilde{T} \in \{30, 150, 600\}$ periods.
In each period, we generate 10 transactions under a selected assortment, whose selection method is given shortly.
The total size of the training data is thus $T = 10 \tilde{T}$.
For $N=10$, we randomly and uniformly select an assortment from all $2^N-1$ assortments.
For $N =30$ and $50$, each product has a probability $1/6$ and $1/10$ respectively to be included in an assortment, so the average number of products in an assortment is always 5 for $N\in \left\{10, 30, 50\right\}$.
We make this choice to compare the RMSE \eqref{RMSE_soft}
across different $N$, because for assortments of different sizes, their RMSE usually differ substantially.
However, since we cannot enumerate $2^{30}$ or $2^{50}$ assortments to compute the RMSE, we randomly sample 10,000 assortments as a test set $\mathcal{T}_{test}$ to approximate the RMSE by \eqref{eq:RMSE_MC}.

We set $B = 1000, z = T, m = \sqrt{N}, l = 50$ as the default hyper-parameters.
We vary one of them to investigate the effect while fixing the other three.
When testing a combination of hyper-parameters, we generate 100 independent training datasets and compute the average and standard deviation of the RMSE.
Table \ref{tab:rank_z} shows the average and standard deviation of RMSE for different $z$ when $B = 1000, m = \sqrt{N}, l = 50$.
Similarly, we test $m$ and $l$ in Tables \ref{tab:rank_m} and \ref{tab:rank_l}.
For the number of trees $B$, we report RMSE and running time in Tables \ref{tab:rank_B} and \ref{tab:rank_B_time} respectively.
We can see that our default choice is among the best in all the cases.
From Tables \ref{tab:rank_B} and \ref{tab:rank_B_time} we can see that the performance is better when $B$ is large.
However, when $B \geq 1000$, the RMSE is  almost the same. But the running time grows linearly in $B$.
Therefore, choosing $B = 1000$ strikes a balance between the performance and computational efficiency.

\begin{table}[h]
	\begin{center}
		\caption{The average and standard deviation of RMSE using random forests with different sub-sample size $z$ when the training data is generated by the rank-based model.}
		\begin{tabular}{rrccccc}
			\hline
			$N$ & $T$  & $z = T/20$      & $z = T/10$      & $z=T/4$                  & $z=T/2$                  & $\bm {z=T}$           \\
			\hline
			10  & 300  & 0.113 (0.017) & 0.113 (0.017) & 0.099 (0.013) & 0.092 (0.012)          & \textbf{0.089} (0.012) \\
			10  & 1500 & 0.092 (0.010) & 0.078 (0.008) & 0.066 (0.006) & 0.060 (0.006)          & \textbf{0.057} (0.005) \\
			10  & 6000 & 0.063 (0.005) & 0.053 (0.005) & 0.044 (0.005) & \textbf{0.040} (0.004) & 0.041 (0.004)          \\
			30  & 300  & 0.217 (0.027) & 0.218 (0.027) & 0.187 (0.021) & 0.171 (0.017)          & \textbf{0.168} (0.017) \\
			30  & 1500 & 0.161 (0.020) & 0.137 (0.012) & 0.124 (0.008) & 0.117 (0.008)          & \textbf{0.114} (0.008) \\
			30  & 6000 & 0.117 (0.009) & 0.104 (0.007) & 0.092 (0.007) & 0.087 (0.007)          & \textbf{0.085} (0.008) \\
			50  & 300  & 0.256 (0.030) & 0.256 (0.030) & 0.224 (0.023) & 0.209 (0.020)          & \textbf{0.206} (0.020) \\
			50  & 1500 & 0.188 (0.018) & 0.155 (0.012) & 0.143 (0.010) & 0.140 (0.010)          & \textbf{0.139} (0.011) \\
			50  & 6000 & 0.128 (0.008) & 0.121 (0.006) & 0.111 (0.007) & 0.108 (0.008)          & \textbf{0.107} (0.009) \\
			\hline
		\end{tabular}
		\label{tab:rank_z}
	\end{center}
\end{table}

\begin{table}[t]
	\begin{center}
		\caption{The average and standard deviation of RMSE using random forests with different split  products $m$ when the training data is generated by the rank-based model.}
		\begin{tabular}{rrccccc}
			\hline
			$N$ & $T$  & $m = 1$                  & $m = \log N$              & $\bm {m = \sqrt{N}}$ & $m = N/2$       & $m = N$         \\
			\hline
			10  & 300  & 0.096 (0.013)          & 0.091 (0.012)          & \textbf{0.089} (0.012) & 0.090 (0.013) & 0.102 (0.017) \\
			10  & 1500 & 0.071 (0.007)          & 0.061 (0.005)          & \textbf{0.057} (0.005) & 0.059 (0.006) & 0.068 (0.008) \\
			10  & 6000 & 0.051 (0.005)          & 0.043 (0.003)          & \textbf{0.041} (0.004) & 0.043 (0.005) & 0.052 (0.006) \\
			30  & 300  & 0.171 (0.020)          & \textbf{0.166} (0.017) & 0.168 (0.017)          & 0.180 (0.021) & 0.191 (0.024) \\
			30  & 1500 & 0.131 (0.015)          & 0.116 (0.008)          & \textbf{0.114} (0.008) & 0.129 (0.017) & 0.140 (0.020) \\
			30  & 6000 & 0.109 (0.012)          & 0.088 (0.006)          & \textbf{0.085} (0.008) & 0.096 (0.015) & 0.104 (0.016) \\
			50  & 300  & \textbf{0.200} (0.019) & 0.201 (0.018)          & 0.206 (0.020)          & 0.222 (0.027) & 0.230 (0.030) \\
			50  & 1500 & 0.155 (0.016)          & \textbf{0.139} (0.009) & 0.139 (0.011)          & 0.162 (0.021) & 0.171 (0.021) \\
			50  & 6000 & 0.127 (0.014)          & 0.108 (0.007)          & \textbf{0.107} (0.009) & 0.126 (0.017) & 0.134 (0.019) \\
			\hline
		\end{tabular}
		\label{tab:rank_m}
	\end{center}
\end{table}

\begin{table}[t]
	\begin{center}
		\caption{The average and standard deviation of RMSE using random forests with different leaf sizes $l$ when the training data is generated by the rank-based model.}
		\begin{tabular}{rrccccc}
			\hline
			$N$ & $T$  & $l = 1$         & $l = 10$                 & $\bm {l = 50}$                 & $l = 100$                & $l = 200$       \\
			\hline
			10  & 300  & 0.094 (0.011) & 0.094 (0.011)          & \textbf{0.089} (0.012) & 0.095 (0.013) & 0.113 (0.017) \\
			10  & 1500 & 0.073 (0.004) & 0.071 (0.004)          & \textbf{0.057} (0.005) & 0.061 (0.006) & 0.068 (0.006) \\
			10  & 6000 & 0.076 (0.004) & 0.074 (0.004)          & \textbf{0.041} (0.004) & 0.041 (0.005) & 0.046 (0.005) \\
			30  & 300  & 0.166 (0.017) & \textbf{0.166} (0.017) & 0.168 (0.017)          & 0.177 (0.018) & 0.217 (0.027) \\
			30  & 1500 & 0.114 (0.010) & \textbf{0.114} (0.010) & 0.114 (0.008)          & 0.121 (0.008) & 0.131 (0.009) \\
			30  & 6000 & 0.086 (0.008) & 0.086 (0.008)          & \textbf{0.085} (0.008) & 0.089 (0.008) & 0.097 (0.007) \\
			50  & 300  & 0.206 (0.021) & \textbf{0.205} (0.021) & 0.206 (0.019)          & 0.214 (0.020) & 0.256 (0.030) \\
			50  & 1500 & 0.139 (0.012) & \textbf{0.139} (0.012) & 0.139 (0.011)          & 0.145 (0.011) & 0.153 (0.011) \\
			50  & 6000 & 0.108 (0.010) & 0.108 (0.009)          & \textbf{0.107} (0.009) & 0.111 (0.009) & 0.120 (0.008) \\
			\hline
		\end{tabular}
		\label{tab:rank_l}
	\end{center}
\end{table}

\begin{table}[t]
	\begin{center}
		\caption{The average and standard deviation of RMSE using random forests with different numbers of trees $B$ when the training data is generated by the rank-based model.}
		\begin{tabular}{rrccccc}
			\hline
			$N$ & $T$  & $B = 10$        & $B = 100$       & $\bm {B = 1000}$      & $B = 5000$      & $B = 10000$     \\
			\hline
			10  & 300  & 0.096 (0.014) & 0.089 (0.012) & 0.089 (0.012) & 0.089 (0.012) & 0.089 (0.012) \\
			10  & 1500 & 0.066 (0.006) & 0.058 (0.005) & 0.057 (0.005) & 0.057 (0.005) & 0.057 (0.005) \\
			10  & 6000 & 0.050 (0.004) & 0.042 (0.004) & 0.041 (0.004) & 0.041 (0.004) & 0.041 (0.004) \\
			30  & 300  & 0.178 (0.018) & 0.168 (0.017) & 0.168 (0.017) & 0.167 (0.017) & 0.167 (0.017) \\
			30  & 1500 & 0.128 (0.010) & 0.115 (0.008) & 0.114 (0.008) & 0.114 (0.008) & 0.114 (0.008) \\
			30  & 6000 & 0.099 (0.009) & 0.086 (0.008) & 0.085 (0.008) & 0.084 (0.008) & 0.084 (0.008) \\
			50  & 300  & 0.219 (0.022) & 0.207 (0.020) & 0.206 (0.020) & 0.206 (0.020) & 0.206 (0.020) \\
			50  & 1500 & 0.154 (0.014) & 0.141 (0.011) & 0.139 (0.011) & 0.139 (0.011) & 0.139 (0.011) \\
			50  & 6000 & 0.122 (0.011) & 0.109 (0.009) & 0.107 (0.009) & 0.107 (0.009) & 0.107 (0.009) \\
			\hline
		\end{tabular}
		\label{tab:rank_B}
	\end{center}
\end{table}

\begin{table}[t]
	\begin{center}
		\caption{The running time (in seconds) of random forests with different numbers of trees $B$ when the training data is generated by the rank-based model.}
		\begin{tabular}{rrccccc}
			\hline
			$N$ & $T$  & $B = 10$        & $B = 100$       & $\bm {B = 1000}$      & $B = 5000$      & $B = 10000$     \\
			\hline
			10  & 300  & 0.03     & 0.16      & 1.48       & 7.44       & 14.86       \\
			10  & 1500 & 0.03     & 0.19      & 1.76       & 8.71       & 17.47       \\
			10  & 6000 & 0.04     & 0.30      & 2.82       & 13.92      & 27.13       \\
			30  & 300  & 0.16     & 0.42      & 3.07       & 14.04      & 27.52       \\
			30  & 1500 & 0.18     & 0.50      & 3.85       & 17.24      & 32.92       \\
			30  & 6000 & 0.18     & 0.63      & 5.76       & 24.96      & 49.86       \\
			50  & 300  & 0.18     & 0.48      & 3.45       & 16.66      & 32.97       \\
			50  & 1500 & 0.19     & 0.56      & 4.69       & 20.52      & 40.89       \\
			50  & 6000 & 0.19     & 0.79      & 7.30       & 32.97      & 66.04        \\
			\hline
		\end{tabular}
		\label{tab:rank_B_time}
	\end{center}
\end{table}

\subsection{The MNL Model}\label{sec:hyper-para-mnl}
In this section, we generate the data using the MNL model, where the expected utility of each product and the outside option is drawn uniformly randomly from $[0,1]$.
Other settings are the same as Section~\ref{sec:hyper-para-rank}.
The results are shown in Tables \ref{tab:MNL_z} to \ref{tab:MNL_B_time}.
We can see that except for $m$, our default choice is among the best in all the cases.
For $m$, the optimal choice seems to be less than $\sqrt{N}$.
Even in this case, using $m=\sqrt{N}$ is within one standard deviation away from the best choice.

\begin{table}[t]
	\begin{center}
		\caption{The average and standard deviation of RMSE using random forests with different sub-sample size $z$ when the training data is generated by the MNL model.}
		\begin{tabular}{rrccccc}
			\hline
			$N$ & $T$  & $z = T/20$      & $z = T/10$      & $z=T/4$                  & $z=T/2$                  & $\bm {z=T}$           \\
			\hline
			10  & 300  & 0.076 (0.014) & 0.076 (0.014) & 0.067 (0.013)          & \textbf{0.065} (0.011) & 0.066 (0.010)          \\
			10  & 1500 & 0.058 (0.009) & 0.050 (0.007) & 0.045 (0.006)          & \textbf{0.045} (0.005) & 0.047 (0.004)          \\
			10  & 6000 & 0.039 (0.005) & 0.035 (0.004) & \textbf{0.033} (0.002) & 0.035 (0.002)          & 0.039 (0.001)          \\
			30  & 300  & 0.196 (0.017) & 0.196 (0.017) & 0.164 (0.015)          & 0.148 (0.013)          & \textbf{0.145} (0.013) \\
			30  & 1500 & 0.140 (0.011) & 0.114 (0.010) & 0.102 (0.008)          & 0.098 (0.006)          & \textbf{0.097} (0.006) \\
			30  & 6000 & 0.095 (0.007) & 0.084 (0.006) & 0.075 (0.004)          & \textbf{0.073} (0.003) & 0.074 (0.003)          \\
			50  & 300  & 0.244 (0.015) & 0.244 (0.014) & 0.204 (0.013)          & 0.183 (0.012)          & \textbf{0.179} (0.012) \\
			50  & 1500 & 0.170 (0.010) & 0.133 (0.008) & \textbf{0.121} (0.007) & 0.122 (0.006)          & 0.124 (0.006)          \\
			50  & 6000 & 0.105 (0.008) & 0.100 (0.007) & 0.092 (0.005)          & \textbf{0.092} (0.004) & 0.095 (0.003)          \\
			\hline
		\end{tabular}
		\label{tab:MNL_z}
	\end{center}
\end{table}

\begin{table}[t]
	\begin{center}
		\caption{The average and standard deviation of RMSE using random forests with different split products $m$ when the training data is generated by the MNL model.}
		\begin{tabular}{rrccccc}
			\hline
			$N$ & $T$  & $m = 1$                  & $m = \log N$              & $\bm {m = \sqrt{N}}$ & $m = N/2$       & $m = N$         \\
			\hline
			10  & 300  & \textbf{0.061} (0.011) & 0.063 (0.010)          & 0.066 (0.010)         & 0.071 (0.011) & 0.079 (0.013) \\
			10  & 1500 & \textbf{0.042} (0.005) & 0.043 (0.004)          & 0.047 (0.004)         & 0.055 (0.005) & 0.064 (0.006) \\
			10  & 6000 & \textbf{0.034} (0.002) & 0.035 (0.002)          & 0.039 (0.001)         & 0.047 (0.002) & 0.055 (0.003) \\
			30  & 300  & \textbf{0.138} (0.015) & 0.141 (0.013)          & 0.145 (0.013)         & 0.156 (0.014) & 0.162 (0.016) \\
			30  & 1500 & 0.092 (0.010)          & \textbf{0.092} (0.007) & 0.097 (0.006)         & 0.111 (0.007) & 0.117 (0.008) \\
			30  & 6000 & 0.068 (0.007)          & \textbf{0.067} (0.004) & 0.074 (0.003)         & 0.089 (0.005) & 0.093 (0.006) \\
			50  & 300  & \textbf{0.171} (0.014) & 0.174 (0.012)          & 0.179 (0.012)         & 0.189 (0.012) & 0.194 (0.013) \\
			50  & 1500 & 0.119 (0.011)          & \textbf{0.117} (0.008) & 0.124 (0.006)         & 0.138 (0.008) & 0.141 (0.009) \\
			50  & 6000 & 0.088 (0.009)          & \textbf{0.085} (0.006) & 0.095 (0.003)         & 0.110 (0.006) & 0.113 (0.006) \\
			\hline
		\end{tabular}
		\label{tab:MNL_m}
	\end{center}
\end{table}

\begin{table}[t]
	\begin{center}
		\caption{The average and standard deviation of RMSE using random forests with different leaf sizes $l$ when the training data is generated by the MNL model.}
		\label{tab:MNL_l}
		\begin{tabular}{rrccccc}
			\hline
			$N$ & $T$  & $l = 1$         & $l = 10$                 & $\bm {l = 50}$                 & $l = 100$                & $l = 200$       \\
			\hline
			10  & 300  & 0.079 (0.007) & 0.078 (0.007)          & \textbf{0.066} (0.010) & 0.067 (0.012)          & 0.076 (0.015) \\
			10  & 1500 & 0.072 (0.003) & 0.070 (0.003)          & 0.047 (0.004)          & \textbf{0.047} (0.005) & 0.049 (0.006) \\
			10  & 6000 & 0.082 (0.002) & 0.080 (0.002)          & 0.039 (0.001)          & \textbf{0.035} (0.002) & 0.036 (0.002) \\
			30  & 300  & 0.141 (0.013) & \textbf{0.141} (0.013) & 0.145 (0.013)          & 0.156 (0.013)          & 0.196 (0.016) \\
			30  & 1500 & 0.096 (0.005) & \textbf{0.096} (0.005) & 0.097 (0.006)          & 0.106 (0.006)          & 0.117 (0.007) \\
			30  & 6000 & 0.076 (0.002) & 0.075 (0.002)          & \textbf{0.074} (0.003) & 0.079 (0.003)          & 0.089 (0.003) \\
			50  & 300  & 0.175 (0.013) & \textbf{0.175} (0.012) & 0.179 (0.012)          & 0.192 (0.012)          & 0.243 (0.015) \\
			50  & 1500 & 0.122 (0.006) & \textbf{0.122} (0.006) & 0.124 (0.006)          & 0.131 (0.006)          & 0.142 (0.007) \\
			50  & 6000 & 0.095 (0.003) & 0.095 (0.003)          & \textbf{0.095} (0.003) & 0.100 (0.003)          & 0.112 (0.004) \\
			\hline
		\end{tabular}
	\end{center}
\end{table}

\begin{table}[t]
	\begin{center}
		\caption{The average and standard deviation of RMSE using random forests with different numbers of trees $B$ when the training data is generated by the MNL model.}
		\begin{tabular}{rrccccc}
			\hline
			$N$ & $T$  & $B = 10$        & $B = 100$       & $\bm {B = 1000}$      & $B = 5000$      & $B = 10000$     \\
			\hline
			10  & 300  & 0.075 (0.010) & 0.067 (0.011) & 0.066 (0.010) & 0.066 (0.010) & 0.066 (0.010) \\
			10  & 1500 & 0.057 (0.004) & 0.048 (0.004) & 0.047 (0.004) & 0.047 (0.004) & 0.047 (0.004) \\
			10  & 6000 & 0.048 (0.002) & 0.040 (0.001) & 0.039 (0.001) & 0.039 (0.002) & 0.039 (0.002) \\
			30  & 300  & 0.157 (0.012) & 0.146 (0.013) & 0.145 (0.013) & 0.145 (0.013) & 0.145 (0.013) \\
			30  & 1500 & 0.112 (0.005) & 0.099 (0.005) & 0.097 (0.006) & 0.097 (0.006) & 0.097 (0.006) \\
			30  & 6000 & 0.090 (0.003) & 0.075 (0.003) & 0.074 (0.003) & 0.073 (0.003) & 0.073 (0.003) \\
			50  & 300  & 0.191 (0.013) & 0.180 (0.012) & 0.179 (0.012) & 0.179 (0.011) & 0.179 (0.011) \\
			50  & 1500 & 0.139 (0.006) & 0.125 (0.006) & 0.124 (0.006) & 0.124 (0.006) & 0.124 (0.006) \\
			50  & 6000 & 0.111 (0.003) & 0.096 (0.003) & 0.095 (0.003) & 0.094 (0.003) & 0.094 (0.003) \\
			\hline
		\end{tabular}
		\label{tab:MNL_B}
	\end{center}
\end{table}

\begin{table}[t]
	\begin{center}
		\caption{The running time (in seconds) of random forests with different numbers of trees $B$ when the training data is generated by the MNL model.}
		\begin{tabular}{rrccccc}
			\hline
			$N$ & $T$  & $B = 10$        & $B = 100$       & $\bm {B = 1000}$      & $B = 5000$      & $B = 10000$     \\
			\hline
			10  & 300  & 0.03     & 0.16      & 1.32       & 7.30       & 14.62       \\
			10  & 1500 & 0.03     & 0.20      & 1.61       & 8.72       & 17.26       \\
			10  & 6000 & 0.04     & 0.30      & 2.58       & 14.15      & 28.20       \\
			30  & 300  & 0.17     & 0.45      & 3.00       & 15.62      & 31.17       \\
			30  & 1500 & 0.17     & 0.55      & 3.58       & 18.76      & 36.26       \\
			30  & 6000 & 0.19     & 0.67      & 5.54       & 26.79      & 50.99       \\
			50  & 300  & 0.18     & 0.51      & 3.47       & 18.17      & 36.52       \\
			50  & 1500 & 0.19     & 0.59      & 4.22       & 22.53      & 44.48       \\
			50  & 6000 & 0.20     & 0.82      & 6.79       & 33.27      & 66.06      \\
			\hline
		\end{tabular}
		\label{tab:MNL_B_time}
	\end{center}
\end{table}

\section{Expected Distance in Section~\ref{sec:distance}} \label{sec:distance_numerical}
In this section, we show that a polynomial number of assortments cannot guarantee the expected distance to be within $O(\log N)$ by numerical studies.
The result complements Proposition~\ref{thm:dis_upper_bound}.
Let $M$ be the number of assortments randomly drawn with replacement in the training data.
We sample 100,000 instances for a combination of  $N$ and $M$ and take the average number of zeros for the largest binary numbers. The results are shown in Figure~\ref{fig: ave_dis}.
\begin{figure}[t]
	\begin{center}
		\caption{Average distances of different $N$ and $M$.} \label{fig: ave_dis}
		\includegraphics[width=0.8\textwidth]{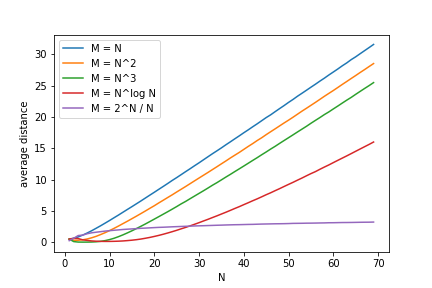}
	\end{center}
\end{figure}
We can observe that even for $M = O(N^{\log N})$, the average distance is still growing linearly.

\section{Numerical Experiments for Rank-based DCM}\label{sec:additional-numerical}
In this section, we complement the results in Theorem \ref{thm:single tree} by additional numerical studies.
Theorem~\ref{thm:single tree} states that the random forest algorithm can recover the rank-based DCM with a single ranking
when the training data is sampled uniformly.
We will demonstrate numerically that the result still holds when the training data is not uniform.
We also provide some examples showing that even though random forests may fail to recover the rankings exactly,
the predicted probability is still quite accurate.
%Essentially, the random forest includes trees that are
%and we are able to obtain the correct tree by reconstruction.
Moreover, we demonstrate the insights that when the rank-based DCM consists of more than one ranks (customer types),
the random forest may output a tree that ``concatenates'' multiple ranks.

\subsection{Non-uniform Training Data}
Consider $N=3$ products for the illustration purpose.
Suppose customers prefer $1 \succ 2 \succ 3 \succ 0$, where 0 denotes the outside option.
Figure \ref{fig:tree_RL} illustrates the decision tree corresponding to the preference of the customers,
where the splits are in the order of 1, 2, 3.
We evaluate the performance of random forests by the number of correct splits until the first mistake is made.
Notice that if the algorithm returns a tree that correctly splits the first $i$ cuts, then it can accurately predict a fraction of $1-1/2^i$ of total assortments.
For example, if the splits are 1, 3, 2 in order, then the algorithm makes one correct split and accurately predicts a half of total assortments.

Now consider $N=10$ and $2^N=1024$ possible assortments (including the empty set).
The parameter of random forests are $B=10, z = T, m = N, l = 1$.
We compare the performance of random forests in the following three sampling schemes:
\begin{enumerate}
	\item Uniform: Each assortment is observed with equal probability $1/2^N$;
	\item Non-uniform: The probability of observing each assortment follows a Dirichlet distribution with concentration parameters all equal to ones;
	\item Different occurrences: generate probability $p_i$ for each product, where $p_i$ independently follows a normal distribution with mean 0.5 and standard deviation 0.15 so that $p_i$ falls into $[0,1]$ with high probability. We set $p_i=0$ if the random variable is negative and $p_i = 1$ if it exceeds 1.
	In each assortment in the data, each product is included with probability $p_i$.
\end{enumerate}
%Notice that in the non-uniform case, although the observation probabilities of assortments are different, the total occurrences of each product may still be very close, at about 0.5.
%However, in the "different occurrences" case, the occurrences of different products may vary.
We test different sizes of training data $T \in [100, 1000, 10000]$.
For each setup, we generate 100 independent datasets and inspect $B=10$ trees for each dataset.
The mean and standard deviation of the correct splits are reported in Table \ref{tab:correct_cut}.
\begin{table}[t]
	\begin{center}
		\caption{The number of correct splits by random forests.}
		\begin{tabular}{rccc}
			\hline
			$T$   & Uniform       & Non-uniform   & Different occurrences \\
			\hline
			100   & 4.281 (0.951) & 4.185 (0.922) & 3.575 (1.372)       \\
			1000  & 7.715 (1.058) & 7.370 (1.110) & 6.369 (1.732)       \\
			10000 & 9.961 (0.256) & 9.131 (0.986) & 8.610 (2.145)      \\
			\hline
		\end{tabular}
		\label{tab:correct_cut}
	\end{center}
\end{table}
In general the uniform sampling has the best performance, but the algorithm is quite robust to other non-uniform sampling schemes.

\subsection{Reconstruction of Decision Trees}
From the numerical studies, we observe that some of the trees learned from the random forest may not have the some structure as the actual ranking.
They nevertheless represent the same choice probabilities.
To illustrate, consider $N=3$ and let the consumers' ranking be represented by Figure~\ref{fig:tree_RL}(a).
%Figure \ref{fig:tree_RL} shows the tree structure exactly cover the single rank list model.
The random forest may learn a tree as shown in Figure~
\ref{fig:tree_RL}(b) or \ref{fig:tree_RL}(c),
 %\ref{fig:tree_reconstruct} or \ref{fig:tree_artificial_cut}, 
which represents the same DCM but with a different structure.
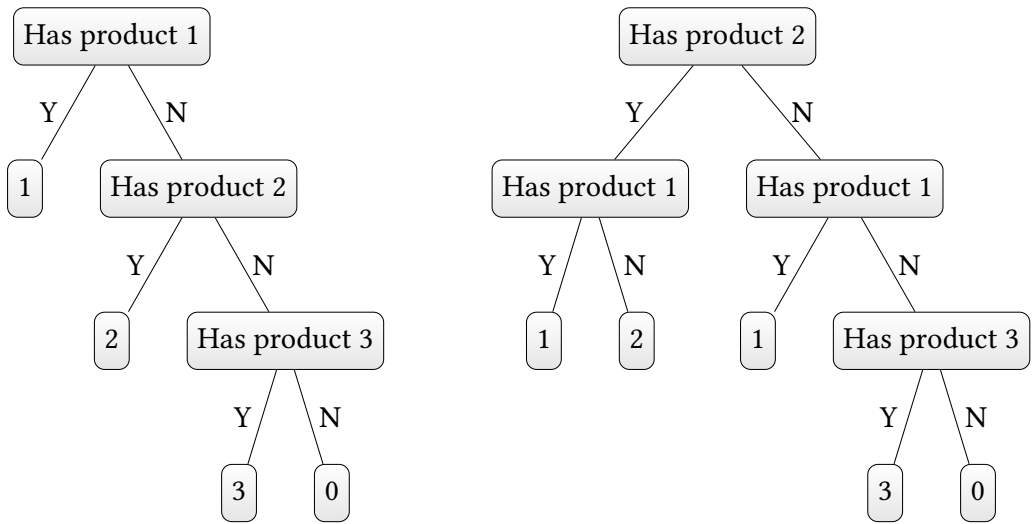
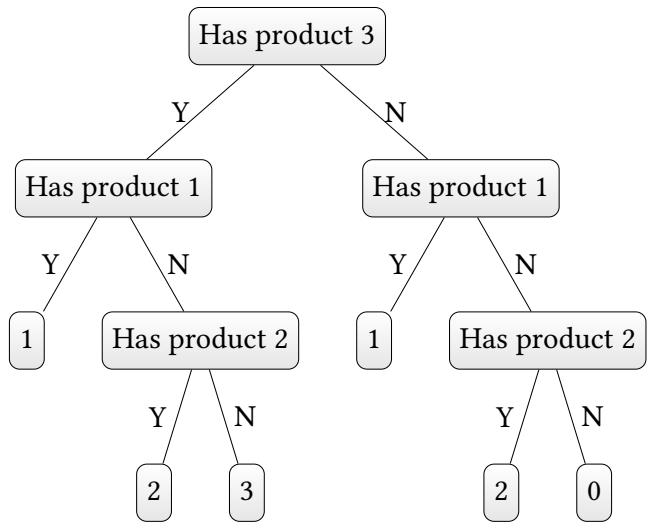
\begin{figure}[t]
	\begin{center}
		\caption{Three possible trees representing the same choice.}
		\scalebox{1}{\label{fig:tree_RL} 
			\subfigure[The tree representing the preference list of consumers]{
				\begin{forest}
					for tree={l sep+=.8cm,s sep+=.5cm,shape=rectangle, rounded corners,
						draw, align=center,
						top color=white, bottom color=gray!20}
					[Has product 1
					[1,edge label={node[midway,left]{Y}}
					]
					[Has product 2,edge label={node[midway,right]{N}}
					[2, edge label={node[midway,left]{Y}}]
					[Has product 3, edge label={node[midway,right]{N}}
					[3, edge label={node[midway,left]{Y}}]
					[0, edge label={node[midway,right]{N}}]
					]
					]
					]
		\end{forest}}}
		\hspace{1cm}
		\scalebox{1}{
			\subfigure[A reconstructed tree having the same choice]{\label{fig:tree_reconstruct}
				\begin{forest}
					for tree={l sep+=.8cm,s sep+=.5cm,shape=rectangle, rounded corners,
						draw, align=center,
						top color=white, bottom color=gray!20}
					[Has product 2
					[Has product 1,edge label={node[midway,left]{Y}}
					[1, edge label={node[midway,left]{Y}}]
					[2, edge label={node[midway,right]{N}}]
					]
					[Has product 1,edge label={node[midway,right]{N}}
					[1, edge label={node[midway,left]{Y}}]
					[Has product 3, edge label={node[midway,right]{N}}
					[3, edge label={node[midway,left]{Y}}]
					[0, edge label={node[midway,right]{N}}]
					]
					]
					]
		\end{forest}}}
		\scalebox{1}{
			\subfigure[A reconstructed tree having the same choice]{\label{fig:tree_artificial_cut}
				\begin{forest}
					for tree={l sep+=.8cm,s sep+=.5cm,shape=rectangle, rounded corners,
						draw, align=center,
						top color=white, bottom color=gray!20}
					[Has product 3
					[Has product 1,edge label={node[midway,left]{Y}}
					[1, edge label={node[midway,left]{Y}}]
					[Has product 2, edge label={node[midway,right]{N}}
					[2, edge label={node[midway,left]{Y}}]
					[3, edge label={node[midway,right]{N}}]
					]
					]
					[Has product 1,edge label={node[midway,right]{N}}
					[1, edge label={node[midway,left]{Y}}]
					[Has product 2, edge label={node[midway,right]{N}}
					[2, edge label={node[midway,left]{Y}}]
					[0, edge label={node[midway,right]{N}}]
					]
					]
					]
		\end{forest}}}
		\label{fig:tree_structure}
	\end{center}
\end{figure}

\subsection{Multiple Rankings}\label{sec:multi-ranking}
{Next we study the case when customers are represented by a mixture of multiple rankings.
	More precisely, consider $N=5$ and two customer segments.
	Suppose 70\% of customers are type one and have preference $1 \succ 2 \succ 3 \succ 4 \succ 5 \succ 0$; the remaining 30\% of customers are type two preferring $5 \succ 4 \succ 3 \succ 2 \succ 1 \succ 0$.
	When $T = 1000, B=10, z = T, m = N, l = 1$, a number of trees learned by Algorithm~\ref{alg:rf} resemble the structure shown in Figure \ref{fig:mixed_RL}.
	There may be more than one chosen product in the leaf node explained in equation \eqref{eq:class-prob}.
	For example, ``1(5)'' implies that most customers select product one in this leaf node, but some others purchase product 5.
	We can see that the main branch on the right resembles the preference list of type one.
	Meanwhile, some branches of the ranking of type two customers are attached to it.
	Random forests somehow merge the preferences of various segments into a single tree.
	This is a phenomenon commonly observed in our experiments and may shed light on the robust performance of random forests.}
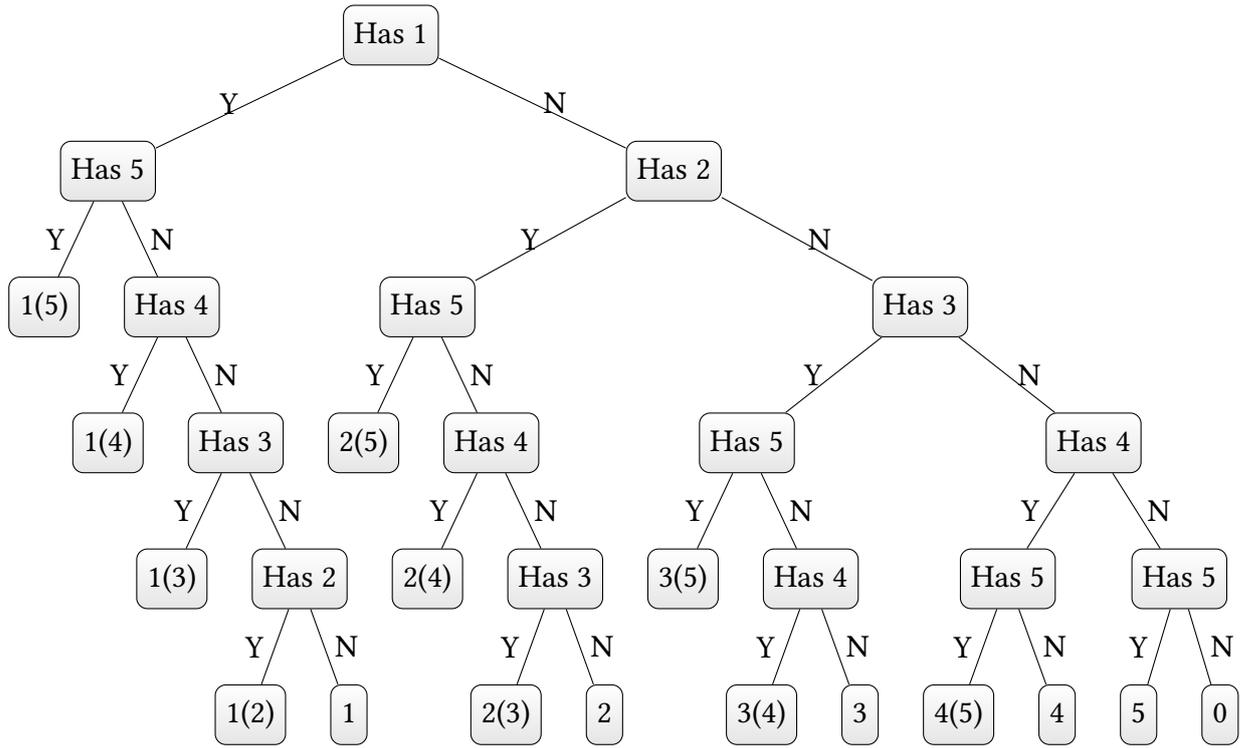
\begin{figure}[t]
	\begin{center}
		\caption{A tree output by random forests that merges two rankings.}
		\begin{forest}
			for tree={l sep+=.5cm,s sep+=.3cm,shape=rectangle, rounded corners,
				draw, align=center,
				top color=white, bottom color=gray!20}
			[Has 1
			[Has 5,edge label={node[midway,left]{Y}}
			[1(5), edge label={node[midway,left]{Y}}]
			[Has 4, edge label={node[midway,right]{N}}
			[1(4), edge label={node[midway,left]{Y}}]
			[Has 3, edge label={node[midway,right]{N}}
			[1(3), edge label={node[midway,left]{Y}}]
			[Has 2, edge label={node[midway,right]{N}}
			[1(2), edge label={node[midway,left]{Y}}]
			[1, edge label={node[midway,right]{N}}]
			]
			]
			]
			]
			[Has 2,edge label={node[midway,right]{N}}
			[Has 5, edge label={node[midway,left]{Y}}
			[2(5), edge label={node[midway,left]{Y}}]
			[Has 4, edge label={node[midway,right]{N}}
			[2(4), edge label={node[midway,left]{Y}}]
			[Has 3, edge label={node[midway,right]{N}}
			[2(3), edge label={node[midway,left]{Y}}]
			[2, edge label={node[midway,right]{N}}]
			]
			]
			]
			[Has 3, edge label={node[midway,right]{N}}
			[Has 5, edge label={node[midway,left]{Y}}
			[3(5), edge label={node[midway,left]{Y}}]
			[Has 4, edge label={node[midway,right]{N}}
			[3(4), edge label={node[midway,left]{Y}}]
			[3, edge label={node[midway,right]{N}}]
			]
			]
			[Has 4, edge label={node[midway,right]{N}}
			[Has 5, edge label={node[midway,left]{Y}}
			[4(5), edge label={node[midway,left]{Y}}]
			[4, edge label={node[midway,right]{N}}]
			]
			[Has 5, edge label={node[midway,right]{N}}
			[5, edge label={node[midway,left]{Y}}]
			[0, edge label={node[midway,right]{N}}]
			]
			]
			]
			]
			]
		\end{forest}
		\label{fig:mixed_RL}
	\end{center}
\end{figure}

\section{Numerical Examples for Product Importance}\label{sec:prod-imp-exp}
\subsection{Synthetic Examples}
Here we give an example of MDI (see Section~\ref{sec:importance} for details) under the rank-based DCM (with one ranking) and the MNL model.
\begin{example}
	\label{ex:MDI}
	Consider 10 products and data size $T = 10000$.
	The assortment in training data is uniformly generated among non-empty subsets.
	We use $B = 1000, m = \sqrt{N}, l = 50$ for the random forest algorithm.
	The ground-truth model is a single ranking: $1 \succ 2 \succ \ldots \succ 10$.
	The MDI is shown in Figure~\ref{fig:MDI_RankList}.
	We can observe that the MDI is decreasing in product indices.
	\begin{figure}[t]
		\begin{center}
			\caption{MDI of a single ranking when $N = 10, T = 10000$.}
			\label{fig:MDI_RankList}
			\includegraphics[width=0.8\textwidth]{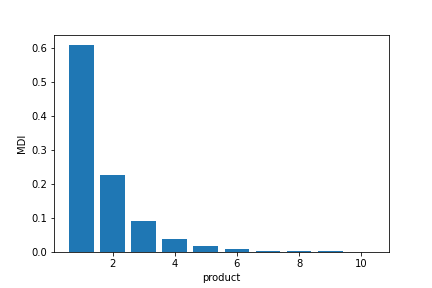}
		\end{center}
	\end{figure}
	
	When the ground-truth model is MNL, we generate utilities $u_i$ from a uniform distribution on $[0,1]$ for each product.
	The attraction of product $i$ is defined as $a_i = \exp(u_i)$ and the choice probability is $p(i, S) = a_i / \sum_{j \in S} a_j$.
	We show the attractions and MDI of random forests in Figure~\ref{fig:MDI_MNL}.
	It is clear that the MDI is highly correlated with attractions.
	\begin{figure}[t]
		\begin{center}
			\caption{MDI of the MNL model when $N = 10, T = 10000$.}
			\label{fig:MDI_MNL}
			\includegraphics[width=1\textwidth]{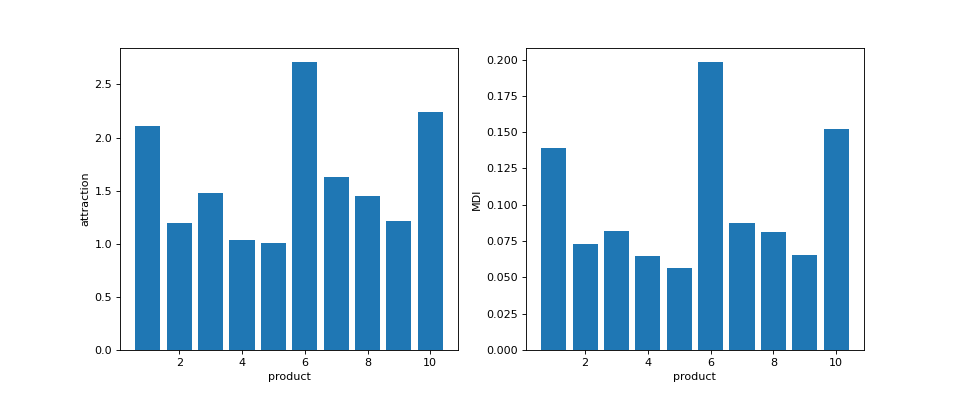}
		\end{center}
	\end{figure}
\end{example}
\subsection{Hotel and IRI Datasets}
%{\color{blue}
We also calculate the MDI for the hotel and IRI datasets.
Products with the highest MDI are listed in Tables~\ref{tab:MDI_Hotel} and \ref{tab:MDI_IRI}.
In the IRI dataset (Table~\ref{tab:MDI_IRI}), the products are represented by the vendor codes.
The products with the highest MDI do not necessarily have high demand, but play an important role in customers' decision process.
Therefore, we recommend the firms to prioritize the experimentation of these products when optimizing over assortments or prices.%}

\begin{table}[t]
	\begin{center}
		\caption{Four room types with the highest MDI in each hotel.}
		\scalebox{1}{
			\begin{tabular}{clc}
				\hline
				Hotel   & Room type                        & MDI   \\ \hline
				Hotel 1 & King Room 3                      & 0.222 \\
				& Special Type Room 1              & 0.160 \\
				& Queen Room 2                     & 0.100 \\
				& King Room 4                      & 0.096 \\ \hline
				Hotel 2 & 2 Queen Beds Room 2 Smoking      & 0.279 \\
				& 2 Queen Beds Room 1 Non-Smoking  & 0.202 \\
				& 2 Queen Beds Room 2 Non-Smoking  & 0.157 \\
				& King Room 3 Non-Smoking          & 0.148 \\ \hline
				Hotel 3 & King Room 1 Non-Smoking          & 0.523 \\
				& 2 Double Beds Room 1 Non-Smoking & 0.263 \\
				& King Room 3 Non-Smoking          & 0.081 \\
				& 2 Double Beds Room 1 Smoking     & 0.049 \\ \hline
				Hotel 4 & King Room 1 Non-Smoking          & 0.641 \\
				& 2 Queen Beds Room 1 Non-Smoking  & 0.179 \\
				& 2 Queen Beds Room 1 Smoking      & 0.100 \\
				& Suite 2 Non-Smoking              & 0.080 \\ \hline
				Hotel 5 & King Room 2 Non-Smoking          & 0.496 \\
				& King Room 1 Non-Smoking          & 0.228 \\
				& 2 Double Beds Room 1 Non-Smoking & 0.085 \\
				& 2 Double Beds Room 2 Non-Smoking & 0.077 \\ \hline
		\end{tabular}}
		\label{tab:MDI_Hotel}
	\end{center}
\end{table}

\begin{table}[t]
	\begin{center}
		\caption{Three product vendors with the highest MDI in each category of the IRI dataset.}
		\scalebox{1}{
			\begin{tabular}{lcccccc}
				\hline
				Product Category & No.1 prod & MDI & No.2 prod & MDI & No.3 prod & MDI \\
				\hline
				Beer & 87692 & 0.375 & 83820 & 0.364 & 78250 & 0.148\\
				Blades & 99998 & 0.302 & 52754 & 0.149 & 41058 & 0.140\\
				Carbonated Beverages & 73800 & 0.276 & 42200 & 0.178 & 71698 & 0.171 \\
				Cigarettes & 710 & 0.322 & 90500 & 0.248 & 99998 & 0.170\\
				Coffee & 75101 & 0.210 & 11141 & 0.200 & 71038 & 0.162\\
				Cold Cereal & 42400 & 0.358 & 18627 & 0.175 & 21908 & 0.148 \\
				Deodorant & 19045 & 0.315 & 9973 & 0.297 & 22600 & 0.150 \\
				Diapers & 99998 & 0.424 & 48157 & 0.216 & 32913 & 0.206 \\
				Facial Tissue & 63435 & 0.357 & 43032 & 0.296 & 99998 & 0.099 \\
				Frozen Dinners/Entrees & 50100 & 0.334 & 17854 & 0.251 & 72655 & 0.234\\
				Frozen Pizza & 74653 & 0.234 & 19600 & 0.202 & 35300 & 0.180\\
				Household Cleaners & 35000 & 0.276 & 23400 & 0.260 & 25700 & 0.190\\
				Hotdogs & 75278 & 0.178 & 85331 & 0.168 & 46600 & 0.151\\
				Laundry Detergent & 45893 & 0.447 & 35000 & 0.186 & 72613 & 0.138\\
				Margarine/Butter & 96451 & 0.479 & 33100 & 0.202 & 34500 & 0.115 \\
				Mayonnaise & 45200 & 0.248 & 52100 & 0.224 & 52500 & 0.172\\
				Milk & 99998 & 0.280 & 75457 & 0.269 & 41483 & 0.241\\
				Mustard & 71828 & 0.250 & 24000 & 0.235 & 70080 & 0.190\\
				Paper Towels & 43032 & 0.293 & 44096 & 0.216 & 30400 & 0.164\\
				Peanut Butter & 71018 & 0.321 & 45300 & 0.202 & 34000 & 0.117\\
				Photo & 74101 & 0.322 & 99998 & 0.280 & 41778 & 0.143\\
				Razors & 41058 & 0.350 & 99998 & 0.325 & 47400 & 0.197 \\
				Salt Snacks & 41780 & 0.401 & 41262 & 0.284 & 72600 & 0.157 \\
				Shampoo & 65632 & 0.350 & 71249 & 0.226 & 99998 & 0.142\\
				Soup & 41789 & 0.351 & 24000 & 0.332 & 50100 & 0.164 \\
				Spaghetti/Italian Sauce & 72940 & 0.313 & 77644 & 0.254 & 6010 & 0.161 \\
				Sugar Substitutes & 99998 & 0.191 & 19098 & 0.174 & 58312 & 0.137\\
				Toilet Tissue & 43032 & 0.416 & 44096 & 0.233 & 30400 & 0.205\\
				Toothbrushes & 70942 & 0.252 & 69055 & 0.216 & 416 & 0.140 \\
				Toothpaste & 68305 & 0.376 & 77326 & 0.248 & 10310 & 0.164 \\
				Yogurt & 21000 & 0.389 & 41148 & 0.171 & 53600 & 0.151\\
				\hline
		\end{tabular}}
		\label{tab:MDI_IRI}
	\end{center}
\end{table}

\section{Numerical Experiments on Aggregated Choice Data} \label{sec:aggregated_numerical}
In this section, we investigate the performance of random forests when the training data is aggregated as in Appendix~\ref{sec:aggregate-choice}.
To generate the aggregated training data, we first generate $T$ observations using the MNL model for $N=10$ products.
The utility of each product and the outside option is generated uniformly between 0 and 1.
%The only difference is that we only simulate one instead of ten transactions for each offered assortment.
Then, we let $a$ be the aggregation level, i.e., aggregate $a$ data points together.
For example, $a = 1$ is equivalent to the original unaggregated data.
For $a=5$, Table~\ref{tab:aggregate example} illustrates five observations in the original dataset for $N=5$.
Upon aggregation, the five transactions are replaced by five new observations with $\bm x_t \equiv [0.6,0.4,0.8,0.4,0.6]$ and $i_t=1,0,4,3,1$ for $t=1,2,3,4,5$.
\begin{table}[h]
	\begin{center}
		\caption{Five observations in the unaggregated original data. Upon aggregation, they are replaced by five new observations with $\bm x_t \equiv [0.6,0.4,0.8,0.4,0.6]$ and $i_t=1,0,4,3,1$ for $t=1,2,3,4,5$.}
		\def\arraystretch{0.9}\begin{tabular}{cccccc}
			\hline
			Product 1 & Product 2 & Product 3 & Product 4 & Product 5& Choices\\
			\hline
			1 & 1 & 1 & 1 & 1 & 1 \\
			0 & 1 & 0 & 0 & 1 & 0 \\
			1 & 0 & 1 & 1 & 1 & 4 \\
			0 & 0 & 1 & 0 & 0 & 3 \\
			1 & 0 & 1 & 0 & 0 & 1 \\
			\hline
		\end{tabular}
		\label{tab:aggregate example}
	\end{center}
	\vspace{-4mm}
\end{table}

We test the performance of different aggregate levels $a \in \{1,5,10,20,50\}$ when $T = 5000$.
The performance is measured in RMSE \eqref{RMSE_soft}.
Note that other DCMs cannot handle the situation naturally.
To apply the benchmarks, we ``de-aggregate'' the data by randomly generating $a$ assortments, each including product $j$ with probability 1-$\bm p_s(j) \in [0,1]$ independently.
Then we estimate the parameters for MNL and the Markov chain model from the unaggregated data.
We simulate 100 instances for each setting to evaluate the average and standard deviation, shown in Table~\ref{tab:aggregate result}.

\begin{table}[h]
	\begin{center}
		\caption{The performance of random forests and two other benchmarks for different aggregate levels.}
		\def\arraystretch{0.9}\begin{tabular}{cccc}
			\hline
			Aggregate levels & RF & MNL & MC \\
			\hline
			$a = 1$  & 0.038 (0.002) & 0.007 (0.002) & 0.020 (0.001) \\
			$a = 5$  & 0.045 (0.005) & 0.048 (0.007) & 0.052 (0.006)\\
			$a =10$ & 0.051 (0.010) & 0.057 (0.008) & 0.060 (0.007)\\
			$a =20$ & 0.054 (0.009) & 0.062 (0.008) & 0.065 (0.007)\\
			$a =50$ & 0.060 (0.010) & 0.065 (0.009) & 0.067 (0.008)\\
			\hline
		\end{tabular}
		\label{tab:aggregate result}
	\end{center}
	\vspace{-4mm}
\end{table}
From the results, MNL and the Markov chain model perform well for the original data ($a=1$).
However, after aggregation, random forests outperform the other two even when the underlying model is indeed MNL.
It showcases the strength of random forests for this type of data.
%We also show random forests outperform MNL and linear demand model when the price information is incorporated in Online Appendix~\ref{sec:pricing_numerical}.

\section{Additional Results for the IRI Dataset}\label{sec:iri-additional}
In this section, we provide additional results for the IRI dataset when the top seven and fifteen products are considered.
The setup is described in Section~\ref{sec:IRI_numerical}.

\begin{table}[t]
	\begin{center}
		\caption{The summary statistics (the data size, the number of unique assortments in the data, and the average number of products in an assortment) of the IRI dataset after preprocessing and the average and standard deviation of the out-of-sample RMSE \eqref{RMSE_realdata} for each category when considering the top 7 products.}
		\scalebox{1}{
			\begin{tabular}{lrrccccc} % {p{4cm}>{\raggedleft}p{1.5cm}>{\raggedleft}p{1.5cm}>{\raggedleft}p{1.5cm}>{\centering}p{3cm}>{\centering}p{3cm}>{\centering}p{3cm}}
				\hline
				Product category & \#Data  &\#Unique &\#Avg & RF & MNL & MC \\
				& & assort & prod \\
				\hline
				Beer & 1,201 & 40 & 3.65 & \textbf{0.088} (0.019) & \textbf{0.088} (0.020) & 0.090 (0.011) \\
				Blades & 1,441 & 48 & 3.71 & 0.053 (0.010) & 0.055 (0.008) & \textbf{0.047} (0.010) \\
				Carbonated Beverages & 416 & 14 & 4.14 & 0.107 (0.068) & 0.117 (0.050) & \textbf{0.104} (0.051) \\
				Cigarettes & 1,707 & 57 & 3.98 & 0.067 (0.023) & 0.074 (0.027) & \textbf{0.060} (0.024) \\
				Coffee & 934 & 31 & 4.45 & \textbf{0.108} (0.028) & 0.119 (0.024) & 0.111 (0.022) \\
				Cold Cereal & 383 & 13 & 5.38 & \textbf{0.087} (0.035) & 0.111 (0.035) & 0.107 (0.033) \\
				Deodorant & 1,538 & 51 & 4.12 & 0.057 (0.009) & 0.064 (0.019) & \textbf{0.053} (0.003) \\
				Diapers & 658 & 22 & 3.32 & 0.081 (0.040) & 0.079 (0.030) & \textbf{0.066} (0.030) \\
				Facial Tissue & 865 & 29 & 3.41 & \textbf{0.132} (0.017) & 0.146 (0.034) & 0.139 (0.015) \\
				Frozen Dinners/Entrees & 772 & 26 & 4.42 & \textbf{0.076} (0.022) & 0.098 (0.026) & 0.088 (0.024) \\
				Frozen Pizza & 1,504 & 50 & 3.76 & \textbf{0.137} (0.022) & 0.146 (0.012) & 0.146 (0.015) \\
				Household Cleaners & 994 & 33 & 4.73 & 0.104 (0.015) & \textbf{0.096} (0.015) & \textbf{0.096} (0.016) \\
				Hotdogs & 1,888 & 63 & 3.76 & \textbf{0.123} (0.023) & 0.130 (0.023) & \textbf{0.123} (0.027) \\
				Laundry Detergent & 1,614 & 54 & 4.20 & \textbf{0.124} (0.017) & 0.154 (0.017) & 0.137 (0.015) \\
				Margarine/Butter & 630 & 21 & 4.90 & \textbf{0.106} (0.012) & 0.135 (0.026) & 0.119 (0.026) \\
				Mayonnaise & 681 & 23 & 3.78 & \textbf{0.110} (0.075) & 0.151 (0.069) & 0.126 (0.052) \\
				Milk & 1,080 & 36 & 3.83 & \textbf{0.097} (0.026) & 0.111 (0.020) & 0.105 (0.024) \\
				Mustard & 897 & 30 & 4.30 & \textbf{0.089} (0.034) & 0.109 (0.032) & 0.109 (0.031) \\
				Paper Towels & 1,286 & 43 & 4.35 & \textbf{0.093} (0.018) & 0.131 (0.028) & 0.108 (0.034) \\
				Peanut Butter & 931 & 31 & 4.03 & \textbf{0.082} (0.031) & 0.085 (0.027) & 0.088 (0.038) \\
				Photography supplies & 2,974 & 99 & 3.48 & 0.085 (0.005) & 0.096 (0.016) & \textbf{0.083} (0.019) \\
				Razors & 1,555 & 52 & 2.92 & 0.073 (0.028) & 0.067 (0.008) & \textbf{0.044} (0.013) \\
				Salt Snacks & 749 & 25 & 4.52 & 0.101 (0.029) & \textbf{0.092} (0.027) & 0.097 (0.031) \\
				Shampoo & 1,775 & 59 & 3.92 & 0.085 (0.028) & \textbf{0.076} (0.013) & 0.077 (0.019) \\
				Soup & 629 & 21 & 5.14 & \textbf{0.130} (0.025) & 0.135 (0.025) & 0.148 (0.014) \\
				Spaghetti/Italian Sauce & 931 & 31 & 4.61 & \textbf{0.101} (0.048) & 0.104 (0.028) & 0.102 (0.031) \\
				Sugar Substitutes & 1,204 & 40 & 3.90 & 0.063 (0.020) & 0.064 (0.018) & \textbf{0.057} (0.009) \\
				Toilet Tissue & 1,020 & 34 & 4.41 & \textbf{0.124} (0.016) & 0.145 (0.016) & \textbf{0.124} (0.019) \\
				Toothbrushes & 2,622 & 87 & 3.86 & 0.092 (0.014) & \textbf{0.086} (0.009) & \textbf{0.086} (0.012) \\
				Toothpaste & 958 & 32 & 4.28 & 0.093 (0.036) & \textbf{0.087} (0.026) & 0.092 (0.020) \\
				Yogurt & 1,465 & 49 & 3.69 & \textbf{0.121} (0.041) & 0.132 (0.040) & 0.128 (0.042) \\
				\hline
		\end{tabular}}
		\label{tab:IRI KFold top7}
	\end{center}
\end{table}

\begin{table}[t]
	\begin{center}
		\caption{The summary statistics (the data size, the number of unique assortments in the data, and the average number of products in an assortment) of the IRI dataset after preprocessing and the average and standard deviation of the out-of-sample RMSE \eqref{RMSE_realdata} for each category when considering the top 15 products.}
		\scalebox{1}{
			\begin{tabular}{lrrccccc} % {p{4cm}>{\raggedleft}p{1.5cm}>{\raggedleft}p{1.5cm}>{\raggedleft}p{1.5cm}>{\centering}p{3cm}>{\centering}p{3cm}>{\centering}p{3cm}}
				\hline
				Product category & \#Data  &\#Unique &\#Avg & RF & MNL & MC \\
				& & assort & prod \\
				\hline
				Beer & 22,341 & 755 & 9.25 & \textbf{0.045} (0.004) & 0.061 (0.004) & 0.058 (0.004) \\
				Blades & 5,789 & 193 & 5.36 & 0.052 (0.014) & 0.054 (0.011) & \textbf{0.051} (0.011) \\
				Carbonated Beverages & 9,386 & 316 & 7.39 & \textbf{0.050} (0.004) & 0.067 (0.004) & 0.063 (0.003) \\
				Cigarettes & 15,052 & 506 & 6.46 & \textbf{0.049} (0.010) & 0.065 (0.012) & 0.056 (0.012) \\
				Coffee & 26,521 & 894 & 9.05 & \textbf{0.056} (0.003) & 0.075 (0.001) & 0.071 (0.001) \\
				Cold Cereal & 16,966 & 575 & 10.20 & \textbf{0.029} (0.002) & 0.036 (0.004) & 0.033 (0.003) \\
				Deodorant & 34,352 & 1137 & 9.43 & \textbf{0.038} (0.002) & 0.039 (0.002) & \textbf{0.038} (0.002) \\
				Diapers & 1,228 & 41 & 3.85 & \textbf{0.076} (0.014) & 0.080 (0.017) & \textbf{0.076} (0.015) \\
				Facial Tissue & 3,611 & 121 & 4.98 & \textbf{0.075} (0.015) & 0.093 (0.016) & 0.078 (0.012) \\
				Frozen Dinners/Entrees & 13,514 & 453 & 10.17 & \textbf{0.057} (0.005) & 0.078 (0.002) & 0.072 (0.003) \\
				Frozen Pizza & 27,204 & 917 & 7.98 & \textbf{0.066} (0.001) & 0.092 (0.003) & 0.086 (0.003) \\
				Household Cleaners & 21,239 & 704 & 10.58 & \textbf{0.049} (0.002) & 0.056 (0.003) & 0.053 (0.002) \\
				Hotdogs & 26,821 & 897 & 7.23 & \textbf{0.087} (0.005) & 0.116 (0.003) & 0.108 (0.004) \\
				Laundry Detergent & 10,943 & 367 & 7.20 & \textbf{0.081} (0.007) & 0.101 (0.009) & 0.092 (0.007) \\
				Margarine/Butter & 6,727 & 226 & 8.78 & \textbf{0.046} (0.006) & 0.068 (0.007) & 0.061 (0.005) \\
				Mayonnaise & 21,512 & 732 & 7.31 & \textbf{0.051} (0.003) & 0.086 (0.004) & 0.076 (0.004) \\
				Milk & 5,278 & 177 & 6.14 & \textbf{0.088} (0.010) & 0.112 (0.008) & 0.109 (0.009) \\
				Mustard & 42,703 & 1451 & 9.34 & \textbf{0.045} (0.002) & 0.056 (0.003) & 0.052 (0.002) \\
				Paper Towels & 5,963 & 199 & 6.33 & \textbf{0.084} (0.006) & 0.105 (0.005) & 0.094 (0.008) \\
				Peanut Butter & 14,484 & 488 & 6.84 & \textbf{0.055} (0.009) & 0.087 (0.007) & 0.082 (0.007) \\
				Photography supplies & 6,944 & 231 & 3.93 & 0.085 (0.013) & 0.096 (0.010) & \textbf{0.082} (0.010) \\
				Razors & 2,600 & 87 & 3.29 & 0.084 (0.014) & 0.071 (0.008) & \textbf{0.054} (0.003) \\
				Salt Snacks & 7,505 & 252 & 7.46 & \textbf{0.048} (0.006) & 0.061 (0.006) & 0.058 (0.007) \\
				Shampoo & 39,352 & 1305 & 9.15 & 0.053 (0.001) & 0.053 (0.002) & \textbf{0.052} (0.002) \\
				Soup & 19,661 & 666 & 10.02 & \textbf{0.051} (0.002) & 0.079 (0.004) & 0.075 (0.004) \\
				Spaghetti/Italian Sauce & 17,109 & 576 & 8.92 & \textbf{0.065} (0.007) & 0.076 (0.005) & 0.073 (0.006) \\
				Sugar Substitutes & 12,763 & 426 & 6.19 & \textbf{0.043} (0.005) & 0.048 (0.003) & 0.045 (0.003) \\
				Toilet Tissue & 4,340 & 145 & 6.86 & \textbf{0.081} (0.012) & 0.098 (0.012) & 0.092 (0.011) \\
				Toothbrushes & 74,686 & 2471 & 9.01 & \textbf{0.058} (0.002) & 0.061 (0.001) & 0.059 (0.002) \\
				Toothpaste & 43,678 & 1500 & 9.15 & \textbf{0.051} (0.001) & 0.053 (0.002) & 0.052 (0.002) \\
				Yogurt & 6,781 & 227 & 6.00 & \textbf{0.091} (0.008) & 0.107 (0.009) & 0.100 (0.007) \\
				\hline
		\end{tabular}}
		\label{tab:IRI KFold top15}
	\end{center}
\end{table}	
}

\end{document}